\documentclass[11pt]{article} %

\usepackage{newtxtext}
\usepackage[left=1.25in, top=1in, bottom=1in, right=1.25in]{geometry}

\usepackage[USenglish]{babel}
\usepackage[utf8]{inputenc} %
\usepackage[T1]{fontenc}    %
\usepackage{csquotes}
\usepackage{url}            %
\usepackage{booktabs}       %
\usepackage{amsfonts}       %
\usepackage{amsmath,amssymb}
\usepackage{nicefrac}       %
\usepackage{microtype}      %
\usepackage{xcolor}         %
\usepackage{enumitem}

\usepackage{amsmath}
\usepackage{amssymb}
\usepackage{amsthm}
\usepackage{thmtools}
\usepackage{mathtools, nccmath}
\usepackage{dsfont}
\usepackage{wrapfig}
\usepackage{comment}
\usepackage{thm-restate}

\definecolor{myblue}{rgb}{0,0.08,0.75}
\usepackage[colorlinks,citecolor=myblue,linkcolor=orange]{hyperref}

\usepackage{caption}
\usepackage{graphicx}
\usepackage{subcaption}
\usepackage{amsmath,amsfonts,bm}

\def\floor#1{\lfloor #1 \rfloor}
\def\1{\bm{1}}

\def\eps{{\epsilon}}

\def\vx{{\bm{x}}}

\DeclareMathAlphabet{\mathsfit}{\encodingdefault}{\sfdefault}{m}{sl}
\SetMathAlphabet{\mathsfit}{bold}{\encodingdefault}{\sfdefault}{bx}{n}

\def\gA{{\mathcal{A}}}

\def\gG{{\mathcal{G}}}

\def\gX{{\mathcal{X}}}

\def\sN{{\mathbb{N}}}

\def\sR{{\mathbb{R}}}
\def\sS{{\mathbb{S}}}

\def\sU{{\mathbb{U}}}

\def\sZ{{\mathbb{Z}}}

\DeclareMathOperator*{\E}{\mathbb{E}}

\newcommand{\R}{\mathbb{R}}

\DeclareMathOperator*{\argmax}{arg\,max}

\DeclareMathOperator{\Tr}{Tr}

\DeclareMathOperator*{\SE}{SE}
\DeclareMathOperator*{\SG}{SG}

\newcommand{\ind}{\mathbf{1}}  %

\newcommand*{\one}{{\bm 1}}

\newcommand{\rank}{\operatorname{rank}}
\newcommand{\Unif}{\mathrm{Unif}}

\newcommand{\inner}[2]{\left\langle #1,#2 \right\rangle}

\newcommand{\F}{\mathrm{F}}

\newcommand{\dist}{\mathsf{dist}}

\newcommand{\nxy}{\left\{(\vx_i,y_i)\right\}_{i=1}^n}

\newcommand{\ngxy}{\left\{(T(\vx_i),y_i)\right\}_{i=1}^n}

\def\cA{\mathcal{A}}

\def\cD{\mathcal{D}}

\def\cM{\mathcal{M}}
\def\cN{\mathcal{N}}

\def\cP{\mathcal{P}}

\newcommand{\id}{\mathsf{id}}

\usepackage{placeins}

\graphicspath{{arxiv/}{../}} %

\usepackage{cancel}

\usepackage[textsize=footnotesize,disable]{todonotes}

\usepackage{relsize}

\usepackage{hyperref,tablefootnote,footnotehyper}

\usepackage[capitalize,noabbrev]{cleveref}

\usepackage[natbib,style=authoryear,sortcites=ynt,sorting=nyt,maxcitenames=2,maxbibnames=60,useprefix,uniquelist=minyear,dashed=false,doi=false,
    backend=bibtex,
    ]{biblatex}
\renewbibmacro{in:}{} %
\DeclareFieldFormat[unpublished,misc]{title}{\mkbibquote{#1}}  %
\DeclareNameAlias{sortname}{given-family}

\newcommand{\zhiyuan}[2][noinline]{\todo[color=blue!30,#1]{Zhiyuan: #2}}

\newcommand{\amin}[2][noinline]{\todo[color=yellow!20,#1]{Amin: #2}}
\newcommand{\danica}[2][noinline]{\todo[color=violet!20,#1]{Danica: #2}}

\DeclareMathOperator{\bigO}{\mathcal{O}}

\newcommand{\dset}{\mathcal{D}}
\newcommand{\train}{\mathrm{train}}

\DeclareMathOperator{\diag}{diag}

\newcommand{\tp}{^\mathsf{T}}

\newcommand{\cG}{\mathcal{G}}
\newcommand{\cH}{\mathcal{H}}
\newcommand{\cL}{\mathcal{L}}
\newcommand{\cR}{\mathcal{R}}
\newcommand{\cV}{\mathcal{V}}
\newcommand{\cW}{\mathcal{W}}
\newcommand{\cX}{\mathcal{X}}
\newcommand{\cY}{\mathcal{Y}}
\newcommand{\cZ}{\mathcal{Z}}

\newcommand{\Q}{\mathbb{Q}}
\newcommand{\N}{\mathcal{N}}

\newcommand{\abs}[1]{\lvert #1 \rvert}

\newcommand{\norm}[1]{\left\lVert #1 \right\rVert}

\usepackage{xspace}
\makeatletter
\DeclareRobustCommand\onedot{\futurelet\@let@token\@onedot}
\def\@onedot{\ifx\@let@token.\else.\null\fi\xspace}
\makeatother

\newcommand\pcref[1]{(\cref{#1})}

\theoremstyle{plain}
\newtheorem{theorem}{Theorem}[section]
\newtheorem{proposition}[theorem]{Proposition}
\newtheorem{lemma}[theorem]{Lemma}
\newtheorem{corollary}[theorem]{Corollary}
\newtheorem{remark}[theorem]{Remark}
\theoremstyle{definition}
\newtheorem{definition}[theorem]{Definition}

\definecolor{almost_white}{rgb}{1.0,1.0,0.95}

\numberwithin{equation}{section}

\title{Why Do You Grok? \\ A Theoretical Analysis of Grokking Modular Addition}

\makeatletter
\def\@fnsymbol#1{\ensuremath{\ifcase#1\or \dagger\or \ddagger\or
   \mathsection\or \mathparagraph\or \|\or **\or \dagger\dagger
   \or \ddagger\ddagger \else\@ctrerr\fi}}
    \makeatother

\date{}
\author{
Mohamad Amin Mohamadi\thanks{%
    Toyota Technological Institute at Chicago
    \; $^\ddagger$University of British Columbia
    \; $^\mathsection$Peking University
    \; $^\mathparagraph$Amii
\\Correspondence to \texttt{\{mohamadamin,zhiyuanli\}@ttic.edu}, \texttt{dsuth@cs.ubc.ca}.%
    }\;\,\footnotemark[2]
\quad Zhiyuan Li\footnotemark[1]
\quad Lei Wu\footnotemark[3]
\quad Danica J.\ Sutherland\footnotemark[2]\;\,\footnotemark[4]
}

\begin{document}

\maketitle

\begin{abstract}
We present a theoretical explanation of the “grokking” phenomenon \citep{power2022grokking},
where a model generalizes long after overfitting,
for the originally-studied problem of modular addition.
First, we show that 
early in gradient descent, %
when the ``kernel regime'' approximately holds,
no permutation-equivariant model
can achieve small population error on modular addition unless it sees at least a constant fraction of all possible data points.
Eventually, however, models escape the kernel regime. 
We show that two-layer quadratic networks
that achieve zero training loss with bounded $\ell_\infty$ norm
generalize well with substantially fewer training points,
and further show such networks exist
and can be found by gradient descent with small $\ell_\infty$ regularization.
We further provide empirical evidence that these networks
as well as simple Transformers,
leave the kernel regime only after initially overfitting.
Taken together, our results strongly support the case for grokking as a consequence of the transition from kernel-like behavior to limiting behavior of gradient descent on deep networks.
\end{abstract}

\section{Introduction} \label{sec:introduction}
Understanding the generalization patterns of modern over-parameterized neural networks has been a long-standing goal of the deep learning community. \citet{power2022grokking} demonstrated an intriguing phenomenon they called ``grokking'' when learning transformers on small algorithmic tasks: neural networks can find a generalizing solution long after overfitting to the training dataset with poor generalization. This observation has lead to a stream of recent works aimed at uncovering the mechanisms that can lead a network to ``grok,'' and properties of the final solutions, on various algorithmic tasks. Later, it was discovered that grokking can happen in tasks beyond modular arithmetic: in learning sparse parities \citep{barak2022hidden,bhattamishra2023simplicity}, image classifiers \citep{liu2022omnigrok}, greatest common divisors \citep{charton2023transformers}, matrix completion \citep{lyu2023dichotomy}, and $k$-sparse linear predictors \citep{lyu2023dichotomy}.

Grokking has been variously attributed to difficulty of representation learning \citep{liu2022towards}, the ``slingshot'' mechanism \citep{thilak2022slingshot}, weight norm \citep{liu2022omnigrok,varma2023explaining}, properties of the loss landscape \citep{notsawo2023predicting}, simplicity of the learned solution \citep{nanda2023progress} and other feature learning mechanisms \citep{levi2023grokking, rubin2024grokking}. 
Theoretically, \citet{gromov2023grokking} presented an analytical construction for a two-layer MLP that solves modular addition is compatible with a grokking pattern.\footnote{\Citet{gromov2023grokking} claims this solution is the one found by gradient descent, but this did not seem to be the case in our experience.}
\citet{kumar2023grokking} demonstrated grokking when training a two-layer MLP on a polynomial regression problem,
as did \citet{xu2023benign} for XOR data.
The notion of delayed generalization was perhaps earlier observed by \citet{li2022happens} when training diagonal linear networks with label noise SGD and through sharpness minimization, before it was known as grokking \citep{power2022grokking}.

\citet{lyu2023dichotomy} present a rigorous theoretical framework in which grokking can be provably demonstrated through a dichotomy of early and late implicit biases of the training algorithm. More specifically, they attribute the overfitting stage of grokking to the initial behavior of gradient descent being similar to a kernel predictor \citep{jacot2018neural,arora2019exact,lee2019wide}, and the generalization stage is to different late-phase implicit biases such as sharpness minimization \citep{blanc2020implicit,li2022happens,damian2021label,haochen2020shape}, margin maximization \citep{soudry2022implicit,nacson2019lexicographic,wei2020regularization,lyu2019gradient}, or parameter norm minimization \citep{gunasekar2017implicit,gunasekar2020characterizing,arora2019implicit}. This transition from kernel to rich regime has been widely observed \citep{chizat2018lazy,moroshko2020implicit,Geiger_2020,telgarsky2022feature,lyu2023dichotomy}.
Consistent with this framework, \citet{kumar2023grokking} hypothesized that grokking can be explained through the transition from the kernel regime to the ``rich'' regime, as long as the size of the training dataset is neither too small (where generalization would be impossible) nor too large (where generalization would be easy). 
They provided empirical support by considering scaling the model output, which is a rough proxy for the rate of feature learning in modular addition on two-layer MLPs. 

\begin{figure*}[!t]
    \centering
    \begin{subfigure}[b]{0.32\textwidth}
        \includegraphics[width=\textwidth]{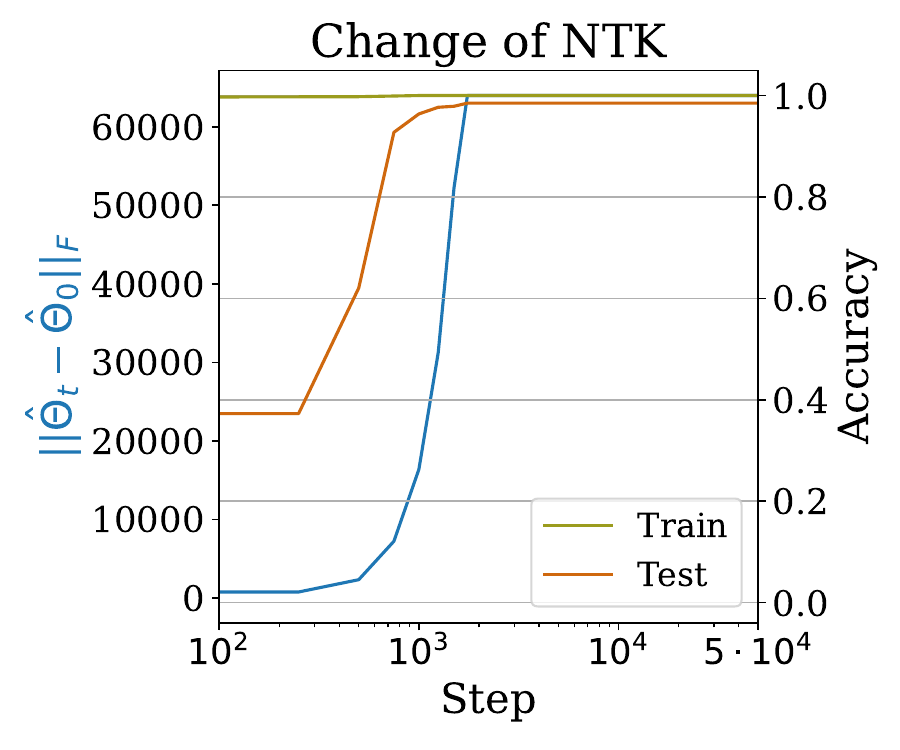}
        \label{fig:summary_entk_p113}
    \end{subfigure}
    \hfill
    \begin{subfigure}[b]{0.32\textwidth}
        \includegraphics[width=\textwidth]{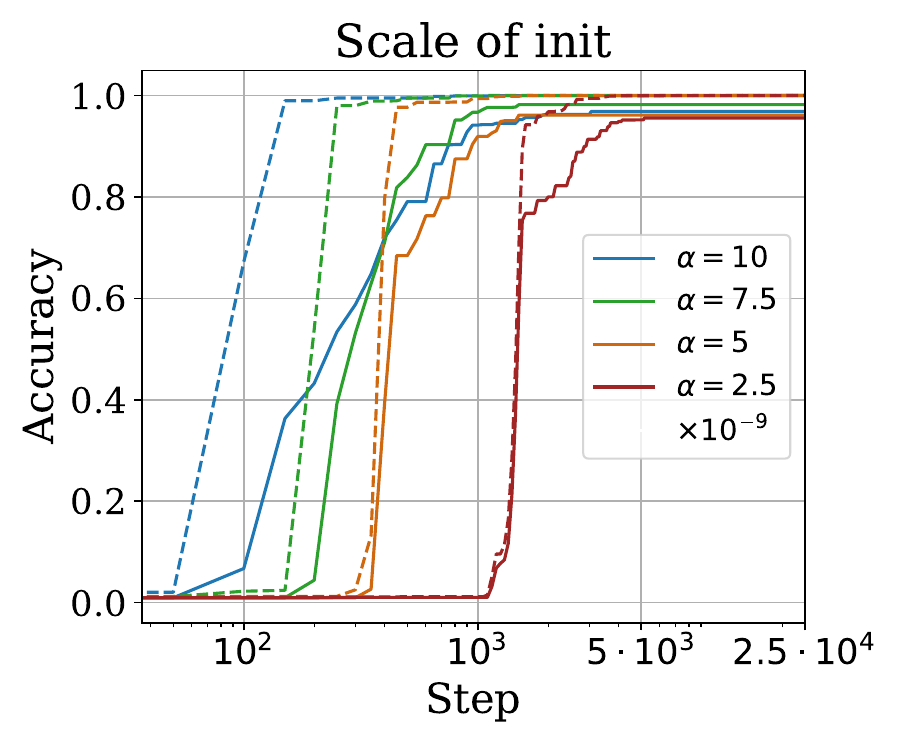}
        \label{fig:smmary_scale_p113}
    \end{subfigure}
    \hfill
    \begin{subfigure}[b]{0.32\textwidth}
        \includegraphics[width=\textwidth]{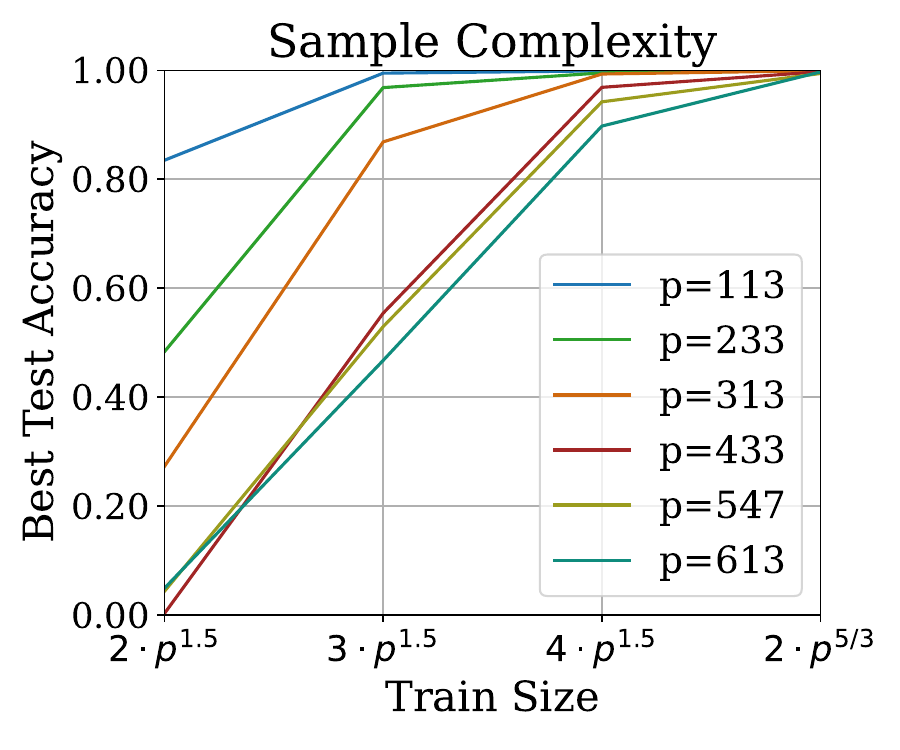}
        \label{fig:best_test_accs}
    \end{subfigure}
    \vspace*{-4mm}
    \caption[12]{%
        \small Empirical investigation into grokking modular addition on two-layer networks in the classification task with cross-entropy loss. \textbf{Left:} Change of empirical NTK\protect\footnotemark[1] ($\lVert \hat\Theta_t - \hat\Theta_0 \rVert_F$) is negligible before fitting the training data. NTK changes drastically after overfitting, implying that the delayed generalization might be caused by delayed a transitioning from kernel to rich regime.
        \textbf{Middle:} \textbf{Reducing initialization scale can mitigate grokking}, to the point of completely eliminating the gap between train and test curves. $\alpha$ denotes scale multiplied by $\theta_0$, the initial weights according to default PyTorch initialization~\citep{he2015delving}. The dashed lines indicate train set statistics, and the solid lines correspond to the test set.
        \textbf{Right:} Empirical evaluations support a sample complexity of $\tilde\bigO(p^{5/3})$ on the classification task with cross-entropy loss. More details in \cref{sec:classification}.
    }
    \label{fig:summary}
\end{figure*}\footnotetext[1]{$\hat\Theta_t$ is the NTK of  the model output on the training data using the parameter at step $t$. When the model output $f$ is a vector, we use the NTK of its first coordinate as an approximation, following \citep{mohamadi2023fast}, where    $\hat\Theta_t \triangleq [\nabla_{\theta}f_1(\theta_t; \cX_{\train})][\nabla_{\theta}f_1(\theta_t; \cX_{\train})]^\top$.}

This dichotomy between early kernel regime and late feature learning triggered by weak implicit or explicit regularization (i.e.\ the transition from lazy to rich regime) seems to be the most promising theory to explain grokking.
Even so, two fundamental questions as to \textit{\textbf{why}} grokking occurs on the original problem of \textbf{modular addition} have remained unanswered:
\begin{enumerate}[itemsep=1ex]
    \item \textbf{\textit{Why}} do a wide variety of architectures all fail to generalize in the initial phase of training, i.e.\ in the kernel regime? Is it because their kernels accidentally share some common property, or it is indeed a property of the {modular addition} task itself?
    \item \textbf{\textit{How}} does weak regularization encourage the network to learn generalizable features later in training? 
\end{enumerate}

\textbf{Our Contributions.} In this work, we address these questions with rigorous theoretical analyses of learning modular addition with gradient descent on two-layer MLPs.

We specifically focus on the problem
\begin{equation} \label{eq:modualr_addition}
    a + b \equiv c \pmod p
\end{equation}
where $a, b, c \in \sZ_p$ for a fixed $p \in \sN$. We use (regularized) gradient descent on a two-layer MLP with quadratic activations, the same architecture as \citet{gromov2023grokking}.
We consider two different tasks:
it is somewhat easier to analyze a ``regression'' task
where we use square loss to learn a function of $(a, b, c)$ that indicates whether \eqref{eq:modualr_addition} holds,
but we also study the ``classification'' task
in which we use cross-entropy loss to learn a $p$-way classifier to predict which $c$ value satisfies \eqref{eq:modualr_addition} for a given $(a, b)$.
This discrete, noiseless setting
has only a finite number of possible distinct data points:
$p^3$ for regression, $p^2$ for classification.

To address the first question,
we prove in \cref{subsec:regression_kernel,subsec:classification_kernel}
that this task is \emph{fundamentally hard} for permutation-equivariant kernel methods, due to inherent symmetries.
Thus, permutation-equivariant networks which are well-approximated by their neural tangent kernel approximation cannot generalize well.
We also prove that, although no such method can generalize,
neural tangent kernel approaches based on our architecture with realistic widths \emph{can} achieve zero training error.
This result is highly suggestive of why drastic overfitting with poor generalization has been empirically observed on this problem across a wide variety of architectures, losses, and learning algorithms
\citep[e.g.][]{power2022grokking,liu2022towards,gromov2023grokking}.

To prove the generalization lower bounds, we develop a novel general technique to analyze the population $\ell_2$ loss of learning general function classes with predictors of intrinsic dimension $n$ (for instance kernel predictors with $n$ training points), presented in \cref{app:sec:general_lower_bound}. This framework allows us to prove lower bounds on population $\ell_2$ loss for the general case of learning modular addition on $m$ summands (rather than 2 or 3) with kernels, and might be of independent interest.

To address the second question -- \emph{why} this occurs --
we identify $\ell_\infty$ norm as an effective and practically-relevant complexity measure for generalization,
which is closely related to the implicit bias of Adam~\citep{xie2024implicit,zhang2024implicit},
and show that networks with small $\ell_\infty$ norm in the regression setting (\Cref{subsec:regression_rich}) and large $\ell_\infty$-normalized margin in the classification setting (\Cref{subsec:classification_rich})
can both provably generalize with far fewer samples than required in the kernel regime.
Thus, models with corresponding implicit biases can generalize well,
answering the second question.  
In regression, our proofs are based on ``optimistic rates'' \citep{srebro2010smoothness} for the smooth $\ell_2$ loss in terms of the Rademacher complexity of networks with bounded $\ell_\infty$ parameter norm.
In classification, our proof applies the PAC-Bayesian framework of \citet{McAllester2003SimplifiedPM}  to networks with bounded $\ell_\infty$ parameter norm. 

In summary, our \textbf{main contributions} are as follows:
\begin{enumerate}
    \item We prove that networks in the kernel regime can only generalize if trained on \textbf{at least a constant fraction} of all possible data points, i.e.\ an $\Omega(1)$ portion, for regression (\cref{subsec:regression_kernel}) and classification (\cref{subsec:classification_kernel}).
    \item We prove that networks with appropriate regularization can generalize with many fewer samples:
    $\boldsymbol{\bigO(1 / p)}$ portion of all possible data points for square loss generalization on the regression task (\cref{subsec:regression_rich}),
    and $\boldsymbol{\tilde{\bigO}(1 / p^{1/3})}$ for zero-one loss generalization on classification \pcref{subsec:classification_rich}.
\end{enumerate}

\section{Preliminaries and Setup} \label{sec:setup}
\paragraph{Notations.} We use $[p]$ to denote the set $\{1,\ldots,p\}$.
We use $e_i$ to denote the $i$ standard basis vector, \emph{i.e.}, the vector with $1$ in its $i$th component and $0$ elsewhere. For vector $a$, we use $a^{\odot 2}$ denotes the element-wise square of the vector $a$
For any nonempty set $\cX$, a symmetric function $K:\cX\times \cX\to \mathbb{R}$ is called a  positive semi-definite kernel (p.s.d.) kernel on $\cX$ if for all $n\in\mathbb{N}$, all $\vx_1,\ldots,\vx_n \in \cX$, and all $\lambda_1,\ldots,\lambda_n\in\mathbb{R}$, it holds that $\sum_{i=1}^n\sum_{j=1}^n\lambda_i\lambda_jK(\vx_i,\vx_j)\ge 0$. Given two non-negative functions \(f, g\), we say that \(f(n) = O(g(n))\) (resp. \(f(n) = \Omega(g(n))\)) iff. there exists absolute constant \(C > 0\), such that for all \(n \geq 0\), \(f(n) \leq Cg(n)\) (resp. \(f(n) \geq Cg(n)\)). We use \(f(n) = \omega(g(n))\) to denote that for all \(C > 0\), there exists \(n_0 > 0\) such that for all \(n \geq n_0\), \(f(n) > Cg(n)\).

\paragraph{Setup.} We focus on learning modular addition, \eqref{eq:modualr_addition}, with a two-layer network with no biases and quadratic activation, following \citet{gromov2023grokking}.
More specifically, given parameters $\theta = (W,V)$, the model $f$ maps the pair of integers $(a, b)$ -- represented as the vector $(e_a, e_b) \in \R^{2 p}$ -- to a vector in $\R^p$.
We use the form
\begin{equation} \label{eq:two_layer_classification_net}
f(\theta; (e_a, e_b)) = V \left(W (e_a, e_b) \right)^{\odot 2},
\end{equation}
where $W \in \R^{h \times 2 p}$, $V \in \R^{p \times h}$ for some integer $h$. We call $h$ the width of the hidden layer and it will be set later.

In the classification setting (\Cref{sec:classification}),
we train $f$ with cross-entropy loss to identify the $c$ such that $a + b \equiv c \pmod p$ in a multi-classification setting. Letting $\cZ = \{ e_i : i \in [p] \}$,
the set of all possible inputs is $\cX = \cZ \times \cZ$
and outputs is $\cY = \cZ$;
there are $N = p^2$ distinct data points.

In the regression setting (\Cref{sec:regression}), we instead aim to learn for each tuple $(a,b,c)$, if $a + b \equiv c \pmod p$, that is, map $\vx = (e_a, e_b, e_c)$
to the scalar $y = p \, \mathbf{1}(a + b \equiv c \pmod p)$ using $g(\theta;(e_a,e_b,e_c)))\triangleq e_c^\top f(\theta;(e_a,e_b))$.\footnote{%
This scaling implies that the predictor $\Psi_{\boldsymbol{0}}(\cdot) = 0$ has population square loss $p$; this scaling allows bounded $\norm\theta_\infty$.}
Here $\cX = \cZ^3$,
$\cY = \{0, p \}$,
and $N = p^3$;
we train the model
$g(\theta; \vx) = \langle e_c, f(\theta; (e_a, e_b)) \rangle$
to minimize the square loss.

In either setting,
we use $\dset$ to denote the distribution over $\cX \times \cY$ which is uniform over all $N$ possible input-output pairs,
while $\dset_\train = \nxy$ is the training dataset of size $n$.
Given parameters $\theta$, we use $\Psi(\theta; \cdot) : \cX \to \hat\cY$ to denote the predictor ($g(\theta;\cdot)$ for regression, $f(\theta;\cdot)$ for classification),
we train the model with gradient descent with tiny $\ell_\infty$ regularization,\footnote{More details about the regularization strength can be found in \Cref{app:experiment_setup}.} which is intended to emulate the implicit bias of Adam (AdamW) which might lead to grokking. In general, we define loss as
\begin{equation} \label{eq:loss}
\!\!\!
    \cL^{\lambda}_\ell(\Psi, \theta, \dset_{\operatorname{train}}) \triangleq \frac{1}{n} \!\!\sum_{(\vx, y) \in \dset_{\operatorname{train}}}\!\!\!\!\!
    \ell(\Psi(\theta; \vx), y) + \lambda \norm{\theta}_\infty 
,\end{equation}
where
$\ell : \hat \cY \times \cY \to \R_{\ge 0}$ denotes either square ($\ell_2$) or cross-entropy loss and $\lambda \ge 0$ controls the strength of regularization. Omitting $\lambda$ refers to the case where no regularization is used and omitting $\dset_\train$ means the whole population is used in evaluating the loss. For non-parametric functions (like $\Psi_{\boldsymbol{0}}$ denoted as the predictor which returns zero on any input) we drop $\theta$ as well.

\paragraph{Connection between $\ell_\infty$-norm Regularization and Grokking.} \citet{power2022grokking} shows that grokking phenomenon happens for Adam and gets more prominent with decoupled weight decay (AdamW). Recent works by \citet{xie2024implicit,zhang2024implicit} show that Adam implicitly regularizes the $\ell_\infty$ norm of the weights. In detail, \citet{xie2024implicit} shows that AdamW can only converge to KKT points of $\ell_\infty$-norm constrained optimization. \citet{zhang2024implicit} shows that Adam converges to max-margin solutions w.r.t. $\ell_\infty$ norm for linear models on separable datasets and conjecture this result could be generalized to general homogeneous models (including our model $f$, which is 3-homogeneous). This connection motivates us to use the explicit $\ell_\infty$-norm regularization on top of gradient descent to understand grokking in the modular addition setting.
\begin{definition}
	A deterministic supervised \emph{learning algorithm} $\cA$ is a mapping from a sequence of training data, $\dset_\train \in (\cX \times \cY)^n$,
    to a hypothesis
    $\cA(\dset_\train) : \cX \to \hat\cY$.
    The algorithm $\cA$ could also be randomized, in which case the output
    $\cA(\dset_\train)$
    is a distribution over hypotheses. Two randomized algorithms $\cA$ and $\cA'$ are the same if for any input, their outputs have the same distribution in function space, 
    written as $\cA(\dset_\train)\overset{d}{=}\cA'(\dset_\train)$.
\end{definition}

We further define equivariance of learning algorithms which is the main component of our analysis in deriving lower bounds. This definition is used in \cref{th:permutation_equivaraince_3d,prop:permutation_equivaraince_2d} to prove equivariance of GD in learning modular addition.
\begin{definition}[Equivariant Algorithms]\label{defi:eqvariance}
	A learning algorithm is \emph{equivariant} under group $\cG_\cX$ (or $\cG_\cX$-equivariant) if %
    for any dataset $\dset_\train \in (\cX\times\cY)^n$
    and for all $T \in\cG_\cX$, %
    it holds that
    $\cA(\ngxy) \,\circ\, T \overset{d}{=} \cA(\nxy)$.
    For deterministic learning algorithms,
    this is equivalent to saying that
    for all $\vx \in \cX$,
    $ \cA(\ngxy)(T(\vx)) = [\cA(\nxy)](\vx)$.
\end{definition}

\begin{definition}[Kernel Methods]\label{defi:kernel_methods}
For $\hat\cY\subseteq\mathbb{R}$, we say a learning algorithm $\cA$ is a \emph{kernel method} if it first picks a (potentially random) positive semi-definite kernel
$K$ on $\cX$
before seeing the data,\footnote{Our definition of kernel methods does not cover learning algorithms that choose the kernel based on the training data. \emph{Any} learning algorithm could be framed as a kernel method with a data-dependent kernel $K(x, x') = \Psi(x) \Psi(x')$.} and then outputs some hypothesis $\Psi$ such that there exist $\{\lambda_i\}_{i=1}^n\in\mathbb{R}$ for which $\Psi(\cdot) = \sum_{i=1}^n K(\cdot,\vx_i)\lambda_i$, where $\lambda_i$ can depend on the training dataset $\dset_\train$. 

In particular, when $\bm{\lambda}(\dset_\train) = \mathbf{K}^{\dagger}\mathbf{y}$, where $\mathbf{K} = \left(K(\vx_i,\vx_j)\right)_{i,j=1}^n$, $\mathbf{y} = \left(y_i\right)_{i=1}^n,$ the kernel method is called a \emph{(ridgeless) kernel regression} method.
\end{definition}

\begin{theorem}\label{thm:kernel_methods_equivariance}
    For any p.s.d. kernel $K:\gX\times\gX\to \mathbb{R}$ and transformation group $\gG_gX$, kernel regression (\Cref{defi:kernel_methods}) with respect to kernel $K$ is $\gG_{\gX}$-equivariant if and only if kernel $K$ is equivariant to $\gG_gX$, \emph{i.e.},  $K(T(x),T(x')) = K(x,x')$ for any $T\in\gG_gX$ and $x,x'\in \gX$.    
\end{theorem}

\begin{proof}[Proof of \Cref{thm:kernel_methods_equivariance}]
    For any $\dset_\train = \{(\vx_i,y_i)\}_{i=1}^n$ and transformation $T\in \gG_\gX$, let $\dset_\train^T$ be the transformed dataset $\{(T(\vx_i),y_i)\}_{i=1}^n$, and $\gA_K$ be the kernel regression algorithm w.r.t.\ $K$. For any $\vx$, we have that
    \vspace{-0.5\baselineskip}
    \begin{align*}
        \gA(T(\dset_\train))(T(\vx)) = \sum_{i=1}^n K(T(\vx),T(\vx_i))\lambda_i (\dset_\train^T) = \sum_{i=1}^n K(\vx,\vx_i)\lambda_i (\dset_\train) = \gA(\dset_\train)(\vx),
   \end{align*}
   \vspace{-0.5\baselineskip}
   where the second equality follows from the equivariance of the kernel $K$ w.r.t. $\gG_\gX$.
\end{proof}

\section{Regression Task} \label{sec:regression}

We will first present our theoretical analysis of the grokking phenomenon on modular arithmetic with two-layer quadratic networks in the regression setting, where we learn a function from $(e_a, e_b, e_c)$ to $p \mathbf{1}(a + b = c \pmod p)$.
Although this is perhaps a less natural way to model modular arithmetic than the classification task,
it admits some useful theoretical tools. 

In this section, we show that networks in the kernel regime will provably fail to generalize as long as they do not have access to nearly all the points in the dataset \pcref{th:kernel_lower_bound_3d},
although they can achieve zero training error (\cref{thm:mainbody-ntk}).
However, the network eventually leave the kernel regime because of the weak $\ell_\infty$ norm regularization. 
We further establish that if the network manages to achieve zero training error with small $\norm{\theta}_\infty$,
it will generalize with only a $\frac n N = \tilde{\omega}( 1 / p )$ portion of the overall dataset, as long as the width of the network is larger than $4p$~\pcref{thm:main_regression_generalization_upper_bound}. We also prove that such networks exist \pcref{thm:small_inf_norm_exist}, and demonstrate that gradient descent can find them with a small amount of explicit regularization \pcref{fig:regression_norms}.

\subsection{Kernel Regime} \label{subsec:regression_kernel}
\begin{figure}[!t]
    \centering
    \begin{subfigure}[b]{0.35\textwidth}
        \includegraphics[width=\textwidth]{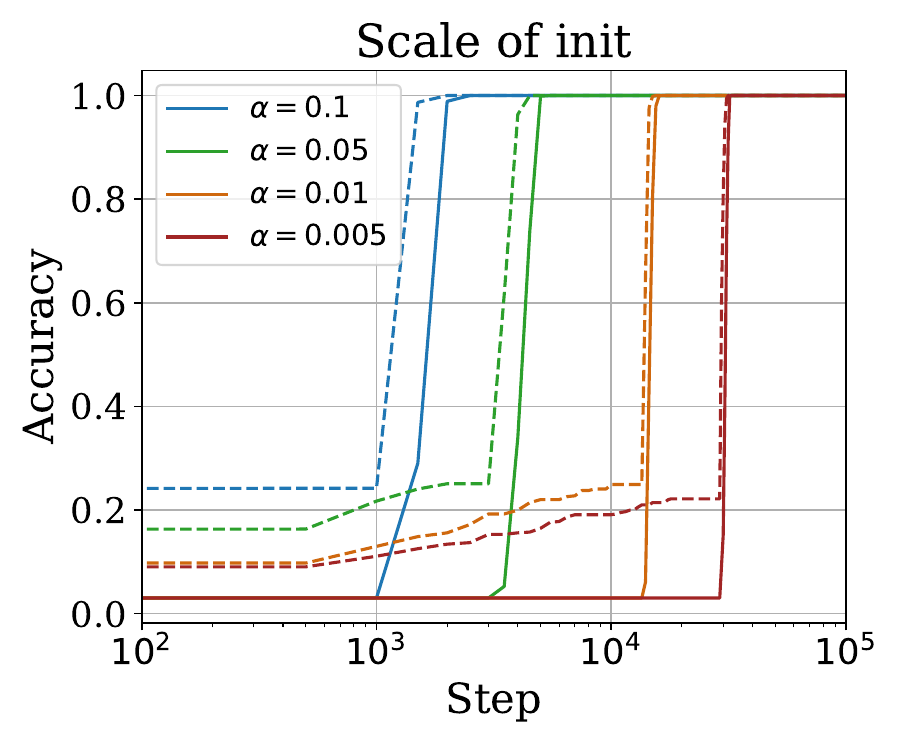}
        \label{fig:regression:acc_vs_scale}
    \end{subfigure}
    \hspace{0.05\textwidth}
    \begin{subfigure}[b]{0.35\textwidth}
        \includegraphics[width=\textwidth]{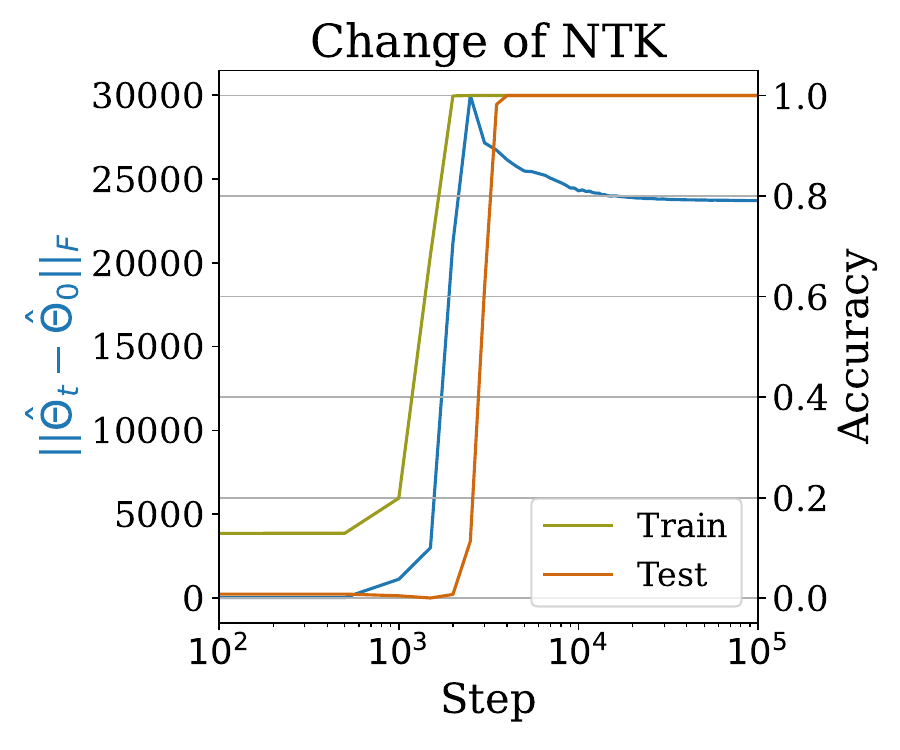}
        \label{fig:regression:scale}
    \end{subfigure}
    \vspace*{-5mm}
    \caption{%
        Empirical evidence for kernel regime in early training.
        \textbf{Left:} Train (dashed) and test (solid) accuracy while training in the regression setting
        with various initialization scales $\alpha$ and a fixed $p=47$. \textbf{Shrinking the scale of initialization can mitigate grokking} in the regression task, eventually eliminating of the gap between train and test accuracies (at the cost of slower improvement in each). \textbf{Right:} eNTK continues to significantly change after overfitting. 
    }
    \vspace*{-3mm}
    \label{fig:regression_scale}
\end{figure}

A recent line of work on the \emph{neural tangent kernel} (NTK) framework
\citep{jacot2018neural,arora2019exact,lee2019wide,chizat2018lazy,pmlr-v139-yang21c}
has shown that with typical initialization schemes,
gradient descent in over-parameterized neural networks 
locally approximates the behavior of a kernel model
using the empirical neural tangent kernel (eNTK)
$K_\theta(\vx, \vx') \triangleq \nabla g(\theta; \vx) \nabla g(\theta; \vx')\tp$.
In the ``kernel regime,''
the change in $\theta$ over the course of gradient descent
does not substantially change the eNTK $K_\theta$,
and hence the neural network behaves similarly to a kernel predictor trained with $K_{\theta_0}$.
With square loss, as here,
these kernel predictors follow a particularly simple optimization path
for which a closed form (corresponding to kernel regression) is available.

For networks of finite width (and in certain infinite-width cases),
the eNTK will stay roughly constant and the network will closely track the kernel model
for the first phase of optimization, but the tiny regularization will eventually lead it to depart the kernel regime. Thus, bounds on kernel models are informative about deep networks in the first part of optimization.

In the first phase of grokking, the model overfits to the training data, achieving very low training loss but retaining high loss on test points.We establish that both phenomena occur in the kernel regime.

\subsubsection{Empirical Neural Tangent Kernels Can Achieve Zero Training Error}
Our first result
shows
that kernel regression with the empirical neural tangent kernel of our quadratic network
achieves zero training loss if the network is mildly wide, for instance, $n = \tilde\Theta(p^2)$ and width $h = \tilde\Theta(p)$.
This implies, for example, that in a ``lazy training'' setting \citep{chizat2018lazy},
or alternatively when the first layer is initialized with huge weights and the second layer with tiny ones,
gradient descent can achieve zero training loss.
Informally, this also strongly suggests that networks with more typical initializations can achieve very small training error without needing to leave the ``kernel regime.''
\begin{theorem} \label{thm:mainbody-ntk}
    Initialize the network of \cref{sec:setup} 
    with any values for $V$,
    and entries of $W$ all independently $\operatorname{Uniform}([-s_W, s_W])$ for some $s_W > 0$.
    Let the most frequent value of $c \in [p]$ appear $\rho_c$ times.
    Then, if $h > 36 \rho_c \log(n / \delta)$,
    it holds with probability at least $1 - \delta$
    over the random initialization of $W$
    that kernel regression using the network's empirical neural tangent kernel
    can achieve zero training loss for \emph{any} target labels such that if $\vx_i = \vx_j$ then $y_i = y_j$.
    Conversely, if $h < n / (3 p)$,
    there exist target labels (with $y_i = y_j$ when $\vx_i = \vx_j$) for which this method cannot achieve zero training loss.
\end{theorem}

Uniform initialization is the default behavior e.g.\ in PyTorch \citep{pytorch2}. Note that $\rho_c$ is always between $ n / p $ and $n$.
Moreover, in the usual setting with randomly selected $\dset_\train$ and $n = \Omega\bigl( p \log p \bigr)$
-- recall that full knowledge of a single $(a, b)$ pair requires $n = p$ in the regression setting --
we have that $\rho_c = \Theta( n / p )$ with high probability~\citep[Theorem 1]{balls-into-bins}.
Thus in this usual setting,
the threshold for interpolation is at a network width $h = \tilde\Theta(n / p)$.

\begin{proof}[Proof sketch of \Cref{thm:mainbody-ntk}]
    Kernel regression can achieve all possible target labels iff its kernel matrix is strictly positive definite (\cref{thm:kernel_interpolating,thm:krr-noninterp}).
    Evaluating the expected neural tangent kernel,
    we can see it is strictly positive definite
    under only mild assumptions on the weights
    (\cref{prop:expected-kern-interp}).
    With bounded weights, a matrix Chernoff bound controls the convergence of the lowest eigenvalue of the kernel matrix to that of the expected kernel matrix (\cref{prop:entk-interp}).
    The lower bound follows from the rank of the kernel matrix (\cref{thm:entk-singular}).
\end{proof}

\subsubsection{Permutation-Equivariant Kernel Methods Cannot Generalize}
We now show
that for any permutation-equivariant kernel method (\Cref{defi:kernel_methods})
(and hence for networks trained by gradient descent close enough to initialization),
generalization is possible only when training on $n = \Omega(p^3)$ samples,
i.e.\ the portion of all possible data points is $\frac{n}{N} = \Omega(1)$.

The key component of our analysis is the permutation equivariance of learning modular addition in this setting.
We first define the permutation group on the modular addition data:

\begin{definition} %
\label{def:permutation_3d}
Let $\sS_p$ denote the set of all permutations on $[p]$.
We define the \emph{permutation group} $\cG_\cX$ on $\cX$
as the group
    \[
    \left\{
    (e_a, e_b, e_c) \mapsto (e_{\sigma_1(a)}, e_{\sigma_2(b)}, e_{\sigma_3(c)}): \sigma_1, \sigma_2, \sigma_3 \in \sS_p
    \right\}
    ,\] 
with the group operation being composition of each permutation.
\end{definition}

The following theorem establishes that our learning algorithm is
equivariant (\Cref{defi:eqvariance}) under this permutation group,
which we call \emph{permutation-equivariant} for short.
Note that this result applies to the actual process of gradient descent on our neural network,
not only to its NTK approximation.
(This result is not particularly specific to the architecture defined in \cref{sec:setup};
it holds broadly.) 

\begin{restatable}[Permuation Equivariance in Regression]{theorem}{regressionpermeq}
\label{th:permutation_equivaraince_3d}
     Training a two-layer neural network with quadratic activations (such as $g(\theta; x))$ initialized with random i.i.d. parameters using gradient-based update rules (such as gradient descent) on the modular addition problem in the regression setting is an equivariant learning algorithm (as defined in \Cref{defi:eqvariance}) with respect to the permutation group of \cref{def:permutation_3d} applied on the input data.
\end{restatable} 
More details and the full proof are in \cref{app:sec:permutation_equivariance}. 

Roughly speaking, this indicates that learning the modular addition task on this setup is exactly the same difficulty as learning any permuted version of the dataset. 
Following the same argument, we can show that the kernel method corresponding to the neural nets in the early phase of training is also permutation-equivariant.

Further, as the distribution $\cD$ is uniform and thus invariant under permutation,
\Cref{th:permutation_equivaraince_3d} shows that the difficulty of learning the original ground-truth is the same as simultaneously learning all the permuted versions of the ground-truths, which turns out to be difficult for any kernel method. The following theorem formalizes this idea, establishing that any such kernel predictor needs at least $n = \Omega(p^3)$, or equivalently $\frac n N = \Omega(1)$, to generalize on the modular addition task;
otherwise it cannot do substantially better than the trivial all-zero predictor.

\begin{restatable}[Lower Bound]{theorem}{kernellowerbound}
\label{th:kernel_lower_bound_3d}
    There exists a constant $C > 0$ such that for any $p \ge 2$, training data size $n < Cp^3$, and any permutation-equivariant kernel method $\cA$, it holds that
    \begin{equation*}
    \!\!
        \E_{ (\vx_i,y_i)_{i=1}^n\sim \cD^n}\E_{\cA}\cL_{\ell_2}  \left(\cA\left( \{(\vx_i,y_i)\}_{i=1}^n\right) \right)
        \ge \frac{w}{2} \; \cL_{\ell_2} \left(\Psi_{\boldsymbol{0}}\right)= \frac{p}{2},
    \end{equation*}
    where $\E_{\cA}$ takes the mean over the randomness in algorithm $\cA$. 
\end{restatable}

A full proof of this result is deferred to \cref{app:sec:kernel_lower_bound_3d}.

\cref{th:kernel_lower_bound_3d} is in fact a special case of a more general theorem that lower bounds the population $\ell_2$ loss of learning modular addition with $m$-summands using permutation-equivariant kernels, showing that poor generalization is inevitable. We present an informal version of this result below and refer the reader to \cref{lem:m_dim_modular_addition_loss_lower_bound} for the formal version.

\begin{theorem}[Informal] Consider the problem of learning modular addition over $m$-summands with a training set of size $n$ using a permutation equivariant kernel method $\cA$. For any $p \ge 2$ and training data size $n < Cp^m$ the expected population $\ell_2$ loss is at least $\Omega(p)$, which is of the same magnitude as the trivial all-zero predictor.
\end{theorem}

Poor population loss is thus inevitable for models which are well-approximated by a permutation-equivariant kernel;
since our quadratic network can also achieve zero training error in the kernel regime,
this strongly suggests drastic overfitting in early training.
Fortunately, however, regularized gradient descent on finite networks will eventually leave the kernel regime and begin learning features.

\subsection{Rich Regime} \label{subsec:regression_rich}

\begin{figure*}[!t]
    \centering
    \begin{subfigure}[b]{0.35\textwidth}
        \includegraphics[width=\textwidth]{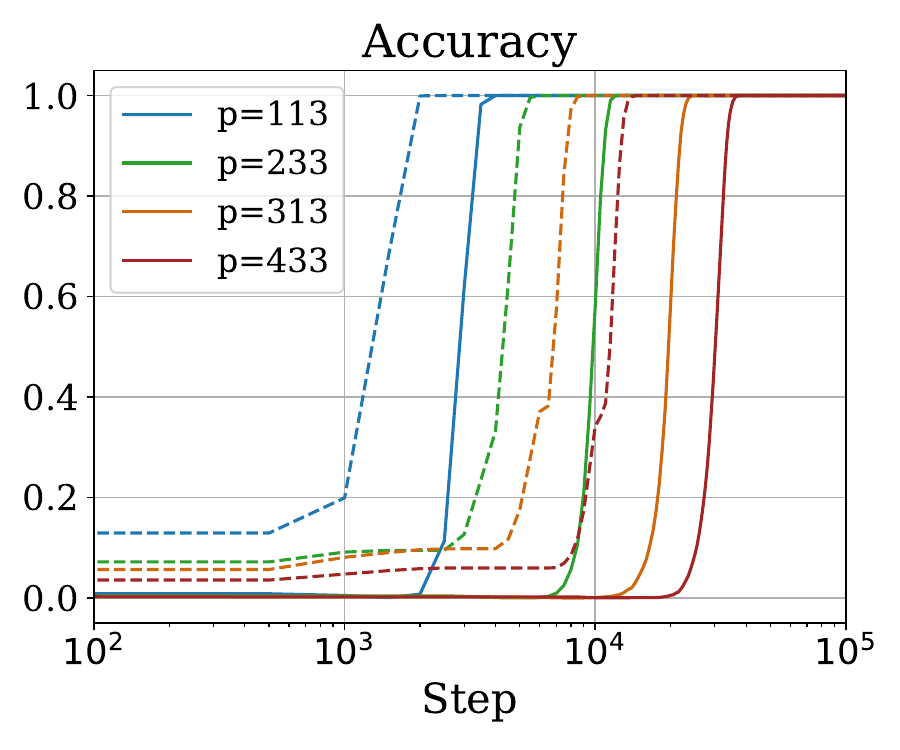}
        \label{fig:regression:acc}
    \end{subfigure}
    \hspace{0.05\textwidth}
    \begin{subfigure}[b]{0.35\textwidth}
        \includegraphics[width=\textwidth]{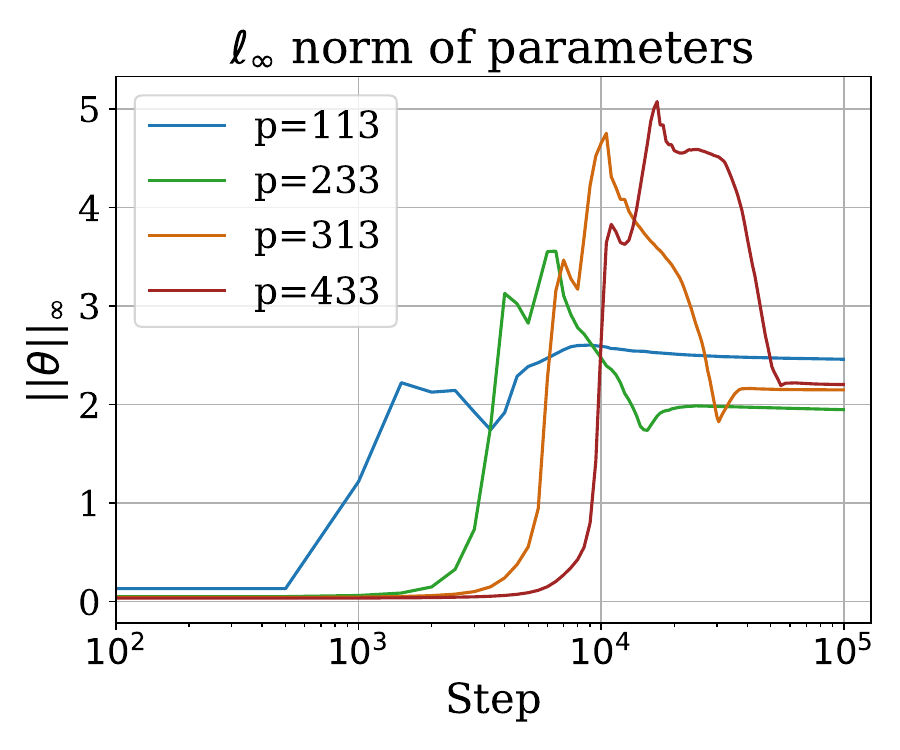}
        \label{fig:regression:theta_norm}
    \end{subfigure}
    \vspace*{-5mm}
    \caption{%
        Empirical verification of our theoretical explanation for generalization. We train a network of width $h=4p$ with gradient descent on $\ell_2$ loss and $10^{-4}$ $\ell_\infty$-regularization on $2 \times p^{2.25}$ training samples  (out of $p^3$) in the regression setting.
        \textbf{Left:} \textbf{Generalization happens when the number of samples are more than $\Omega(p^2)$}, as predicted by \cref{th:generalization_bound_3d}.  The dashed and solid lines indicate train and test set statistics respectively. \textbf{Right:} $\ell_\infty$ norm of parameters after grokking remains the same for different problem dimensions $p$, as predicted by }
    \label{fig:regression_norms}
\end{figure*}

While the transient behavior of gradient descent after leaving the kernel regime may be complicated,
it is often the case that the limiting behavior can be better-understood based on analyses of implicit bias.
Motivated by this,
we present a generalization analysis of networks that achieve zero training error
(as we expect in the long-term optimization limit)
with bounded $\norm{\theta}_\infty$.
This assumption is motivated by recent work of \citet{xie2024implicit}, who show that full-batch AdamW
can only converge to KKT points of the $\ell_\infty$ norm-constrained optimization problem, $\min_{\norm{\theta}_\infty\le \frac{1}{\lambda}} \cL(\Psi, \theta, \dset_\train)$, where $\lambda$ is the weight decay coefficient.
We will discuss the feasibility of this assumption more afterwards.

\begin{theorem}[Upper Bound]\label{thm:main_regression_generalization_upper_bound}
    For any width $h\ge 8p$, with probability at least $1-\delta$ over the randomness of training dataset $\dset_{\operatorname{train}}$ of size $n$, define the set of interpolating solutions as $\tilde\Theta^* = \{\theta\mid \cL_{\ell_2}(\Psi, \theta, \dset_{\operatorname{train}}) = 0\}$.
    For any interpolating solution $\theta^*\in \tilde\Theta^*$ with small $\ell_\infty$ norm, \emph{i.e.}, satisfying that $\|\theta^*\|_\infty = \bigO(\min_{\theta\in\tilde\Theta^*}\norm{\theta}_\infty)$, it holds that 
    \begin{align*}
        \cL_{\ell_2}(g, \theta^*)
        &= \bigO\left(\frac{ p^2}{n} \left(\log^3 n + \frac{1}{p}\log \frac1\delta \right)\right)\cdot  \cL_{\ell_2} \left(\Psi_{\boldsymbol{0}}\right)
    .\end{align*}
\end{theorem}

Comparing \cref{thm:main_regression_generalization_upper_bound} to \cref{th:kernel_lower_bound_3d}, we see that when $n=\tilde{\omega}(p^2)$ and $h = \Omega(p)$, there is a strict separation between generalization in the kernel and rich regimes. Note that the classifier $\phi(e_a, e_b) \triangleq \argmax_{c} f_c(\theta^*;(e_a,e_b))$ 
has population error rate at most $\frac{2}{p} \cL_{\ell_2}(\phi)$ (as shown in \cref{lem:bounded_l2_loss_bounded_error})
and thus, its population error goes to zero when $n = \tilde\omega(p^2)$.
The proof of \Cref{thm:main_regression_generalization_upper_bound} consists of
showing all networks with small training error and small $\ell_\infty$ norms generalize (\cref{th:generalization_bound_3d}),
and at least one such network exists (\cref{thm:small_inf_norm_exist}).

\begin{restatable}{proposition}{generalizationbound}
\label{th:generalization_bound_3d}
For any $R > 0, \delta \in (0, 1)$, $\dset_{\operatorname{train}}$ of size $n$, and $\theta^* \in \{\theta = (W, V): \cL_{\ell_2}(g, \theta, \dset_{\operatorname{train}}) = 0 \wedge \norm{\theta}_\infty \le R\}$, there exists a positive constant $C > 0$ such that with probability at least $1 - \delta$ over the randomness of $\dset_{\operatorname{train}}$,
    \begin{equation*}
        \cL_{\ell_2}(g, \theta^*) \le \frac{C R^6 h^2}{n} \left(p \log^3 n + \log \frac1\delta \right).
    \end{equation*}
\end{restatable}
\begin{proof}[Proof sketch of \Cref{th:generalization_bound_3d}]
We bound the Rademacher complexity of the set of networks with small $\ell_\infty$ weights,
and then apply
\cref{prop:smoothness_risk_ub}, due to \citet{srebro2010smoothness},
which gives an ``optimistic'' bound on the excess risk of smooth loss functions. %
Details in \cref{app:sec:generalization_bound_3d}.
\end{proof}
\begin{restatable}[]{proposition}{smallinfnormexist}\label{thm:small_inf_norm_exist}
    Let the set of models with zero population loss be $\Theta^*\triangleq\{\theta\mid \cL_{\ell_2}(g, \theta)=0\}$.
    For any $p\ge 2$ and $h\ge8p$, $\Theta^*$ is nonempty and $\min_{\theta\in\Theta^*}\norm{\theta}_{\infty} \le \left\lfloor \frac{h}{8p}\right\rfloor^{-\frac{1}{3}}$.  
\end{restatable}

\begin{proof}[Proof sketch of \Cref{thm:small_inf_norm_exist}] %
    The main difficulty
    is a manual Fourier-based construction of a zero-loss solution with $h=8p$ and $\ell_\infty$ norm at most one; this is inspired by similar constructions of \citet{gromov2023grokking} and results of \citet{nanda2023progress}. In a concurrent work, \citet{morwani2023feature} presented a similar construction. Once we have that, because our model is 3-homogeneous in its parameters, we can easily reduce the $\ell_\infty$ norm by duplicating neurons, without changing the input-output function.
    Details in \cref{app:manual_construction}.
\end{proof}

We know that an interpolating solution with small $\ell_\infty$ norm exist, and that any such solution will generalize.
Does gradient descent find such a solution?
In \cref{fig:regression_norms}, we empirically evaluate the $\ell_\infty$ norm of weights learned by running gradient descent on the regression task, with varying $p$,
and a very small $\ell_\infty$ regularizer.\footnote{We used a regularization weight of $10^{-4}$. Without explicit regularization, a small number of network weights do grow with $p$, but we believe this phenomenon is not important to the overall behavior of gradient descent on this task. }
Consistent with our manual construction of weights, the $\ell_\infty$ norm of the network weights does not grow with the problem dimension $p$.
This supports the applicability of \cref{thm:main_regression_generalization_upper_bound} and a far better sample complexity compared to the kernel regime.

\vspace*{3mm}
\zhiyuan{check later}
\textbf{Empirical Evidence that Kernel Lower Bounds Lead to Overfitting:} One way to mitigate the grokking effect in learning modular addition is by preventing the network from overfitting the training set in the kernel regime. To do so, one can reduce the \textit{scale of initialization} (e.g.\ smaller variance scales for weight initializations in the scheme of \citet{he2015delving}). Roughly speaking, networks initialized with a very small scale of parameters need to undergo a significant growth in the norm of parameters to be able to fit the training data. This norm growth will eventually lead the network to leave the kernel regime and begin learning features before overiftting the data. \Cref{fig:regression_scale} confirms that in our setting:
using a very small weight initialization can substantially mitigate the grokking effect by preventing the network from overfitting in the kernel regime, incidacting that grokking is indeed caused by a delayed transition from kernel (overiftting) to rich (generalizing) regime. However, this comes at a cost: training becomes increasingly more difficult as the scale of initialization decreases \citep{chizat2018lazy,moroshko2020implicit,telgarsky2022feature}.

In \cref{fig:regression_scale}, we also evaluate the change of the empirical NTK during training, by computing the difference in eNTK matrices through training.
To make this empirical investigation computationally feasible,
we evaluated the eNTK approximation of \citet{mohamadi2023fast}
on 20,000 random data points,
similar to previous schemes \citep{fort2020deep,wei2020regularization,mohamadi2023fast}.
We see that the change in empirical NTK is orders of magnitude larger after overfitting the training set, implying that most feature learning happens only later,
supporting our hypothesis that the initial overfitting occurs roughly in the kernel regime.

\section{Classification Task} \label{sec:classification}

We now move onto the multi-class classification setting established in \cref{sec:setup},
where we train with cross-entropy loss.
Similar to the regression problem, we first analyze the early stage of training where the network operates like a kernel machine, and then move onto the rich regime and focus on the weights learned through margin-maximization implicit bias of gradient descent. We prove that a sample complexity gap between kernel \pcref{th:kernel_lower_bound_2d} and rich \pcref{th:generalization_bound_2d} regime exists when learning a two-layer neural network with gradient descent on modular addition modeled as a multi-class classification task. Our results imply that as long as the max-margin implicit bias of gradient descent on exponential-type loss functions drives the net to leave the kernel regime and $\hat{\Omega}(p^{5/3})$ samples are used for training, generalization to the whole population is guaranteed.

\subsection{Kernel Regime} \label{subsec:classification_kernel}

In the multi-class setting, the output is $p$-dimensional, and thus the eNTK for each pair of points is a $p\times p$ psd matrix \citep[see][]{alvarez2012kernels}. Our notion of kernel methods (\cref{defi:multi_out_kernel_methods}) slightly changes due to account for multi-output functions, and the main lower bound result is similar to that in the regression case: kernel methods must see a constant fraction of data before learning. 

\begin{definition} \label{defi:multi_out_kernel_methods}
For $\cY\subseteq\R^p$, we say a learning algorithm $\cA$ is a \emph{kernel method} if before seeing the data it first picks a (potentially random) kernel $K$ on $\cX$ such that for any $x, x' \in \cX$, $K(x, x')$ is a $p$ by $p$ positive semi-definite matrix, and then outputs some hypothesis $h$ based on $\dset_\train$ such that there exist $\{\lambda_i\}_{i=1}^n\in\R^p$ for which $h(\cdot) = \sum_{i=1}^n K(\cdot,\vx_i)\lambda_i$.
\end{definition}

We next establish the permutation-equivariance of gradient-based learning algorithms in the classification setting. They key difference is that in the multi-output setting, gradient-based learning algorithms are equivariant to permutations on \emph{both input data and output labels}. More details are available in \Cref{app:sec:permutation_equivariance}.

\begin{restatable}[Permuation Equivariance in Classification]{proposition}{classificationpermeq}
\label{prop:permutation_equivaraince_2d}
    We define the \emph{permutation group} $\cG_{\cX, \cY}$ on $\cX \times \cY$ (defined in \Cref{sec:setup}) as the group
    \[
    \left\{
    (e_a, e_b), e_c \mapsto (e_{\sigma_1(a)}, e_{\sigma_2(b)}), e_{\sigma_3(c)}: \sigma_1, \sigma_2, \sigma_3 \in \sS_p
    \right\}.
    \]
     Training a two-layer neural network with quadratic activations (such as $f(\theta; x))$ initialized with random i.i.d. parameters using gradient-based update rules (as defined in \Cref{def:gradient_based_training_alg})\footnote{GD is an example of such algorithm.} on the modular addition problem in the classification setting is an equivariant learning algorithm (as defined in \Cref{defi:eqvariance}) with respect to $\cG_{\cX, \cY}$.\footnote{A similar result holds for adaptive algorithms like Adam as well, as discussed in \cref{lem:adam_nn_permutation_invaraince_3d}.}
\end{restatable} 

We say a learning algorithm is \emph{input-output permutation-equivariant} it is equivariant with respect to $\cG_{\cX, \cY}$. The following theorem establishes that any such kernel learning algorithm needs at least $n = \Omega(p^2)$, or equivalently $\frac n N = \Omega(1)$, to generalize on the modular addition task in the classification setting.
\begin{restatable}[Kernel Lower Bound]{theorem}{kernellowerboundclass}
\label{th:kernel_lower_bound_2d}
    There exists a constant $C > 0$ such that for any training data size $n < C p^2$ and any kernel method $\cA$ which is input-output permutation-equivariant (with respect to $\cG_{\cX, \cY}$), it holds that 
    \begin{equation*}
        \E_{ (\vx_i,y_i)_{i=1}^n\sim \cD^n}\E_{\cA}\cL_{\ell_2}  \left(\cA\left( \{\vx_i,y_i\}_{i=1}^n\right) \right) \ge \frac{1}{2} \; \cL_{\ell_2} \left(\Psi_{\boldsymbol{0}}\right),
    \end{equation*}
    where $\E_{\cA}$ takes expectation over the randomness in algorithm $\cA$. 
\end{restatable}

\begin{figure*}[!t]
    \centering
    \begin{subfigure}[b]{0.35\textwidth}
        \includegraphics[width=\textwidth]{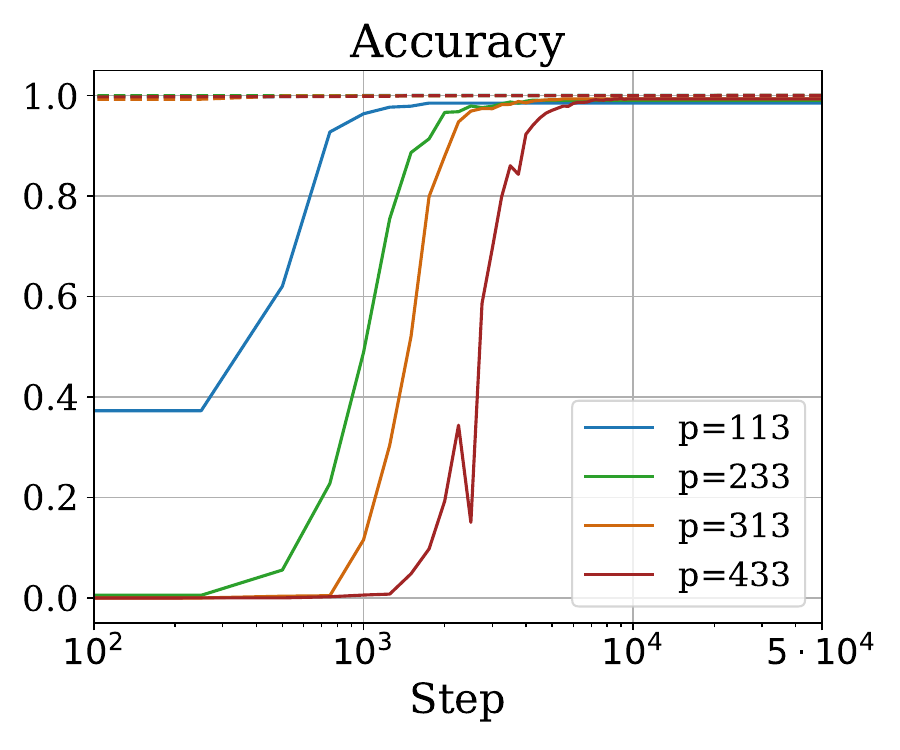}
        \label{fig:class:acc}
    \end{subfigure}
    \hspace{0.05\textwidth}
    \begin{subfigure}[b]{0.35\textwidth}
        \includegraphics[width=\textwidth]{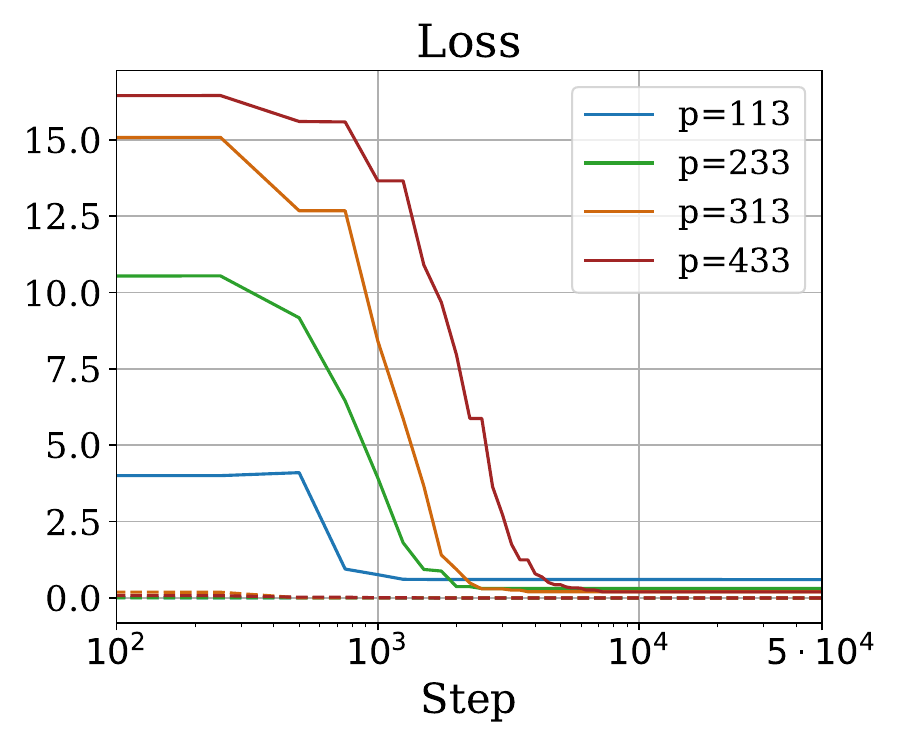}
        \label{fig:class:w_norm}
    \end{subfigure}
    \vspace*{-5mm}
    \caption{
        \textbf{Empirical investigation of grokking in the classification setting with tiny $\ell_\infty$ regularization on different problem dimensions $p$}. Networks are trained with normalized gradient descent\protect\footnotemark[8] on cross-entropy loss and an $\ell_\infty$ regularization of $10^{-20}$ strength. The dashed lines in the indicate train set statistics, and the solid lines correspond to the test set.
    }
    \vspace*{-2mm}
    \label{fig:classification_experiments}
\end{figure*}

Intuitively, the classification setting is not very different from the regression setting. They share the same goal and the amount of knowledge contained in each data in the classification setting is exactly equal to $p$ data in the regression setting, where these $p$ data share the same first two coordinates and only differ in the last coordinate.\footnotetext[8]{We refer the reader to \cref{app:experiment_setup} for the definition of normalized GD.}

More formally, we can define the following correspondence: given one data point in classification $(\vx,y)\in\mathbb{R}^{2p}\times\mathbb{R}^p$, we can view it as $p$ data points in regression, $\{((\vx,e_i),y_i)\}_{i=1}^p$, denoted by $F(\vx,y)$. Similarly, any function $\Psi$ mapping from $\vx = \{(e_i,e_j)\}$ to $\mathbb{R}^p$ can be viewed as a function mapping $\{(e_i,e_j,e_k)\}$ to $\mathbb{R}$ by defining $\Psi'(\vx,e_k)=[\Psi(\vx)]_k$. Moreover, these two functions $\Psi$ and $\Psi'$ share the same population $\ell_2$ loss. Under this view, matrix-valued kernel learning with $n$ classification data points is exactly the same as scalar-valued kernel learning with $np$ regression data points.

The only obstacle to directly applying \Cref{th:kernel_lower_bound_3d} is that the data distribution is different. For the regression setting, each data is sampled independently and uniformly, and the number of data points can be any integer; in the classification setting, the data points are sampled in independent groups and the number of data must be a multiple of $p$. However, it is easy to see this new distribution is still invariant under permutations as in \cref{def:permutation_3d}. Thus, we can directly apply \Cref{th:kernel_lower_bound_3d} on this new distribution to get \Cref{th:kernel_lower_bound_2d}. A formal version of this argument is available in \Cref{app:sec:classification_lower_bound}.

Therefore, for networks operating in the kernel regime, generalization to unseen data in $\ell_2$ loss is impossible unless $n / N = \Omega(1)$,
i.e.\ we have observed a constant fraction of all the $N = p^2$ possible data points. It is worth emphasizing, however, that a large $\ell_2$ loss does not necessarily imply a large classification error. It remains open to prove a classification error lower bound for permutation-equivariant kernel methods.

\subsection{Rich Regime} \label{subsec:classification_rich}

To analyze the behaviour of the network in the rich regime, we first introduce the notion of margin and elaborate on the margin-maximization implicit bias of gradient descent. For a multi-class classification problem with $p$ classes and a fixed network $f$, the margin of a data point $(\vx, y)$ is
\[
q(\theta; \vx, y) \triangleq f_y(\theta; \vx) - \max_{y' \ne y} f_{y'}(\theta; \vx)
.\]
The margin for a dataset $\dset$ is defined as the minimum margin of all points on the dataset: 
\[ 
q_{\min}(\theta;\dset) \triangleq \min_{(\vx, y) \in \dset} q(\theta; \vx, y).
\]

When the network $f$ is homogeneous with respect to its parameters (as is the case in our setup), one can observe that as long as $\theta$ linearly separates the dataset $\dset$ , \emph{i.e.}, $q_{\min}(\theta; \dset) > 0$, it is possible to arbitrary scale the minimum margin, through scaling the parameters of the network. Hence, in such homogenous networks, one is usually concerned with a \textbf{normalized margin} $q_{\min}(\theta/\norm{\theta}; \dset)$ according to some norm.

\citet{lyu2019gradient} proved that gradient descent on homogeneous models with the cross-entropy (or similar) losses on dataset $\dset$,
in the absence of explicit regularization,
maximizes the normalized margin.
Specifically, although $\lVert \theta \rVert_2 \to \infty$ as $t \to \infty$,
$\theta / \lVert \theta \rVert_2$ converges to
a solution (or more generally, a KKT point) of the following problem
when one exists:
\begin{equation} \label{eq:max_margin_def}
    \min \; \frac{1}{2} \norm{\theta}_2^2 \quad \text{s.t.} \quad q_{\min}(\theta; \dset) \ge 1
.\end{equation}

To establish our results in the classification setting, we borrow the following theorem from \citet{wei2020regularization}, who prove that when the strength of the regularization used in \eqref{eq:loss} is small enough, the maximum normalized margin of the regularized loss converges to that of the unregularized loss.

\begin{proposition}[\citealp{wei2020regularization}, Theorem 4.1] \label{prop:wei_convergence}
    Consider a positively homogeneous function $f$ with respect to parameters $\theta$ and a dataset $\dset_\train$ separable with $f$. Let $\lVert \cdot \rVert$ be any norm. Let $\gamma^*$ be the maximum normalized margin of the unregularized loss and $\gamma^\lambda$ be normalized margin of the minimizer of the regularized loss with strength $\lambda$. As $\lambda \to 0$, $\gamma^\lambda \to \gamma^*$.
\end{proposition}

That is, if we use a small enough regularization weight $\lambda$ with any norm in \eqref{eq:loss},
we will obtain approximately the same solution as the unregularized problem. This confirms that training our two-layer network using gradient descent with a small enough $\ell_\infty$ regularizer can lead to weights close to the solution of the max $\ell_\infty$-normalized margin problem. 
Next we present the main result of this subsection, a upper bound on test error, showing that all parameters whose $\ell_\infty$-normalized margin is close enough to the max $\ell_\infty$-normalized margin solution will generalize with a sample complexity of $\tilde{\bigO}(p^{5/3})$

\begin{theorem}[Upper Bound]\label{thm:final_margin_bound} Let $\delta \in (0,1)$ be positive constants such that $8p \le h = O(p)$ independent of $p$ and $n$. For any $n$ representing the size of the training set $\dset_\train$ it holds with probability at least $1-\delta$ over randomness of $\dset_\train$, for all interpolating solutions $\theta$ with normalized $\ell_\infty$-margin $\Omega(p)$, it holds that 
\begin{align}\label{eq:margin_test_error}
L_0(f, \theta, \dset) \le \tilde\bigO \left( \sqrt{\frac{p^{5/3}}{n}} \right).
\end{align}
In other words, any interpolating solution $\theta$ with approximately maximal normalized  $\ell_\infty$-margin, \emph{i.e.}, $  q_{\min}(\theta/\norm{\theta}_\infty; \dset_{\train}) \ge \Omega(\max_{\|\theta'\|_\infty \le 1} q_{\min}(\theta'; \dset_{\train}))$, has a sample complexity of $\tilde{\bigO}(p^{5/3})$, since \Cref{thm:small_inf_norm_exist} shows that the maximal normalized $\ell_\infty$-margin is at least $\Omega(p)$.  
\end{theorem}

Our proof of \Cref{thm:final_margin_bound} is
based on the PAC-Bayesian framework \citep{McAllester2003SimplifiedPM},
specifically using
Lemma 1 of \citet{neyshabur2018a}.
This result provides a margin-based high probability generalization bound for any predictor based on the \emph{margin loss},
which counts a prediction as correct only if it predicts with a margin at least $\gamma$: 
\begin{equation} \label{eq:margin_loss_def}
    L_\gamma \left(f, \theta, \dset \right) \triangleq \Pr_{(\vx, y) \sim \dset} \Big[ q(\theta, \vx, y) > \gamma \Big]. \vspace*{-1mm}
\end{equation}
$L_0$ is simply the misclassification rate, or 0-1 error.
The following statement uses this result to prove an upper bound on the population 0-1 error of networks using margin loss on the training dataset. 

\begin{restatable}{theorem}{classgeneralizationbound}
\label{th:generalization_bound_2d}
For any $p \ge 2$, training set size $n \ge 1, \delta \in (0, 1)$, norm $r > 0$, width $h' > 4 \log \frac{2}{\delta}$ and normalized margin $\gamma/r^3 = \tilde\Omega(\sqrt{h})$ it holds with probability at least $1-\delta$ over the randomness of the training set $\cD_{\operatorname{train}}$ that for any $\theta'$ of width $h'$ with $\norm{\theta'}_\infty \le r$ 
\begin{align}\label{eq:margin_generalization}
L_0(f, \theta', \cD) \le L_{\gamma} (f, \theta', \dset_{\train}) + \tilde\bigO\left(\sqrt{\frac{p}{n}} \cdot \sqrt[3]{\frac{h^2}{\gamma / r^3}}\right)
\end{align}
where $\cD$ denotes the population and $f, L$ are defined in \Cref{eq:two_layer_classification_net,eq:margin_loss_def} respectively.
\end{restatable}

\begin{proof}[Proof Sketch of \Cref{th:generalization_bound_2d}]
Through a series of concentration inequalities, we show that as long as the assumptions of \Cref{th:generalization_bound_2d} are met,  the output logits under gaussian perturbation on parameters, $f(\theta' + \tilde\theta'; \cdot)$,  where $\tilde\theta'$ is entry-wise Gaussian noise of mean zero and variance $\sigma^2$, are close to that of $f(\theta'; \cdot)$ up to an absolute difference of $\tilde\bigO(\sqrt{h}\sigma^3)$. The proof is completed by applying a PAC-bayesian generalization bound, e.g., using Lemma 1 of \citet{neyshabur2018a} and setting $\sigma$ based on the value of $\ell_\infty$-normalized margin $\gamma / r^3$ and width $h$. The full proof is available in \Cref{app:gen_bound}.
\end{proof}

Now we show how to derive \Cref{thm:final_margin_bound} using \Cref{th:generalization_bound_2d}. Because \Cref{th:generalization_bound_2d} only relies on the normalized margin and our model is 3-homogeneous, without loss of generality, we can fix norm bound $r=1$. Then it suffices to set $8p<h=O(p)$ and $\gamma = \Omega(p)$ in \Cref{eq:margin_generalization}, where the first term becomes $0$ because \Cref{thm:final_margin_bound} assumes $\theta$ has normalized margin $\Omega(p)$ and the second term becomes $\tilde O(\sqrt{\frac{p^{5/3}}{n}})$.
\Cref{thm:final_margin_bound} confirms and explains previous observations \citep[e.g.][]{power2022grokking,gromov2023grokking,nanda2023progress,liu2022omnigrok} on the minimum threshold for the fraction of data used to achieve generalization.
In combination with Theorem 4.2 of \citet{lyu2019gradient}, this shows that with enough training data, gradient descent will eventually find a generalizing solution for this setting of the modular addition problem.

\begin{figure*}[t]
    \centering
    \begin{subfigure}[b]{0.32\textwidth}
        \includegraphics[width=\textwidth]{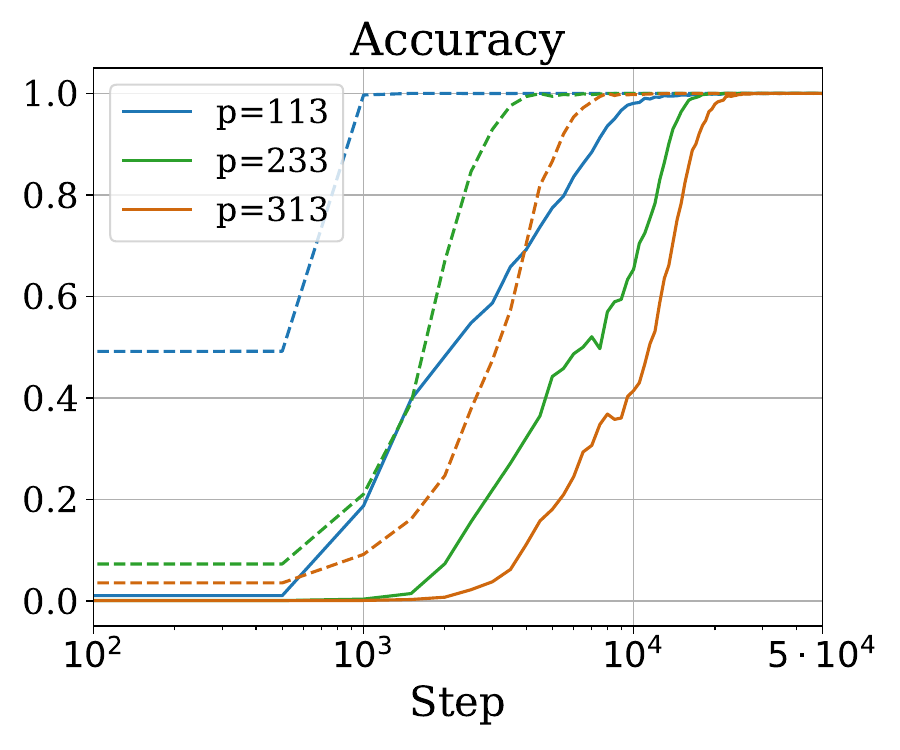}
        \label{fig:transformer:acc}
    \end{subfigure}
    \hfill
    \begin{subfigure}[b]{0.32\textwidth}
        \includegraphics[width=\textwidth]{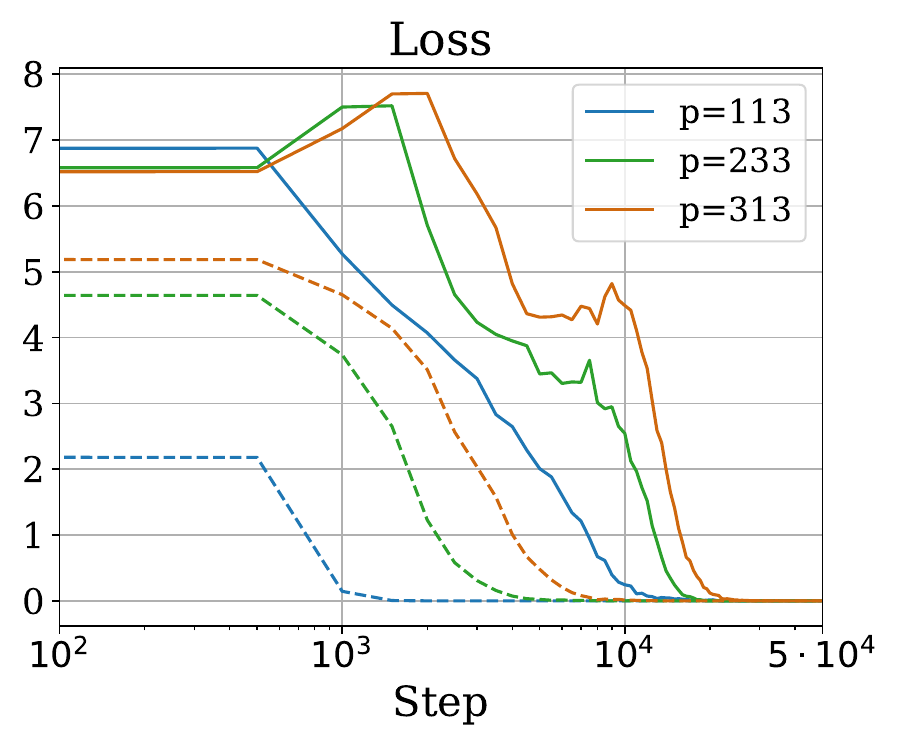}
        \label{fig:transformer:loss}
    \end{subfigure}
    \hfill
    \begin{subfigure}[b]{0.32\textwidth}
        \includegraphics[width=\textwidth]{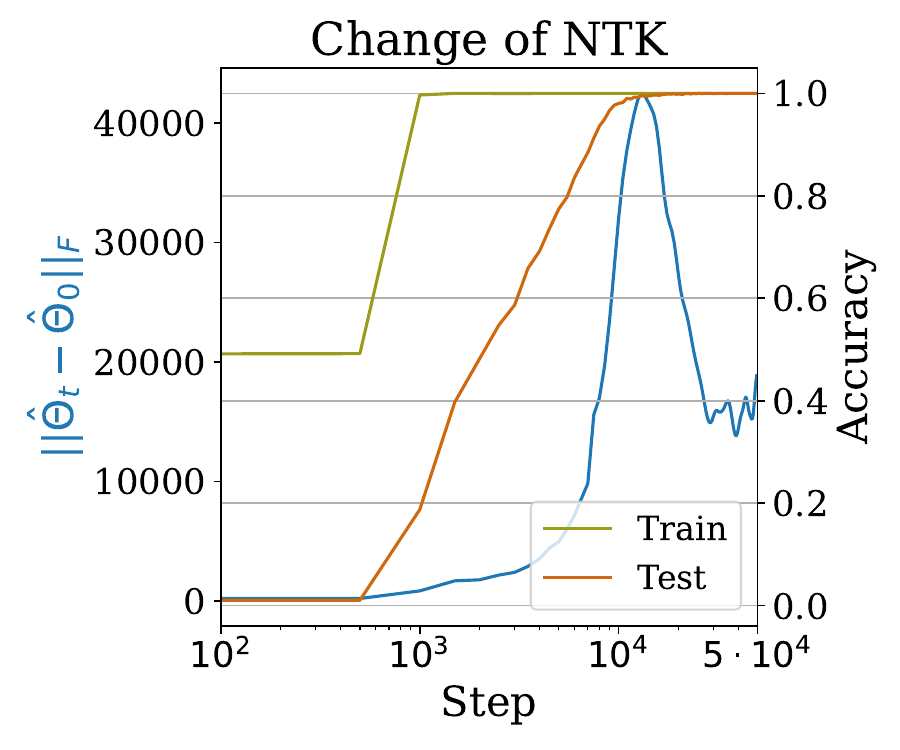}
        \label{fig:transformer:entk}
    \end{subfigure}
\vspace*{-3mm}
    \caption{%
        \textbf{Grokking in transformers happens after a delayed transition from kernel to rich regime.}
        A one-layer transformer is trained with gradient descent using cross-entropy loss and a tiny $\ell_\infty$ regularization of $10^{-20}$ strength on $2 \times p^{5/3}$ training samples from the modular addition problem with various $p$s. Change of eNTK up to the point of fitting the training set is negligible. The eNTK has a drastic change only after fitting the whole training set, implying minimal feature learning until past overfitting. The dashed lines in the middle and left figures indicate train set statistics, and the solid lines correspond to the test set. 
    }
    \vspace*{-2mm}
    \label{fig:transformer}
\end{figure*}

\vspace*{3mm}
\textbf{Empirical Verification of Our Theory:} In terms of whether gradient descent finds a model satisfying these conditions: \cref{fig:classification_experiments} empirically evaluates the $\ell_\infty$ norm of the learned weights through gradient descent on the classification task, with tiny $\ell_\infty$ regularization (weight of $10^{-20}$), across different values of $p$. Again, it can be seen that the $\ell_\infty$ norm of the weights of the network do not grow with the problem dimension. These experiment use $ n = 2 \, p^{5/3}$ data points for each $p$, with a learning rate of 10.

Similar to the regression setting, \cref{fig:summary} presents empirical evidence on the impact of kernel regime on the poor generalization capabilities in the early phase of training by showing the minimal change of empirical NTK until after overfitting in the classification setting. 

Through changing the \emph{scale of initialization}, in \cref{fig:summary} we demonstrate that by decreasing the scale of initialization in the classification task one can mitigate grokking such that the gap between overfitting and generalization becomes larger or smaller. This further supports that as long as the network lies in the kernel regime, generalization without having access to the a constant fraction of the dataset is impossible.

\vspace*{3mm}
\textbf{Grokking Modular Addition in Transformers:} \label{sec:experiments}\Cref{fig:transformer} suggests that, similarly to the two-layer network, grokking in the original Transformer studied by \citet{power2022grokking} might be explainable through the same mechanism. In fact, \cref{th:permutation_equivaraince_3d}, albeit with small modifications to incorporate the shared embeddings in the transformer architecture, applies to transformers as well.

\section{Additional Related Works}
\textbf{Limitations imposed by Equivariance of Learning Algorithms.} \citet{abbe2022nonuniversality} study the impact of equivariance of the training algorithms on the efficiency of learning different functions on different networks. In particular, they consider two main setups: \textbf{a)} learning with FCNs using noisy GD with clipped gradients throughout the training, and \textbf{b)} learning a specific instance of the modular addition task ($p=2$ with noisy inputs) with FCNs using SGD. Although their approach in studying lower bounds for efficient learning shares some high level similarity with ours in using equivariance of the training algorithm, the settings considered are significantly differs from ours. Moreover, in \Cref{app:sec:general_lower_bound} we present a novel abstract framework for analyzing lower bounds on population $\ell_2$ loss for general function classes. \citet{JMLR:v23:20-940} present another technique in analysis of population loss lower bounds, which shares some high-level similarity with our framework, albeit their analysis is more restrictive on function classes. \citet{ng2004feature} also discusses rotational equivariance of many learning algorithms and presents a general lower bound on the 0-1 population error of such algorithms in the general case.

\textbf{Margin Maximization as the Late Phase Implicit Bias.}
\citet{morwani2023feature} present an analytical solution for the max-margin solution of learning modular addition in a classification setting similar to \Cref{sec:classification} when using all of the dataset in training the network. Similar to our analysis of the rich regime in this setting, they face difficulties in proving results under the assumption of bounded $\ell_2$ norm for weights of the network, and assume an $\ell_{2,3}$ bound instead.

\section{Conclusion} \label{sec:discussion}

In this work, we studied the phenomenon of grokking in learning modular addition with gradient descent on two-layer networks, modeled as regression or classification. 
We showed that learning modular addition as presented is fundamentally a difficult task for kernel models (for example neural networks in kernel regime) due to the inherent symmetry and permutation-equivariance of the task.
We theoretically established this difficulty by presenting sample complexity lower bounds of order of constant fraction of the whole dataset.
We further showed that networks satisfying certain conditions generalize far better than those in the kernel regime,
that such networks exist,
and showed empirical evidence that simple regularized gradient descent can eventually find them,
once it escapes the kernel regime.
These results, in combination, attempt to address \textbf{\textit{why}} grokking is observed when learning modular addition.
We provide strong evidence to the hypothesis \citep{lyu2023dichotomy,kumar2023grokking}
that, on this important problem,
it is indeed due to a separation between kernel and non-kernel behavior of gradient descent.

\textbf{Future Work.} We have guaranteed large $\ell_2$ population loss for networks in kernel regime on both regression and classification settings.
It is possible, however, for networks to have arbitrarily high $\ell_2$ loss 
while having perfect classification accuracy.
We conjecture that impossibility results for accuracy-based generalization may also be possible based on permutation equivariance in the early phase of training, but leave it as future work.
Moreover, we only study the cause of grokking in these settings, but do not analyze possible training techniques to enable quick generalization on this task.
Although we are able to eliminate grokking through changing the scale of initialization, doing so actually slows down the time to final generalization; finding practical methods to enable quick generalization would be more useful.

\section*{Acknowledgments}
We would like to thank Kaifeng Lyu and Wei Hu for helpful discussions. This work was enabled in part by support provided by the Natural Sciences and Engineering Research Council of Canada,
the Canada CIFAR AI Chairs program,
Advanced Research Computing at the University of British Columbia,
Calcul Québec,
the BC DRI Group,
and the Digital Research Alliance of Canada. 

\printbibliography

\newpage
\appendix
\section{Experimental Setup} \label{app:experiment_setup}

In this section, we briefly explain the setup used for our experimental evaluations. 

\subsection{Regression}
We use vanilla gradient descent with squared loss and tiny $\ell_\infty$ regularization for 50,000, 100,000 or 200,000 steps for each experiment. In regression, our learning rate has been fixed to 1, and the regularization strength has been set to $10^{-4}$. The network has been initialized according to \citet{he2015delving}. The amount of data used for training in regression task has been set to $2 \times p^{2.25}$.

\subsection{Classification}
We use vanilla gradient descent with cross-entropy loss and tiny $\ell_\infty$ regularization, for up to 100,000 steps. The learning rate in the presented experiments was set to 10 and was kept constant during the training. The regularization strength of $\ell_\infty$ regularizer has been set to $10^{-20}$. 
To accelerate the training with cross-entropy loss, we use the "normalized" GD trick, where the learning rate of each step is scaled by the inverse of the norm of the gradient: 
\begin{equation}
    \theta_{t+1} = \theta_t - \eta \, \frac{\nabla_\theta \ell(\theta_t)}{\norm{\nabla_\theta \ell(\theta_t)}_2} 
\end{equation}
where $\ell$ denotes the loss function and $\eta$ denotes the learning rate. 
The network has been initialized according to \citet{he2015delving}. The amount of data used for training in regression task has been set to $2 \times p^{5/3}$.

\subsection{Transformer}
To train the one-layer transformer we have used full-batch gradient descent with a learning rate of $\eta = 0.25$ and a tiny $\ell_\infty$ regularization with the strength of $1.0e-20$. The network has been initialized according to default PyTorch initialziation (migrated ot JAX). We have used $2 \times p^{5/3}$ of the data for training.

\subsection{Logistics}
We used the JAX framework \citep{jax2018github} to implement and run the experiments on machines using NVIDIA V100 or A100 GPUs.

\section{Ability of Neural Tangent Kernel Models to Interpolate} \label{app:ker-interpolate}

In this section,
we prove \cref{thm:mainbody-ntk},
that neural tangent kernel models for our one-hidden-layer quadratic network are able to exactly interpolate their training data,
in the regression setting.

If the training set contains duplicate $\vx$,
say $\vx_i = \vx_j$ for $i \ne j$,
then $\mathbf K$ cannot be full rank -- the $i$th and $j$th rows are necessarily identical.
Of course, there also exist unattainable targets;
just set $y_i \ne y_j$.
Call a labeling \emph{consistent}
if for all $i$ and $j$ such that $\vx_i = \vx_j$,
$y_i = y_j$.
A learning algorithm can achieve any consistent labeling
if and only if
that algorithm applied to the largest distinct subset
(e.g.\ removing $j$ but keeping $i$)
can achieve any labeling.
Thus, we assume without loss of generality in the remainder of this section
that the training set contains no duplicates.

\subsection{Full-Rank Kernels are Expressive}

Recalling that neural tangent kernel models correspond to ``ridgeless'' kernel regression \citep[e.g.][]{jacot2018neural,lee2019wide},
we first notice that this method can interpolate \emph{any} set of training labels when the kernel is strictly positive definite.

Specifically,
empirical neural tangent kernel regression corresponds to ridgeless regression with a prior mean corresponding to the network at initialization $f_0$,
\[
    \operatornamewithlimits{argmin}_{f : \forall i,\; f(\vx_i) = y_i} \lVert f - f_0 \rVert_{K}
    = f_0 + 
    \operatornamewithlimits{argmin}_{f : \forall i,\; f(\vx_i) = y_i - f_0} \lVert f \rVert_{K}
    = f_0 + \sum_{i=1}^n K(\cdot, \vx_i) \lambda_i
    \;\; \text{ for } \lambda = \mathbf K^\dagger \mathbf{y'}
,\]
where
$\mathbf K^\dagger$ is the Moore-Penrose pseudoinverse of the $n \times n$ kernel matrix $[K(\vx_i, \vx_j)]_{ij}$,
and $\mathbf{y'}$ has $i$th entry $y_i - f_0(\vx_i)$.
By changing variables between $\mathbf{y'}$ and $\mathbf y$,
we need not explicitly consider $f_0$ in the following two results.

\begin{lemma}\label{thm:kernel_interpolating}
    Let $\dset_\train = \{ (x_i, y_i) \}_{i \in [n]}$,
    and let $K$ be a kernel such that the kernel matrix
    $\mathbf K = [K(x_i, x_j)]_{ij}$ is strictly positive definite.
    Then kernel ridgeless regression achieves zero training error on $\dset_\train$.
\end{lemma}
\begin{proof}%
    This well-known result follows from the fact that
    $[\hat f(\vx_i)]_{i} = \mathbf K \mathbf K^\dagger \mathbf y$.
    If $\mathbf K$ is strictly positive definite,
    $\mathbf K^\dagger = \mathbf K^{-1}$,
    and so $[\hat f(x_i)]_{i} = \mathbf y$.
\end{proof}

The following result is a partial converse.
\begin{lemma} \label{thm:krr-noninterp}
    Let $\{ x_i \}_{i \in [n]}$
    and $K$ be a kernel
    such that the kernel matrix
    $\mathbf K = [K(x_i, x_j)]_{ij}$ is singular.
    Then there exists an assignment of $y_i \in \R$
    such that kernel ridgeless regression
    achieves nonzero training error on
    $\dset_\train = \{ (x_i, y_i) \}_{i \in [n]}$.
\end{lemma}
\begin{proof}
    Let $\mathbf y$ be any nonzero vector in the null space of $\mathbf K$; such vectors exist since $\mathbf K$ is singular.
    Then $\mathbf y$ is also in the null space of $\mathbf K^\dagger$,
    and the training set predictions
    are $\mathbf K \mathbf K^\dagger \mathbf y = 0 \ne \mathbf y$.
\end{proof}

\subsection{Expected NTK is Full-Rank}

We next show that, in the regression setting,
the \emph{expected} neural tangent kernel is strictly positive definite.
While perhaps of interest of its own accord,
this will be a key component in our analysis of finite-width networks that follows.

\begin{proposition} \label{prop:expected-kern-interp}
    Let the entries of $W$ follow a distribution $\mathcal P_W$,
    all mutually independent.
    Assume
    $\mathcal P_W$
    has mean zero,
    variance $\sigma_W^2 > 0$,
    skewness zero,
    and kurtosis $\kappa_W$ such that $\E_{w \sim \mathcal P_W} w^4 = \kappa_W \sigma_W^4$.
    The entries of $V$ can be arbitrary.
    Then the expected neural tangent kernel for the regression network with any finite $h \ge 1$
    is strictly positive definite
    on any set of distinct inputs $\{ x_i : i \in [n] \}$,
    with minimum eigenvalue at least $4 h \sigma_W^4$.
\end{proposition}

We first note that if $\mathcal P_W = \mathcal N(0, \sigma_W^2)$,
$\kappa_V = 3$.
If $\mathcal P_W = \Unif([-s_W, s_W])$,
$\sigma_W^2 = \frac13 s_W^2$
and $\kappa_W = \frac95$.
These two distributions cover the vast majority of initialization schemes used in practice.

\begin{proof}
It will be more convenient in this proof to give different names to the sub-matrix of the parameter vector $W$ that act on the $a$ inputs and the $b$ inputs; we will write $W = \begin{bmatrix} Q & R \end{bmatrix}$, where $Q, R$ are each $h \times p$ matrices.
Then we can write the full model as
\[
    g\left( \theta; (e_a, e_b, e_c) \right)
    = e_c\tp V (Q e_a + R e_b)^{\odot 2}
    = \sum_{k=1}^h V_{ck} (Q_{ka} + R_{kb})^2
.\]
The empirical neural tangent kernel
between two inputs
$x = (e_a, e_b, e_c)$
and $x' = (e_{a'}, e_{b'}, e_{c'})$
is then given by
\begin{equation}
\begin{split}
     K_\theta(x, x')
  = \sum_{k=1}^h \Biggl[ \sum_{c''=1}^p & \frac{\partial g(\theta; x)}{\partial V_{c''k}} \frac{\partial g(\theta; x')}{\partial V_{c''k}}
\\
   &+ \sum_{a''=1}^p \frac{\partial g(\theta; x)}{\partial Q_{k a''}} \frac{\partial g(\theta; x')}{\partial Q_{ka''}}
   + \sum_{b''=1}^p \frac{\partial g(\theta; x)}{\partial R_{k b''}} \frac{\partial g(\theta; x')}{\partial R_{kb''}}
   \Biggr]
.\end{split}
\label{eq:entk-gen}
\end{equation}

Evaluating the $V$ derivatives,
\begin{align*}
     \frac{\partial g(\theta; x)}{\partial V_{c'' k}}
  &= \begin{cases}
      (Q_{k a} + R_{k b})^2  & \text{if } c = c''
      \\
      0 & \text{otherwise}
  \end{cases}
\\
     \sum_{c''=1}^p \frac{\partial g(\theta; x)}{\partial V_{c''k}} \frac{\partial g(\theta; x')}{\partial V_{c''k}}
  &= \begin{cases}
      (Q_{k a} + R_{k b})^2 (Q_{k a'} + R_{k b'})^2  & \text{if } c = c'
      \\
      0 & \text{otherwise}
  .\end{cases}
\end{align*}

While it would not be difficult to analyze the $Q$ and $R$ derivatives as well,
their exact form will not be important.
Instead, we only need to write
\begin{align}
     K_\theta(x, x')
  &= J_{\theta}(x, x')
   + \sum_{k=1}^h L_{\theta_k}(x, x')
\label{eq:entk}
\\   J_{\theta}(x, x')
   &=
   \sum_{k=1}^h
   \sum_{a''=1}^p \frac{\partial g(\theta; x)}{\partial Q_{k a''}} \frac{\partial g(\theta; x')}{\partial Q_{ka''}}
   + \sum_{b''=1}^p \frac{\partial g(\theta; x)}{\partial R_{k b''}} \frac{\partial g(\theta; x')}{\partial R_{kb''}}
\nonumber
\\   L_{\theta_k}(x, x')
  &= (Q_{ka} + R_{kb})^2 (Q_{ka'} + R_{kb'})^2 \ind(c = c')
\label{eq:entk-l}
.\end{align}
The function $J_\theta$ has an explicit feature map corresponding to the relevant gradients;
thus, for any set of inputs $\{ x_i : i \in [n] \}$ and any value of $\theta$,
the kernel matrix $\mathbf J_\theta = [J_\theta(x_i, x_j)]_{ij}$
is positive semi-definite.
This implies,
e.g.\ via Weyl's inequality,
that $\E_\theta \mathbf J_\theta$ is also positive semi-definite.

For any fixed set of inputs $\{ x_i : i \in [n] \}$,
the kernel matrices
$\mathbf L_{\theta_k} = [ L_{\theta_k}(\vx_i, \vx_j) ]_{ij}$
are independent and identically distributed,
and in particular have the same mean
$\E \mathbf L_{\theta_k}$ for any arbitrary choice of $k \in [p]$.
Thus the expected neural tangent kernel matrix can be written as
\[
    [ \E K_\theta(x_i, x_j) ]_{ij}
    = \E \mathbf J_\theta + h \E \mathbf L_{\theta_k}
.\]
We will show that $\E \mathbf L_{\theta_k}$ has minimum eigenvalue at least $4 \sigma_W^4$,
from which the result follows.

To do so, we will evaluate
\[
     \E L_{\theta_k}(x, x')
  = \E \Bigl[ (Q_{ka} + R_{kb})^2 (Q_{ka'} + R_{kb'})^2
   \Bigr] \ind(c = c') 
\]
for arbitrary nonrandom
$\vx = (e_a, e_b, e_c)$ and $\vx' = (e_{a'}, e_{b'}, e_{c'})$,
where all expectations will be over the relevant parameters $\{ Q_{ka} : a \in [p] \} \cup \{ R_{kb} : b \in [p] \}$.

If $a = a'$ and $b = b'$,
\begin{align*}
     \E\left[ (Q_{ka} + R_{kb})^4 \right]
  &=   \underbrace{\E Q_{ka}^4}_{\kappa_W \sigma_W^4}
   + 4 \underbrace{\E Q_{ka}^3}_0 \underbrace{\E R_{kb}}_0
   + 6 \underbrace{\E Q_{ka}^2}_{\sigma_W^2} \underbrace{\E R_{kb}^2}_{\sigma_W^2}
   + 4 \underbrace{\E Q_{ka}}_0 \underbrace{\E R_{kb}^3}_0
   +   \underbrace{\E R_{kb}^4}_{\kappa_W \sigma_W^4}
\\&= (2 \kappa_W + 6) \sigma_W^4
.\end{align*}

If instead $a = a'$ but $b \ne b'$,
then
\begin{align*}
     \E\left[ (Q_{ka} + R_{kb})^2 (Q_{ka} + R_{kb'})^2\right]
  &= \E\left[ (Q_{ka}^2 + 2 Q_{ka} R_{kb} + R_{kb}^2) (Q_{ka}^2 + 2 Q_{ka} R_{kb'} + R_{kb'}^2) \right]
\\&= \underbrace{\E Q_{ka}^4}_{\kappa_W \sigma_W^4}
   + 2 \underbrace{\E Q_{ka}^3}_0 \underbrace{\E R_{kb'}}_0
   + \underbrace{\E Q_{ka}^2}_{\sigma_W^2} \underbrace{\E R_{kb'}^2}_{\sigma_W^2}
\\&\qquad
   + 2 \underbrace{\E Q_{ka}^3}_0 \underbrace{\E R_{kb}}_0
   + 4 \underbrace{\E Q_{ka}^2}_{\sigma_W^2} \underbrace{\E R_{kb}}_0 \underbrace{\E R_{kb'}}_0
   + 2 \underbrace{\E Q_{ka}}_0 \underbrace{\E R_{kb}}_0 \underbrace{\E R_{kb'}^2}_{\sigma_W^2}
\\&\qquad
   + \underbrace{\E R_{kb}^2}_{\sigma_W^2} \underbrace{\E Q_{ka}^2}_{\sigma_W^2}
   + 2 \underbrace{\E R_{kb}^2}_{\sigma_W^2} \underbrace{\E Q_{ka}}_0 \underbrace{\E R_{kb'}}_0
   + \underbrace{\E R_{kb}^2}_{\sigma_W^2} \underbrace{\E R_{kb'}^2}_{\sigma_W^2}
\\&= \left( \kappa_W + 3 \right) \sigma_W^4
;\end{align*}
the case where $a \ne a'$ but $b = b'$ is the same by symmetry.

Finally, when $a \ne a'$ and $b \ne b'$,
we have by independence that
\begin{align*}
    \E\bigl[ (Q_{ka} + R_{kb})^2 (Q_{ka'} + R_{kb'})^2 \bigr]
    &= \E\bigl[ (Q_{ka} + R_{kb})^2 \bigr] \E\bigl[ (Q_{ka'} + R_{kb'})^2\bigr]
     = \Bigl( \E\bigl[ (Q_{ka} + R_{kb})^2 \bigr] \Bigr)^2
\\  &= \Bigl( \underbrace{\E Q_{ka}^2}_{\sigma_W^2}
    + 2 \underbrace{\E Q_{ka}}_0 \underbrace{\E R_{kb}}_0
    + \underbrace{\E R_{kb}^2}_{\sigma_W^2} \Bigr)^2
     = (2 \sigma_W^2)^2 = 4 \sigma_W^4
.\end{align*}

Combining the cases, it holds in general that
\begin{align*}
     \E L_{\theta_k}(x, x')
  &= \sigma_W^4 \ind(c = c') \begin{cases}
      4 & \text{if } a \ne a', b \ne b' \\
      \kappa_W + 3  & \text{if } a = a', b \ne b' \\
      \kappa_W + 3 & \text{if } a \ne a', b = b' \\
      2 \kappa_W + 6 & \text{if } a = a', b = b' \\
  \end{cases}
\\&= 4 \sigma_W^4 \ind(c = c')
\\&\quad+ (\kappa_W - 1) \sigma_W^4 \ind(a = a', c = c')
\\&\quad+ (\kappa_W - 1) \sigma_W^4 \ind(b = b', c = c')
\\&\quad+ 4 \sigma_W^4 \ind(a = a', b = b', c = c')
.\end{align*}
The kurtosis of any probability distribution is at least 1 by Jensen's inequality,
so all the coefficients in this last form are nonnegative.
We can then use this to construct an explicit feature map
for each term in $\E L_{\theta_k}$,
because if $\gamma \ge 0$,
\begin{equation}
    \gamma \ind(c = c')
    = \begin{bmatrix} \sqrt\gamma \ind(c = 1) \\ \sqrt\gamma \ind(c = 2) \\ \vdots \\ \sqrt\gamma \ind(c = p) \end{bmatrix}\tp
    \begin{bmatrix} \sqrt\gamma \ind(c' = 1) \\ \sqrt\gamma \ind(c' = 2) \\ \vdots \\ \sqrt\gamma \ind(c' = p) \end{bmatrix}
\label{eq:explicit-features}
\end{equation}
corresponds to explicit features in $\R^p$.
The other indicator functions can be implemented in the same way,
in $\R^{p^2}$ or $\R^{p^3}$;
their sum can then be obtained by concatenating the individual features together.
This construction makes it clear that the function
\[
     M(x, x')
  = 4 \sigma_W^4 \ind(c = c')
   + (\kappa_W - 1) \sigma_W^4 \ind(a = a', c = c')
   + (\kappa_W - 1) \sigma_W^4 \ind(b = b', c = c')
\]
is a positive semi-definite kernel,
so $\mathbf M = [M(x_i, x_j)]_{ij}$ is a positive semi-definite matrix for any set of inputs $\{ x_i \}$.

It remains to show that the final component of the kernel is strictly positive definite.
Noticing that $\ind(a = a', b = b', c = c') = \ind(x = x')$,
if the $\{ \vx_i : i \in [n] \}$ are distinct,
the expected kernel matrix is $\E \mathbf L_{\theta_k} = \mathbf M + 4 \sigma_W^4 \mathbf I$.
Since $\mathbf M$ is positive semi-definite,
this has minimum eigenvalue at least $4 \sigma_W^4 > 0$,
as desired.
\end{proof}

\subsection{Empirical NTKs are Likely Full-Rank}
Now, for \emph{bounded} initialization schemes,
we use matrix concentration inequalities to show that
the empirical neural tangent kernel is also likely to be full-rank when $h$ is large enough.
\begin{proposition} \label{prop:entk-interp}
    In the setting of \cref{prop:expected-kern-interp},
    further assume that $\Pr_{w \sim \mathcal P_W}(\abs{w} \le s_W) = 1$.
    Let the set of inputs $\{ x_i : i \in [n] \}$
    be distinct
    and such that the most common $c$ value is seen $\rho_c$ times.
    Then the empirical neural tangent kernel
    for the regression network is strictly positive definite with probability at least $1 - \delta$ over the choice of random parameters $\theta$
    as long as
    \[
        h > \frac{4 s_W^4}{\sigma_W^4} \rho_c \log \frac{n}{\delta}
    .\]
    If $\mathcal P_W$ is uniform on $[-s_W, s_W]$ for $s_W > 0$, this condition is equivalent to
    \[
        h > 36 \rho_c \log \frac{n}{\delta}
    .\]
\end{proposition}

\begin{proof}
    Recall the decomposition of
    $K_\theta(\vx, \vx') = J_\theta(\vx, \vx') + \sum_{k=1}^h L_{\theta_k}(\vx, \vx')$ from \eqref{eq:entk},
    and the corresponding $n \times n$ matrices
    $\mathbf K_\theta = [ K_\theta(\vx_i, \vx_j) ]_{ij}$,
    $\mathbf J_\theta = [ J_\theta(\vx_i, \vx_j) ]_{ij}$,
    and
    $\mathbf L_{\theta_k} = [ L_{\theta_k}(\vx_i, \vx_j) ]_{ij}$,
    so that $\mathbf K_\theta = \mathbf J_\theta + \sum_{k=1}^h \mathbf L_{\theta_k}$.
    As shown in the proof of \cref{prop:expected-kern-interp}, $\mathbf J_\theta$ is positive semi-definite;
    we now wish to show that the sum of $h$ iid matrices
    $\mathbf L_{\theta_k}$,
    whose mean
    $h \E \mathbf L_{\theta_k}$ has minimum eigenvalue $\mu_{\min} := 4 h \sigma_W^4$,
    is likely to be full-rank.
    We can do so by applying \cref{thm:matrix-full-rank-chernoff},
    a direct corollary of matrix Chernoff bounds,
    if we additionally have an almost sure upper bound on the operator norm $\lVert \mathbf L_{\theta_k} \rVert$.

    Notate the input $\vx_i$ as $(e_{a_i}, e_{b_i}, e_{c_i})$ for $a_i, b_i, c_i \in [p]$.
    Using \eqref{eq:entk-l},
    we can write
    \begin{equation}
        \mathbf L_{\theta_k}
        = \diag(\mathbf w_{\theta_k}) \mathbf C \diag(\mathbf w_{\theta_k})
    \label{eq:entk-l-c}
    ,\end{equation}
    where 
    $\diag : \R^n \to \R^{n \times n}$
    constructs a diagonal matrix with the given vector on its diagonal and
    \begin{gather*}
        (\mathbf C)_{ij} = \ind(c_i = c_j)
    \\
        (\mathbf w_{\theta_k})_i = (Q_{k a_i} + R_{k b_i})^2
    .\end{gather*}
    Using $\lVert A B \rVert \le \lVert A \rVert \lVert B \rVert$
    and
    $\lVert \diag(v) \rVert = \max_i \lvert v_i \rvert$,
    we obtain that
    \[
           \lVert \mathbf L_{\theta_k} \rVert
       \le (2 s_W)^2 \lVert \mathbf C \rVert (2 s_W)^2
         = 16 s_W^4 \lVert \mathbf C \rVert
    .\]
    Using \eqref{eq:explicit-features},
    we can write
    $\mathbf C = \Phi_c \Phi_c\tp$
    so that $\lVert \mathbf C \rVert = \lVert \Phi_c \rVert^2$,
    where $\Phi_c \in \R^{n \times p}$ is given by
    \[
        \Phi_c = \begin{bmatrix}
            \ind(c_1 = 1) & \cdots & \ind(c_1 = p) \\
            \vdots & \ddots & \vdots \\
            \ind(c_n = 1) & \cdots & \ind(c_n = p)
        \end{bmatrix}
    .\]
    We can evaluate this operator norm with
    \begin{align*}
        (\Phi_c x)_i
        &= \sum_{\ell=1}^p \ind(c_i = \ell) x_\ell
         = x_{c_i}
    \\
        \lVert \Phi_c x \rVert^2
        &= \sum_{i=1}^n x_{c_i}^2
         = \sum_{\ell = 1}^p x_\ell^2 \left( \sum_{i=1}^n \ind(c_i = \ell) \right)
    \\
        \lVert \Phi_c \rVert^2
        &= \sup_{\lVert x \rVert = 1}
         \sum_{\ell = 1}^p x_\ell^2 \left( \sum_{i=1}^n \ind(c_i = \ell) \right)
         = \max_{\ell \in [p] } \sum_{i=1}^n \ind(c_i = \ell)
         = \rho_c
    .\end{align*}
    Thus it holds almost surely that
    \[
        \lVert \mathbf L_{\theta_k} \rVert
        \le 16 s_W^4 \rho_c
        =: L
    .\]
    Plugging in \cref{thm:matrix-full-rank-chernoff}, we have shown that
    \[
        \Pr\left( \lambda_{\min}\left( \mathbf K_{\theta} \right) > 0 \right)
        \ge 1 - n \exp\left( - \frac{4 h \sigma_W^4}{16 s_W^4 \rho_c} \right)
    ,\]
    and so $\mathbf K_\theta$ is full rank with probability at least $1 - \delta$ as long as
    \[
        h > \frac{4 s_W^4}{\sigma_W^4} \rho_c \log \frac n \delta
    .\]
    If the entries of $W$ are iid $\Unif([-s_w, s_w])$ for $s_w > 0$,
    $\sigma_W^2 = \frac13 s_W^2$,
    giving the condition
    \[
        h > 36 \rho_c \log \frac n \delta
    .\qedhere\]
\end{proof}

\begin{lemma} \label{thm:matrix-full-rank-chernoff}
    Consider a finite sequence of independent positive semi-definite $d \times d$ real random matrices $\mathbf X_1, \dots, \mathbf X_m$.
    Assume that $\lVert \mathbf X_k \rVert \le L$ almost surely,
    and that $\mu_{\min} = \lambda_{\min}\left( \sum_{k=1}^m \E \mathbf X_k \right) > 0$.
    Then
    \[
        \Pr\left( \lambda_{\min}\left( \sum_{k=1}^m \mathbf X_k \right) > 0 \right) \ge 1 - d \exp\left( - \frac{\mu_{\min}}{L} \right)
    .\]
\end{lemma}
\begin{proof}
    A matrix Chernoff bound \citep[Theorem 5.1.1]{tropp:intro}
    in gives that
    for each $\eps \in [0, 1)$,
    \begin{equation} \label{eq:matrix-chernoff}
        \Pr\left( \lambda_{\min}\left( \sum_{k=1}^h \mathbf X_k \right) > (1 - \eps) \mu_{\min} \right)
        \ge 1 - d \left( \frac{e^{-\eps}}{(1-\eps)^{1-\eps}} \right)^{{\mu_{\min}}/{L}}
    .\end{equation}
    If the smallest eigenvalue is strictly positive,
    it must exceed some $\delta > 0$,
    which corresponds to some $\eps < 1$.
    Thus the event of being full-rank
    is the limit of the events of exceeding each $\eps$,
    and since the right-hand side of \eqref{eq:matrix-chernoff} is continuous for $\eps < 1$,
    we can take the limit as $\eps \nearrow 1$;
    since ${e^{-\eps}}/{(1-\eps)^{1-\eps}} \to 1/e$,
    this gives the desired result.
\end{proof}

This threshold is indeed tight up to logarithmic factors,
in the sense that there are some labels which kernel ridgeless regression cannot achieve
when $h = o(n / p)$.

\begin{proposition} \label{thm:entk-singular}
    In the setting of \cref{prop:expected-kern-interp},
    the empirical neural tangent kernel of the regression network $\mathbf K_\theta$
    is guaranteed to be singular when $h < n / (3 p)$.
\end{proposition}
\begin{proof}
Expanding on \eqref{eq:entk-gen},
we will additionally need to evaluate
the $Q$ and $R$ derivatives:
\begin{align*}
     \frac{\partial g(\theta; \vx)}{\partial Q_{k a''}}
  &= \begin{cases}
      2 V_{ck} (Q_{ka} + R_{kb})  & \text{if } a = a''
      \\
      0 & \text{otherwise}
  \end{cases}
\\
   \sum_{a''=1}^p \frac{\partial g(\theta; \vx)}{\partial Q_{k a''}} \frac{\partial g(\vx')}{\partial Q_{ka''}}
   &= \begin{cases}
       4 V_{ck} V_{c'k} (Q_{ka} + R_{kb}) (Q_{ka} + R_{kb'})
         & \text{if } a = a' \\
       0 & \text{otherwise}
   \end{cases}
\\
     \frac{\partial g(\theta; \vx)}{\partial R_{k b''}}
  &= \begin{cases}
      2 V_{ck} (Q_{ka} + R_{kb})  & \text{if } b = b''
      \\
      0 & \text{otherwise}
  \end{cases}
\\
     \sum_{b''=1}^p \frac{\partial g(\vx)}{\partial R_{kb''}} \frac{\partial g(\vx')}{\partial R_{kb''}}
   &= \begin{cases}
       4 V_{ck} V_{c'k} (Q_{ka} + R_{kb}) (Q_{ka'} + R_{kb})
         & \text{if } b = b' \\
       0 & \text{otherwise}
    .\end{cases}
\end{align*}
Writing as in \eqref{eq:entk} and \eqref{eq:entk-l-c},
we have that
\[
    \mathbf K_\theta
    = \sum_{k=1}^h
        \diag(\mathbf w_{\theta_k}) \mathbf C \diag(\mathbf w_{\theta_k})
        + \diag(\mathbf v_{\theta_k}) \mathbf A \diag(\mathbf v_{\theta_k})
        + \diag(\mathbf v_{\theta_k}) \mathbf B \diag(\mathbf v_{\theta_k})
,\]
where $(\mathbf v_{\theta_k})_i = 2 V_{ck} (Q_{ka} + R_{kb})$,
$\mathbf A_{ij} = \ind(a_i = a_j)$,
and $\mathbf B_{ij} = \ind(b_i = b_j)$.
By \eqref{eq:explicit-features},
each of the matrices $\mathbf A$, $\mathbf B$, and $\mathbf C$
have rank at most $p$;
thus $\rank(\mathbf K_\theta) \le 3 p h$.
If $3 p h < n$, this $n \times n$ matrix cannot be full-rank.
\end{proof}

\section{Permutation-Equivariance of Gradient-Based Training} \label{app:sec:permutation_equivariance}

We first define the notion of permutation equivariance. Towards this, we borrow the following proof from Appendix C of \citet{li2020convolutional}:

\begin{definition}[Gradient-Based Algorithm $\cA$] \label{def:gradient_based_training_alg}
We borrow the definition of Algorithm 1\danica{Should write out the actual definition to be self-contained, it's not very long}{} in \citet{li2020convolutional} with the slight modification of restricting the update rule $F(U, \cM, \dset_{\operatorname{train}})$ where $U$ denotes the parameters and $\cM: U \to (\cX \to \R)$ denotes the model mapping parameters to a function. We restrict the update rule to only allow gradient-based update rules, such that
\[
F(U, \cM, \dset_{\operatorname{train}}) = U - \eta \nabla_U \cL(\cM(U), \dset_{\operatorname{train}})
\]
for some $\eta \in \R$ and loss function $\cL$ mapping a function and a dataset to a loss $(\cX \to \R) \times \dset \to \R$.
\end{definition}

\begin{theorem} \label{th:permutation_equivariance_definition}
    Suppose $\cG_\cX$ is a group acting on $\cX$. The gradient-based iterative algorithm $\mathcal{A}$ (defined in \Cref{def:gradient_based_training_alg}) is $\cG_\cX$-equivariant if: \zhiyuan{you didn't copy the definition correct. 1. do not use $\theta$, give parameter space some name, e.g., $\mathcal{\Theta}$. 2. the defined algorithm $A$ is not used. the transformation should be on the updated parameters, but not grad. 3. We need a theorem saying Adam is permutation equivariant.}
    \begin{enumerate}
        \item There exists a group $\cG_\theta$ acting on parameters $U$ and a group isomorphism $\tau: \cG_\cX \to \cG_\theta$ such that for all $x \in \cX, T \in \cG_\cX, U$ we have that $g(U; x) = g(\tau(T)(U); T(x))$.
        \item The gradient update rule is invariant under joint group action $(T, \tau(T))$ for all $T \in \cG_\cX$: $\tau(T)(\nabla_U g(U; x)) = \nabla_U g(\tau(T)(U); T(x))$. 
        \item The initialization distribution $P_{\operatorname{init}}$ is invariant under $\cG_\cX$.
    \end{enumerate}
\end{theorem}

We now present our permutation-equivariance results.

\begin{definition} \label{def:weight_permutation_3d}
We define $\cG_\theta$ as a group of actions to be applied on $U = (W, V)$ as 
\[
\cG_\theta \triangleq \left\{\pi_{\sigma_1, \sigma_2, \sigma_3}\mid \sigma_1, \sigma_2, \sigma_3 \in \mathbb{S}_p \right\}
\]
where $\pi_{\sigma_1, \sigma_2, \sigma_3}$ is defined as $\pi_{\sigma_1, \sigma_2, \sigma_3}(W,V)\triangleq (W',V')$, where  $\forall i\in[h],j\in[p], W'_{ij} = W_{i, \sigma_1(j)}, W_{i(j+p)}' = W_{i, \sigma_2(j)},V'_{\sigma_3(j),i} = V_{j, i}$. This means the concatenation of the row is permuted in the same way as the data does. Furthermore, the rows of $V$ also get permuted according to $\sigma_3$.
\end{definition}

\begin{remark} \label{remark:bijection_of_group_actions}
    There is a one-to-one mapping $\tau$ between $\cG_\theta$ and $\cG_\cX$, which is $\tau(\sigma_1,\sigma_2,\sigma_3) = \pi_{\sigma_1,\sigma_2,\sigma_3}$.
\end{remark}

\begin{lemma} \label{lemma:permutaiton_equivariant_forward_3d}
For every $x = (i, j, k)$ where $i, j, k \in [p]$, $T \in \cG_\cX$ we have 
\[
g(U; x) = g(\tau(T)(U); T(x)).
\]
\end{lemma}
\begin{proof}
    For each $x = (i, j, k) \in \cX; \sigma_1, \sigma_2, \sigma_3 \in \mathbb{S}_p; U = (W, V)$ we have that 

    \begin{align*}
        g(\tau(\sigma_1, \sigma_2, \sigma_3)(U); (\sigma_1,\sigma_2,\sigma_3)(x)) &= \inner{e_{\sigma_3(k)}}{ V'_{\sigma_3} \left( \begin{bmatrix}
            W_1'^\top (\sigma_1(i), \sigma_2(j)) \\ W_2'^\top (\sigma_1(i), \sigma_2(j)) \\ \vdots \\ W_h'^\top (\sigma_1(i), \sigma_2(j))
        \end{bmatrix} \right)^{\odot 2}}\\ 
        &= \inner{e_k}{V \left( \begin{bmatrix}
            W_{1,\sigma_1(i)}' + W_{1,\sigma_2(j)+p}' \\ W_{2,\sigma_1(i)}' + W_{2,\sigma_2(j)+p}' \\ \vdots \\ W_{h,\sigma_1(i)}' + W_{h,\sigma_2(j)+p}'
        \end{bmatrix} \right)^{\odot 2}}\\ 
        &= \inner{e_k}{V \left( \begin{bmatrix}
            W_{1,i} + W_{1,j+p} \\ W_{2,i} + W_{2,j+p} \\ \vdots \\ W_{h,i} + W_{h,j+p}
        \end{bmatrix} \right)^{\odot 2}}\\ 
        &= \inner{e_k}{V (Wx)^{\odot 2}} \\
        &= g(U; x).
    \end{align*}    
\end{proof}

\begin{lemma} \label{lemma:permutaiton_equivariant_grad_3d}
For every $x = (i, j, k)$ where $i, j, k \in [p]$, $T \in \cG_\cX$ we have 
\[
\tau(T)(\nabla_U g(U; x)) = \nabla_U g(\tau(T)(U); T(x)).
\]
\end{lemma}

\begin{proof}
We first consider the gradient of the second layer. For all $a, b \in [p]$ and $\sigma_1, \sigma_2, \sigma_3 \in \mathbb{S}_p$:

\begin{align*}
    &\tau^{-1}(\pi_{\sigma_1, \sigma_2, \sigma_3}) \left(\nabla_{V_b} g(\tau(\sigma_1, \sigma_2, \sigma_3)(U); (\sigma_1,\sigma_2,\sigma_3)(x)) \right) \\
    &= I(\sigma_3(k) = \sigma_3(b)) \nabla_{V_b} g \left(\tau(\sigma_1, \sigma_2, \sigma_3)(U); (\sigma_1,\sigma_2, \sigma_3)(x) \right) \\
    &= I(k=b) (W' (\sigma_1, \sigma_2)(x))^{\odot 2} \\
    &= I(k = b) (Wx)^{\odot 2} \\
    &= \nabla_{V_b} g(U; x).
\end{align*}

For the gradients of the first layer we have:

\begin{align*}
    &\tau^{-1}(\pi_{\sigma_1, \sigma_2, \sigma_3}) \left( \nabla_W g (\tau(\sigma_1, \sigma_2, \sigma_3)(U); (\sigma_1,\sigma_2, \sigma_3)(x)) \right) \\
    &= \tau^{-1}(\pi_{\sigma_1, \sigma_2, \sigma_3}) \left( 2V_k \odot (W'_{(\sigma_1, \sigma_2)}(e_{\sigma_1(i)}, e_{\sigma_2(j)})^\top) \; (e_{\sigma_1(i)}, e_{\sigma_2(j)}) \right) \\
    &= \tau^{-1}(\sigma_1, \sigma_2,\sigma_3) \left( 2V_k \odot (W (e_i, e_j)) \; (e_{\sigma_1(i)}, e_{\sigma_2(j)})^\top \right) \\
    &= 2V_k \odot (Wx) \; (\sigma_1, \sigma_2)^{-1} \left((e_{\sigma_1(i)}, e_{\sigma_2(j)})\right) \\
    &= 2 V_k \odot (W(e_i, e_j)^\top) \; (e_i, e_j) \\
    &= \nabla_{W} g(U; x).
\end{align*}
\end{proof}

\regressionpermeq*

\begin{proof}
We show that training a neural network with gradient descent in our setting satisfies the three conditions proposed in \cref{th:permutation_equivariance_definition} and thus, this theorem applies to the training process. Equivariance of the forward pass and the backward pass (update rule) are settled through \cref{lemma:permutaiton_equivariant_forward_3d,lemma:permutaiton_equivariant_grad_3d}. Finally, it is straightforward to see that the initialization is invariant under $\cG_\theta$ defined in \cref{def:weight_permutation_3d}. Since the distribution of initialization is symmetric and each parameter is initialized independently and identically from the same distribution, the action of swapping rows or columns doesn't change the distribution. 
\end{proof}

\classificationpermeq*

\begin{proof}[Proof Sketch]
Consider the permutation group $\cG_\theta$ defined in \Cref{def:weight_permutation_3d}. Note that there is a one-to-one mapping $\kappa$ mapping $\cG_{\cX,\cY}$ defined in \Cref{def:permutation_3d} to $\cG_{\cX}$ and another one-to-one mapping $\tau$ mapping $\cG_{\cX}$ to $\cG_\theta$ where $\tau(\sigma_1, \sigma_2, \sigma_3) = \pi_{\sigma_1, \sigma_2, \sigma_3}$. Since $f$ and $g$ share the parameters, it's clear that the distribution of initialization for $f$ is equivariant under $\cG_\theta$. To see the equivariance of forward and backward passes on $f(U; \cdot)$ it suffices to see that \Cref{lemma:permutaiton_equivariant_forward_3d,lemma:permutaiton_equivariant_grad_3d} hold for ($\cG_{\cX, \cY}, \cG_\theta$) and function $f$ under the isomorphism $\tau \circ \kappa$ since for any parameters $U$ and inputs $i, j, k \in [p]$ it holds that $f_k(U; (i, j)) = g(U; (i, j, k))$. 
\end{proof}

\begin{corollary}[Equivariance of Adam] \label{lem:adam_nn_permutation_invaraince_3d}
\Cref{th:permutation_equivaraince_3d} (and similarly \Cref{prop:permutation_equivaraince_2d}) applies to other gradient-based training algorithms that need memory, such as Adam. 
\end{corollary}

\begin{proof}[Proof Sketch for \Cref{lem:adam_nn_permutation_invaraince_3d}]
To prove that other gradient-based algorithms, particularly Adam, are also equivariant to the permutation group $\cG_\cX$ we just need to ensure that the update rule of these algorithms is equivariant under the joint group action $(T, \tau(T))$. First, note that linear operations (such as weight decay) on gradients and parameters equivariant. To track the momentum at different steps of the algorithm we can apply an induction on equivariance of these variables on the step number. At $t=0$ they're both zero. $m_t$ is a linear combination $m_{t-1}$ and $g_t$ which both are equivarint under the joint action $(T, \tau(T))$. Since the gradient is equivariant, the coordinate-wise squared gradient is also equivariant and the linear combination of it with $v_{t_1}$ is also equivariant. This settles the equivariance of the update rule of Adam and other similar gradient-based algorithms that need memory and buffers. 
\end{proof}

\section{Lower Bound of Population Loss for Kernel Methods} \label{app:sec:kernel_lower_bound_3d}

In this section we present the formal version of \cref{th:kernel_lower_bound_3d} alongside a proof of it. Note that a kernel-based predictor $h$ on a training data $\left\{(\vx_i, y_i)\right\}_{i=1}^n$ can be expressed as $h(x) = \sum_{i=1}^n \lambda_i K(\vx_i, x)$ where $\lambda_i; i \! \in  \! [n]$ are constants. Assuming that the kernel's feature maps are of dimension $d$, the predictions are linear combinations of $d$-dimensional feature maps. We first present a general proof that every permutation invariant kernel requires $\Omega(p^2)$ training points to outperform the null predictor in terms of $\ell_2$ loss, and then show that this theorem applies to the distribution of empirical NTKs at initialization. 

\subsection{Notation}
We use $[p]$ to denote the set $\{1,\ldots,p\}$. We use $\mathbb{S}_p$ to denote the permutation groups over $p$ elements, $\sS_m$ as permutation group over $m$ elements and $\id$ is the identity mapping. For any nonempty set $\cX$, a symmetric function $K:\cX\times \cX\to \mathbb{R}$ is called a  positive semi-definite kernel (p.s.d) kernel on $\cX$ if for any $n\in\mathbb{N}$, any $\vx_1,\ldots,\vx_n$ and $\lambda_1,\ldots,\lambda_n\in\mathbb{R}$, it holds that $\sum_{i=1}^n\sum_{j=1}^n\lambda_i\lambda_jK(\vx_i,\vx_j)\ge 0$. For a subspace $V$ of $\mathbb{R}^n$ and vector $x\in\mathbb{R}^n$, we define $\dist(x,V)\triangleq \min_{v\in V}\norm{x-v}_2$.

For any $p^m$-dimensional vector $v \in \R^{p^m}$ we denote by $v(x)$ the vector $v$ indexed by an $m$-dimensional index vector $x \in [p]^m$. We define the vector $s_{i, a} \in \R^{p^m}$ whose entries are 
\begin{align}\label{eq:defi_s}
s_{i,a}(x) \triangleq \begin{cases}
    1, & x[i] = a; \\
    0, & \text{otherwise.}
\end{cases}
\end{align}
which helps us denote the all-once slices in this vector space, $V_{s} \triangleq \operatorname{span} \{s_{i,a}\}_{i \in [m], a \in [p]}$. We further define $\Delta(x, x')$ where $x, x'$ are two index vectors of size $m$ as the number of equal indices between them, formally:
\[
\Delta(x, x') = \sum_{i=1}^m \mathbf{1}_{x[i] = x'[i]}.
\]
We also define $\cX_{a,b}$ as the set of all $x, x' \in [p]^b$ index vector pairs such that $\dist(x, x') = a$, formally:
\[
\cX_{a,b} \triangleq \{(x, x'): x, x' \in [p]^b \wedge \Delta(x, x') = a\}.
\]
When $b=m$, we drop the second index and write $\cX_a$ (instead of $\cX_{a, m}$) for simplicity. It's clear that the collection of all $\cX_d$ for $0\le d \le m$ is a partitioning of the set of all pairs of index vectors. Unless stated otherwise, $x$ refers to the $m$-dimensional index vector. Finally, we define $\mathbb{U}_m \triangleq [\operatorname{Unif}(\sS_p)]^m$ as the product of $m$ Uniform distribution on $\sS_p$. 

\subsection{Loss Lower Bound: Regression Setting}

For convenience of notation, we will use the $(i,j, k)$ and $e_i+e_{j+p} + e_{k+2p}$ interchangebly for $i,j, k\in [p]$.  We denote the function $K(\vx_t,\cdot):[p]\times[p] \times [p] \to \mathbb{R}$ as a tensor on $\mathbb{R}^{p\times p \times p}$ by $v_t(\cdot)$ for each $t\in[n]$. We also define function $\Psi_{\sigma_1,\sigma_2, \sigma_3}(i,j, k)\triangleq \one \bigg(\sigma_1(i)+\sigma_2(j)\equiv \sigma_3(k)\ (\!\!\!\! \mod p)\bigg)$. We can view a function mapping from $[p]\times[p] \times [p] \to \mathbb{R}$ as a vector of size $p^3$ and define inner products and $\dist$ on the function space, \emph{i.e.}, $\inner{\Psi}{\Psi'}\triangleq\sum_{i,j,k\in [p]} \Psi(i,j,k)\Psi'(i,j,k)$ and $\norm{h}_2^2\triangleq \inner{h}{h}$.

\begin{theorem}\label{thm:kernel_lower_bound}
For any integers $n\ge 1$, $p\ge 2$ and kernel $K: \left([p]\times [p] \times [p]\right) \times \left([p]\times [p] \times [p]\right)\to \mathbb{R}$, for any $\vx_t = (i_t,j_t, k_t) \in [p]^3$ for each $t\in[n]$, it holds that 
\begin{align}\label{eq:kernel_loewr_bound}
\min_{\vx_1,\ldots,\vx_n}\E_{\sigma_1,\sigma_2,\sigma_3\sim \Unif(\mathbb{S}_p)}\inf_{\lambda_1,\ldots,\lambda_n\in\mathbb{R}} \norm{\sum_{t=1}^n \lambda_t K(\vx_t,\cdot)- \Psi_{\sigma_1,\sigma_2,\sigma_3}(\cdot)}_2^2 \ge p^2 \bigg( 1 - \frac{1}{p} - \frac{n}{p^3} \exp \left(\frac{2}{p-1} \right) \bigg)
\end{align}

In other words, if $n\le (1-\Omega(1))p^3$, then the expected population $\ell_2$ loss is at least $\Omega(p^2)$, which is of the same magnitude as the trivial all-zero predictor.
\end{theorem} 

\begin{proof}[Proof of \Cref{thm:kernel_lower_bound}]
This Theorem is a direct result of \Cref{lem:m_dim_modular_addition_loss_lower_bound} for the case where $m=3$.
\end{proof}

\kernellowerbound*

\begin{proof}[Proof of \Cref{th:kernel_lower_bound_3d}]
    Because $\cA$ is permutation-equivariant, we have that 
    \begin{align*}
         &\E_{ (\vx_i,y_i)_{i=1}^n\sim \cD^n}\E_{\cA}\cL_{\ell_2}  \left(\cA\left( \{(\vx_i,y_i)\}_{i=1}^n\right) \right) \\
        =&\E_{ (\vx_i,y_i)_{i=1}^n\sim \cD^n}\E_{\cA} \|\cA\left( \{(\vx_i,y_i)\}_{i=1}^n\right) - p\cdot \Psi_{\id,\id,\id}\|^2/p^3 \\
         =&\E_{ (\vx_i,y_i)_{i=1}^n\sim \cD^n}\E_{\cA} \|\cA\left( \{(\vx_i,y_i)\}_{i=1}^n\right)/p - \Psi_{\id,\id,\id}\|^2/p \\
         =&\E_{ (\vx_i,y_i)_{i=1}^n\sim \cD^n} \E_{\sigma_1,\sigma_2,\sigma_3\sim \Unif(\mathbb{S}_p)} \E_{\cA} \|\cA\left( \{(\vx_i,y_i)\}_{i=1}^n\right) / p - \Psi_{\sigma_1,\sigma_2,\sigma_3}\|^2/p. 
    \end{align*}
    Because $\cA$ is a kernel method, we have for any $(\vx_i,y_i)_{i=1}^n$ and $\sigma_1, \sigma_2, \sigma_3 \in \sS_p$
    $$\|\cA\left( \{(\vx_i,y_i)\}_{i=1}^n\right) / p - \Psi_{\sigma_1,\sigma_2,\sigma_3}\|^2\ge\inf_{\lambda_1,\ldots,\lambda_n\in\mathbb{R}} \norm{\sum_{t=1}^n \lambda_t K(\vx_t,\cdot)-\Psi_{\sigma_1,\sigma_2,\sigma_3}(\cdot)}_2^2. $$ Applying \Cref{thm:kernel_lower_bound} completes the proof.
\end{proof}

\subsection{Loss Lower Bound: Classification Setting} \label{app:sec:classification_lower_bound}
For this subsection, we define the modular addition function $\Psi: [p] \times [p] \to [p]$ as $\big[\Psi^{\sigma_3}_{\sigma_1, \sigma_2}(i, j)\big]_k \triangleq \mathbf{1} \Big( \sigma_1(i) + \sigma_2(j) \equiv \sigma_3(k) \pmod p \Big)$ for all $k \in [p], \sigma_1, \sigma_2, \sigma_3 \in \sS_p$.
\begin{theorem} \label{thm:kernel_lower_bound_2d_internal}
For any integers $n > 1, p \ge 2$ and input-output permutation equivariant kernel $K: ([p] \times [p]) \times ([p] \times [p]) \to \R^{p \times p}$ (according to \Cref{prop:permutation_equivaraince_2d}), suppose $x_t = (i_t, j_t) \overset{i.i.d.}{\sim} \Unif([p] \times [p])$ for each $t \in [n]$ it holds that
\[
\E_{x_1,\ldots,x_n}\E_{\sigma_1,\sigma_2\sigma_3\sim \Unif(\sS_p)}\inf_{\lambda_1,\ldots,\lambda_n\in\R^p} \sum_{x \in [p] \times [p]} \norm{\sum_{t=1}^n K(x_t,x) \lambda_t - \Psi^{\sigma_3}_{\sigma_1,\sigma_2}(x)}_2^2 \ge p^2 \bigg( 1 - \frac{1}{p} - \frac{n}{p^2} \exp \left(\frac{2}{p-1} \right) \bigg).
\]
In other words, if $n\le (1-\Omega(1))p^2 $, then the expected population $\ell_2$ loss is at least $\Omega(p^2)$, which is of the same magnitude as the trivial all-zero predictor.
\end{theorem}

\begin{proof}
Note that
\begin{align}
&\E_{x_1,\ldots,x_n}\E_{\sigma_1,\sigma_2,\sigma_3\sim \Unif(\sS_p)}\inf_{\lambda\in\R^{n \times p}} \sum_{x \in [p] \times [p]} \norm{\sum_{i=1}^n K(x_i,x) \lambda_i - \Psi^{\sigma_3}_{\sigma_1,\sigma_2}(x)}_2^2 \notag \\
&\qquad =\E_{x_1,\ldots,x_n}\E_{\sigma_1,\sigma_2\sim \Unif(\sS_p)}\inf_{\lambda\in\R^{n \times p}} \sum_{x \in [p] \times [p]} \sum_{j=1}^p \left(\sum_{i=1}^n \inner{\left[K(x_i, x) \right]_{j,:}}{\lambda_i} - \left[ \Psi^{\sigma_3}_{\sigma_1, \sigma_2}(x) \right]_{j} \right)^2 \notag \\
&\qquad = \E_{x_1,\ldots,x_n}\E_{\sigma_1,\sigma_2,\sigma_3\sim \Unif(\sS_p)}\inf_{\lambda\in\R^{n \times p}} \sum_{x \in [p] \times [p]} \sum_{j=1}^p \left(\sum_{i=1}^n \sum_{k=1}^p \lambda_{i,k} \left[K(x_i, x) \right]_{j, k}  - \left[ \Psi^{\sigma_3}_{\sigma_1, \sigma_2}(x)\right]_{j} \right)^2.
\end{align}
We define
\begin{align}
&K'\big((i, j, k), ({i'}, {j'}, {k'})\big) \triangleq K\big((i, j), ({i'}, {j'})\big)_{k, k'},  \notag \\
&\Psi'_{\sigma_1, \sigma_2, \sigma_3}\big( (i, j, k) \big) \triangleq \left[\Psi^{\sigma_3}_{\sigma_1, \sigma_2} \big( (i, j) \big) \right]_{k}, \notag \\
&\lambda'_t \triangleq \lambda_{\floor{t/p}, t \hspace*{-2mm} \mod p}, \notag \\
&x'_t \triangleq (i_{\floor{t/p}}, j_{\floor{t/p}},  t \hspace*{-2mm} \mod p) 
\end{align}
for any $i, j, k, i', j', k' \in [p]$, $\sigma_1, \sigma_2, \sigma_3 \in \sS_p$, $t \in [np]$, $\lambda \in \R^{n \times p}$ and $\lambda' \in \R^{np}$. It can be seen that
\begin{align}
&\E_{\substack{x_1,\ldots,x_n \\ \sigma_1,\sigma_2,\sigma_3\sim \Unif(\sS_p)}} \inf_{\lambda\in\R^{n \times p}} \sum_{x \in [p] \times [p]} \sum_{j=1}^p \left(\sum_{i=1}^n \sum_{k=1}^p \lambda_{i,k} \left[K(x_i, x) \right]_{j, k}  - \left[ \Psi^{\sigma_3}_{\sigma_1, \sigma_2}(x)\right]_{j} \right)^2 \notag \\
&\E_{\substack{x'_1,\ldots,x'_{np} \\ \sigma_1,\sigma_2,\sigma_3\sim \Unif(\sS_p)}}\inf_{\lambda'\in\R^{np}} \sum_{x \in [p] \times [p]} \sum_{j=1}^p \left(\sum_{i=1}^{np} \lambda'_i K'\big(x'_i, (x[0], x[1], i \hspace*{-2.5mm}\mod p)\big)- \Psi'_{\sigma_1, \sigma_2, \sigma_3}(x) \right)^2.
\end{align}
This is similar to \Cref{thm:kernel_lower_bound} except that the underlying data is generated through sampling from independent groups each having a size of $p$ (we sample i.i.d from $\Unif([p] \times [p])$ and for each sample we consider the set of all possible responses). Applying the result of \Cref{thm:kernel_lower_bound} yields
\begin{align}
&\E_{\substack{x'_1,\ldots,x'_{np} \\ \sigma_1,\sigma_2,\sigma_3\sim \Unif(\sS_p)}}\inf_{\lambda'\in\R^{np}} \sum_{x \in [p] \times [p]} \sum_{j=1}^p \left(\sum_{i=1}^{np} \lambda'_i K'\big(x'_i, (x[0], x[1], i \hspace*{-2.5mm}\mod p)\big)- \Psi'_{\sigma_1, \sigma_2, \sigma_3}(x) \right)^2 \notag \\
&\qquad\qquad \ge p^2 \left(1 - \frac{1}{p} - \frac{n}{p^2} \exp \left(\frac{2}{p-1} \right) \right)
\end{align}
which completes the proof.
\end{proof}

\kernellowerboundclass*

\begin{proof}[Proof of \Cref{th:kernel_lower_bound_2d}]
    Because $\cA$ is input-output permutation-equivariant, we have that 
    \begin{align*}
         &\E_{ (\vx_i,y_i)_{i=1}^n\sim \cD^n}\E_{\cA}\cL_{\ell_2}  \left(\cA\big( \{(\vx_i,y_i)\}_{i=1}^n\big) \right) \\
         &\qquad=\E_{ (\vx_i,y_i)_{i=1}^n\sim \cD^n}\E_{\cA} \sum_{x \in [p] \times [p]} \left\| \cA\big( \{(\vx_i,y_i)\}_{i=1}^n\big)(x) - p\cdot \Psi^{\id}_{\id,\id}(x) \right\|_2^2/p^2 \\
         &\qquad=\E_{ (\vx_i,y_i)_{i=1}^n\sim \cD^n}\E_{\cA} \sum_{x \in [p] \times [p]} \left\|\cA\big( \{(\vx_i,y_i)\}_{i=1}^n\big)(x)/p - \Psi^{\id}_{\id,\id}(x)\right\|_2^2 \\
         &\qquad=\E_{ (\vx_i,y_i)_{i=1}^n\sim \cD^n} \E_{\sigma_1,\sigma_2,\sigma_3\sim \Unif(\mathbb{S}_p)} \E_{\cA} \sum_{x \in [p] \times [p]} \left\|\cA\big( \{(\vx_i,y_i)\}_{i=1}^n\big)(x) / p - \Psi^{\sigma_3}_{\sigma_1,\sigma_2}(x)\right\|_2^2. 
    \end{align*}
    Because $\cA$ is a kernel method, we have for any $(\vx_i,y_i)_{i=1}^n$ and $\sigma_1, \sigma_2, \sigma_3 \in \sS_p$
    $$\sum_{x \in [p] \times [p]} \left\|\cA\big( \{(\vx_i,y_i)\}_{i=1}^n\big)(x) / p - \Psi^{\sigma_3}_{\sigma_1,\sigma_2}(x)\right\|_2^2\ge\inf_{\lambda_1,\ldots,\lambda_n\in\mathbb{R}} \sum_{x \in [p] \times [p]}\norm{\sum_{t=1}^n \lambda_t K(\vx_t,x)-\Psi^{\sigma_3}_{\sigma_1,\sigma_2}(x)}_2^2. $$ Applying \Cref{thm:kernel_lower_bound_2d_internal} completes the proof.
\end{proof}

\subsection{Loss Lower Bound: General Theorem For Arbitrary Functions} \label{app:sec:general_lower_bound}

\begin{lemma}\label{lem:dist}
	For any subspace $V$ of $\mathbb{R}^n$ and vector $x\in\mathbb{R}^n$, let $\{v_i\}_{i=1}^m$ be an orthonormal basis of $V$. It holds that $\dist^2(x,V) = \norm{x}_2^2 - \sum_{i=1}^m\inner{x}{v_i}^2$.
\end{lemma}
The proof of \Cref{lem:dist} is straightforward and thus omitted.   

\begin{lemma} \label{lem:lower_bounded_loss_top_eigvals}
For any distribution $\cD$ over functions mapping from $\cX \to \R$ and $n$ functions $\{k_i\}_{i=1}^n$ where $k_i: \cX \to \R$ for each $i \in [n]$ it holds that
\[
\E_{h \sim \cD} \left[ \min_{\alpha \in \R^n} \frac{1}{\abs{\cX}} \sum_{x \in \cX} \left(\Psi(x) - \sum_{i=1}^n \alpha_i k_i(x) \right)^2\right] \ge \frac{1}{\abs{\cX}} \sum_{i=n+1}^{\abs{\cX}} \lambda_i(\Sigma).
\]
 where $\Sigma_\cD(x, x') \triangleq \E_{\Psi \sim\cD}\left[\Psi(x) \Psi(x') \right]$ and $\lambda_i$ denotes the $i$'th largest eigenvalue function. For notational convenience, we also view $\Sigma_\cD$ as a $\R^{\abs{\cX} \times \abs{\cX}}$ matrix. 
\end{lemma}

\begin{proof}[Proof of \Cref{lem:lower_bounded_loss_top_eigvals}]
For each $\Psi$, we define $r_\Psi \in \R^{\abs{\cX}}$ whose entries are realization of the function $\Psi$ on inputs $x \in \cX$. It suffices to show that
\[
\E_{h \sim \cD} \left[ \dist^2(r_\Psi, V) \right] \ge \sum_{i=n+1}^{\abs{\cX}} \lambda_i(\Sigma_\cD)
\]
for any subspace $V = \operatorname{span}\{v_1, v_2, \cdots, v_n\} \subset \R^{\abs{\cX} \times \abs{\cX}}$ where $v_i$ for $i \in [n]$ are orthonormal vectors. Note that
\begin{align}
\E_{\Psi \sim \cD} \left[ \dist^2(r_\Psi, V) \right] &= \E_{\Psi \sim \cD} \left[ \norm{r_\Psi}_2^2 - \sum_{t=0}^n \inner{v_t}{r_\Psi}^2 \right] \notag \\
&= \Tr\left(\E_{\Psi \sim \cD} \left[ \Psi \Psi^\top \right] \right) -  \sum_{t=1}^n v_t^\top \left(\E_{\Psi \sim \cD} \left[ \Psi \Psi^\top \right] \right) v_t \notag \\
&\ge \Tr\left( \Sigma_\cD \right) - \sum_{i=1}^n \lambda_i \left(\Sigma_\cD \right) \notag \\
&= \sum_{i=n+1}^{\abs{\cX}} \lambda_i (\Sigma_\cD)
\end{align}
where the inequality in the second to last line is due to the min-max theorem (also called Courant–Fischer–Weyl min-max principle) \citep{courant1953methods}.
This completes the proof. 
\end{proof}

The following \Cref{lem:lower_bounded_loss_vs_top_n_eigvals} is a direct consequence of \Cref{lem:lower_bounded_loss_top_eigvals}, noting that $\E_{\Psi \sim \cD} \left[\frac{1}{\abs{\cX}} \sum_{x \in \cX} \Psi(x)^2 \right] = \Tr(\Sigma_\cD)$.
\begin{corollary} \label{lem:lower_bounded_loss_vs_top_n_eigvals}
For any distribution $\cD$ over functions mapping from $\cX \to \R$ and $n$ functions $\{k_i\}_{i=1}^n$ where $k_i: \cX \to \R$ for each $i \in [n]$, if $\sum_{i=1}^n \lambda_i(\Sigma_\cD) \le \frac{1}{2} \Tr(\Sigma_\cD)$ then it is guaranteed that
\[
\E_{\Psi \sim \cD} \left[ \min_{\alpha \in \R^n} \frac{1}{\abs{\cX}} \sum_{x \in \cX} \left(\Psi(x) - \sum_{i=1}^n \alpha_i k_i(x) \right)^2\right] \ge \frac{1}{2} \E_{\Psi \sim \cD} \left[\frac{1}{\abs{\cX}} \sum_{x \in \cX} \Psi(x)^2 \right]
\]
where the right-hand side denotes the expected loss of the all-0 predictor.
\end{corollary}

\begin{lemma} \label{lem:bound_top_eigvals_vs_first_projected_eigval}
For any matrix $\Sigma \in \R^{d \times d}$, subspace $V$ and projection matrix $P_V \in \R^{d \times d}$ corresponding to $V$, it holds that 
\[
\sum_{i=1}^n \lambda_i (\Sigma) \le n \cdot \lambda_1\bigg((I - P_V) \Sigma (I - P_V)\bigg) + \Tr\left( P_V \Sigma P_V \right).
\]
\end{lemma}

\begin{proof}
Let $V_n = \operatorname{span}\{\alpha_i\}_{i=1}^n$ where $\alpha_i \in \R^d$ is the $i$'th eigenvector of $\Sigma$ and let $P_{V_n}$ be its corresponding projection matrix. Let $P_{V_n + V}$ be the projection onto the sum of two subspaces $V_n + V$, it holds that
\begin{align}
\sum_{i=1}^n \lambda_i(\Sigma) &= \Tr(P_{V_n} \Sigma P_{V_n}) \notag \\
&\le \Tr(P_{V_n + V} \Sigma P_{V_n + V}) \notag \\
&= \Tr(P_V \Sigma P_V) + \Tr\big( (P_{V_n + V} - P_V) \Sigma (P_{V_n + V} - P_V)\big) \notag \\
&\le \Tr(P_V \Sigma P_V) + n \cdot \lambda_1 \big((P_{V_n + V}-P_V) \Sigma (P_{V_n + V} - P_V)\big) \notag\\
&\le \Tr(P_V \Sigma P_V) + n \cdot \lambda_1 \big((I-P_V) \Sigma (I - P_V)\big).
\end{align}
Here the second to last inequality is because $P_{V_n + V}-P_V$ is at most rank-$n$ and so is $(P_{V_n + V}-P_V) \Sigma (P_{V_n + V} - P_V)$. The last inequality is because $P_{V_n + V}\le I$ and thus $(P_{V_n + V}-P_V) \Sigma (P_{V_n + V} - P_V) \preceq (I-P_V) \Sigma (I - P_V)$.
\end{proof}

\subsection{Loss Lower Bound: Modular Addition with \texorpdfstring{$m$}{m} Summands}

\begin{lemma} \label{lem:top_n_eigvals_for_modular_addition}
Let $\cD$ be the uniform distribution over  
\begin{align} \label{def:m_dim_modualr_addition_func_class}
\cH \triangleq \left\{\Psi(x) = \mathbf{1} \bigg(\sum_{i=1}^m \sigma_i(x_i) \equiv 0 \; (\hspace*{-3.5mm}\mod p)\bigg) \, \Bigg \vert \; \sigma_i \in \sS_p \text{ for all } i \in [m] \right\}
\end{align}
and $\Sigma_\cD(x,x') = \E_{\Psi \sim \cD} \left[ \Psi(x)\Psi(x') \right]$. It holds that 
\begin{align*}
\sum_{i=1}^n \lambda_i(\Sigma) \le p^{m-2} + \frac{n}{p} \exp \left(\frac{m-1}{p-1}\right).
\end{align*}
\end{lemma}
\begin{proof}[Proof of \Cref{lem:top_n_eigvals_for_modular_addition}]
Consider the vector space $V_s$ defined in \Cref{eq:defi_s} and the projection matrix $P_{V_s} \in \R^{\abs{\cX} \times \abs{\cX}}$ onto $V_S$. To find $\Tr\big( P_{V_s} \Sigma P_{V_s} \big)$ we can first derive $\Tr\big(P_{V_s} \Psi \Psi^\top P_{V_s} \big)$ where $\Psi \in \R^{\abs{\cX}}$ is a sample from $\cD$. Note that since $(I-P_{V_s})h$ should be orthogonal to $V_s$ (the sum of each slice of the projected vector should be zero) we have that
\[
\Big[(I - P_{V_s}) \Psi\Big](x) = \begin{cases}
    -\frac{1}{p} & \Psi(x) = 0 \\
    \frac{p-1}{p} & \Psi(x) = 1
\end{cases}.
\]
Hence, $P_{V_s}\Psi(x) = \frac{1}{p}$ for all $x \in \cX$. Thus, 
\begin{align}
\Tr\bigg( P_{V_s} \Sigma P_{V_s} \bigg) &= \E_{\Psi \sim \cD} \left[ \Tr\bigg(P_{V_s} \Psi \Psi^\top P_{V_s} \bigg) \right] \notag \\
&= \E_{\Psi \sim \cD} \left[ \Tr\bigg( \frac{1}{p^2} \mathbf{1}_{\abs{\cX} \times \abs{\cX}} \bigg) \right] \notag \\
&= \frac{\abs{\cX}}{p^2} \notag \\
&= p^{m-2}
\end{align}
where $\mathbf{1}_{\abs{\cX} \times \abs{\cX}}$ denotes the all-one square matrix of size $\abs{\cX}$. Moreover, by \Cref{lem:bounded_vh_inner_prod} we have that
\begin{align}
n \cdot \lambda_1 \bigg((I - P_{V_s}) \Sigma (I - P_{V_s}) \bigg) = \sup_{\norm{v}_2\le 1,v\perp V_s} v^\top \Sigma_{\cD}v \le \frac{n}{p} \exp \left(\frac{m-1}{p-1} \right).
\end{align}
Combining the two equations above, we can see that
\[
\sum_{i=1}^n \lambda_i(\Sigma) \le p^{m-2} + \frac{n}{p} \exp \left( \frac{m-1}{p-1} \right).
\]
This concludes the proof.
\end{proof}

\begin{corollary} \label{lem:m_dim_modular_addition_loss_lower_bound}
Consider the function class $\cH$ defined in \Cref{def:m_dim_modualr_addition_func_class} and the uniform distribution over it denoted by $\cD$. For any integers $n \ge 1, p \ge 2$, $1 \le m < p$, permutation-equivariant kernel $K: [p]^m \times [p]^m \to \R$ it holds that
\[
\min_{x_1, x_2, \cdots x_n} \E_{\Psi \sim \cD} \inf_{\alpha \in \R^n} \left \lVert \sum_{t=1}^n \lambda_t K(x_t, \cdot) - \Psi \right \rVert_2^2 \ge p^{m-1} \bigg( 1 - \frac{1}{p} - \frac{n}{p^m} \exp \left(\frac{m-1}{p-1} \right) \bigg).
\]
for any $x_1, x_2, \cdots, x_n \in [p]^m$. In other words, if $n < (1 - \Omega(1)) p^m$, then the expected population $\ell_2$ loss is at least $\Omega(p^{m-1})$, which is of the same magnitude as the trivial all-zero predictor.
\end{corollary}

\begin{proof}[Proof of \Cref{lem:m_dim_modular_addition_loss_lower_bound}]
It suffices to show that for any $n$-dimensional subspace $V \subset \R^{p^m}$ it holds that
\[
\E_{\Psi \sim \cD} \dist^2(V, \Psi) \ge  p^{m-1} \bigg( 1 - \frac{1}{p} - \frac{n}{p^m-1} \exp \left(\frac{m-1}{p-1} \right) \bigg).
\]
Combining \Cref{lem:top_n_eigvals_for_modular_addition} and \Cref{lem:lower_bounded_loss_vs_top_n_eigvals} yields this statement.
\end{proof}

\begin{lemma} \label{lem:bounded_vh_inner_prod}
For any $v \in \sR^{p^m}$ such that $\norm{v} = 1$ and $v \perp V_s$ (defined in \Cref{eq:defi_s}), it holds that
\[\E_{\sigma \sim \mathbb{U}_m} \left[ \inner{\Psi_\sigma}{v}^2 \right] \le \frac{1}{p}\exp\left(\frac{m-1}{p-1}\right). \]
\end{lemma}

The main implication of this \Cref{lem:bounded_vh_inner_prod} is that for any $n$-dimensional space $V \subset \R^{p^m}$ can not "cover" the vector space of all $\Psi_\sigma$ functions for different permutations $\sigma \in \sS_p$. To prove this Lemma, we decompose the inner product $\inner{v}{h_\sigma}$ to $d+1$ sums using the Multinomial Theorem. This decomposition is enabled through observing the fact that there are $d+1$ equivalence groups in the possible set of indices of $v$. Based on \Cref{lem:equiv_group_on_m_indices_exp_bound}, we can use the Binomial Theorem to decompose the expectation of the inner product as follows

\begin{align}
    \E_{\sigma \sim \mathbb{U}_m} \left[\inner{v}{\Psi_\sigma}^2 \right] &= 
    \E_{\sigma \sim \mathbb{U}_m} \left[ \left( \sum_{x} v(x) \Psi_\sigma(x) \right)^2 \right] \notag \\
    &= \E_{\sigma \sim \mathbb{U}_m} \left[ \sum_{d=0}^m \sum_{\substack{x, x' \\ \dist(x, x') = d}} \Psi_\sigma(x) \Psi_\sigma(x') v(x)v(x') \right] \notag \\
    &= \sum_{d=0}^m C_d \sum_{\substack{x, x' \\ \dist(x, x') = d}} v(x) v(x') 
\end{align}
where $C_d \triangleq \E_{\sigma \sim \mathbb{U}_m} \left[ \Psi_\sigma(x) \Psi_\sigma(x') \right]$ for any $(x, x') \in \cX_d$.

To complete the proof, we now have to bound two terms. First, we need to show that the expectation $\E_{\sigma \sim \mathbb{U}_m} \left[ \Psi_\sigma(x) \Psi_\sigma(x') \right]$ for each $(x, x')$ in the same equivalence group is bounded. Next, we need to show that for each set $\cX_d$, the sum $\sum_{\substack{(x, x') \in \cX_d}} v(x) v(x')$ is also bounded. These are correspondingly shown in \Cref{lem:equiv_group_on_m_indices_exp_bound,lem:equiv_group_interaction_equality}. Based on these two Lemmas, we can now present the proof of \cref{lem:bounded_vh_inner_prod}.

\begin{proof}[Proof of \cref{lem:bounded_vh_inner_prod}] 
\begin{align}
    \E_{\sigma \sim \mathbb{U}_m} \left[\inner{v}{\Psi_\sigma}^2 \right] 
    &= \sum_{d=0}^m C_d \sum_{\substack{x, x' \\ \dist(x, x') = d}} v(x) v(x') \notag \\
    &= \sum_{d=0}^m \frac{1}{p^2} \left( 1 - \frac{1}{(1-p)^{d-1}} \right) (-1)^d {m \choose d} \notag \\
    &= \frac{1}{p} \left( 1 + \frac{1}{p-1} \right)^{m-1} \notag \\
    &\le \frac{1}{p} \exp \left( \frac{m-1}{p-1}\right).
\end{align}
where the second to last step is due to the binomial Theorem and the last step is due to the fact that for all $x \in \R$, $1 + x \le \exp(x)$.

\end{proof}

\begin{lemma}\label{lem:equiv_group_on_m_indices_exp_bound} For any index vector pair $(x, x') \in \cX_d$ it holds that
\[
\E_{\sigma \sim \mathbb{U}_m} \left[ \Psi_\sigma(x) \Psi_\sigma(x') \right] = \frac{1}{p^2} \left( 1 - \frac{1}{(1-p)^{d-1}}  \right).
\]
\end{lemma}

\begin{proof}[Proof of \Cref{lem:equiv_group_on_m_indices_exp_bound}]

Let us re-iterate the definition of $h$:
\[
\Psi_\sigma(x) = \mathbf{1} \bigg(\sum_{i=1}^m \sigma_i(x_i) \equiv 0 \; (\hspace*{-3.5mm}\mod p)\bigg).
\]
For each pair $(x',x')$ we define $I(x, x') \triangleq \{i: x[i] \neq x'[i]\}$ and $E(x, x') \triangleq \{i: x[i] = x'[i]\}$. We also define  
\[
C_d \triangleq \E_{\sigma \sim \mathbb{U}_m} \left[ \Psi_\sigma(x) \Psi_\sigma(x') \right]
\]
where for each $(x, x') \in \cX_d$. Note that for $d = 0$, we have that
\begin{align}
C_0 &= \E_{\sigma \sim \mathbb{U}_m} \left[ \Psi_\sigma(x) \Psi_\sigma(x') \right] \notag \\
&= \E_{\sigma \sim \sU_m} \left[ \Psi_\sigma(x) \right] \notag \\
&= \E_{\sigma \sim \sU_m} \left[ \mathbf{1} \bigg( \sum_{i=1}^m \sigma_i(x_i) \equiv 0 \pmod p \bigg) \right] \notag \\
&= \E_{a \sim \Unif([p])} \left[ a \equiv 0 \pmod p \right] \notag \\
&= 1/p.
\end{align}

For any $2 \le d \le m$ it holds that 
\begin{align}
C_d &= \E_{\sigma \sim \mathbb{U}_m} \left[ \mathbf{1} \bigg(\sum_{i=1}^m \sigma_i(x_i) \equiv \sum_{i=1}^m \sigma_i(x'_i) \equiv 0 \; (\hspace*{-3.5mm}\mod p)\bigg) \right] \notag \\
&= \E_{\sigma \sim \mathbb{U}_m} \left[ \mathbf{1} \bigg(\sum_{i \in I(x,x')} \sigma_i(x_i) + \sum_{i \in E(x,x')} \sigma_i(x_i) \equiv \sum_{i \in I(x,x')} \sigma_i(x'_i) + \sum_{i \in E(x,x')} \sigma_i(x_i) \equiv 0 \; (\hspace*{-3.5mm}\mod p)\bigg) \right] \notag \\
&= \E_{\sigma \sim \mathbb{U}_m} \left[ \mathbf{1} \bigg(\sum_{i \in I(x,x')} \sigma_i(x_i) \equiv \sum_{i \in I(x,x')} \sigma_i(x'_i) \equiv - \sum_{i \in E(x,x')} \sigma_i(x_i)  \; (\hspace*{-3.5mm}\mod p)\bigg) \right] \notag \\
&= \E_{\substack{\sigma \sim \mathbb{U}_m \\\zeta \sim \Unif([p])}} \left[ \mathbf{1} \bigg(\sum_{i \in I(x,x')} \sigma_i(x_i) \equiv \sum_{i \in I(x,x')} \sigma_i(x'_i) \equiv \zeta \; (\hspace*{-3.5mm}\mod p)\bigg) \right] \notag \\
&= \E_{\substack{(y, y') \sim \Unif(\cX_{d,d}) \\ \zeta \sim \Unif([p])}}  \left[  
\mathbf{1} \bigg( \sum_{i=1}^d y_i \equiv \sum_{i=1}^d y_i \equiv \zeta \; (\hspace*{-3.5mm}\mod p) \bigg) \right]
\end{align}
where in the second to last line, we replaced $\E_{\sigma \sim \sU_m} \left[ - \sum_{i \in E(x,x')} \sigma_i(x_i) \right]$ with $\E_{\zeta \sim \Unif([p])} [\zeta]$ (assuming $d \ge 1$). We aim to find the closed form formula for the general $d \ge 2$. As $\zeta \sim \Unif([p])$ is independent of the two sums, we can absorb it and write (from here until the rest of the proof we drop $(\hspace*{-2mm}\mod p)$ from equivalences for ease of presentation) 
\begin{align}
    Q_d \triangleq \Pr_{(y,y') \sim \operatorname{Unif}(\cX_{d,d})} \left[ \mathbf{1} \left( \sum_{i=1}^d y_i \equiv \sum_{i=1}^d y'_i \right) \right] = p \cdot C_d.
\end{align}

Note that for $d \ge 2$ we have that 
\begin{align}
    Q_d &= \Pr_{(y,y') \in \Unif(\cX_{d,d})} \left[ \mathbf{1} \left( \sum_{i=1}^{d-1} y_i \equiv \sum_{i=1}^{d-1} y'_i \right) \cdot \mathbf{1}(y_d = y'_d) \right] \notag \\ 
    &\qquad\qquad +\Pr_{(y,y') \in \Unif(\cX_{d,d})} \left[ \mathbf{1} \left( \sum_{i=1}^{d-1} y_i \not\equiv \sum_{i=1}^{d-1} y'_i \right) \cdot \mathbf{1} \left(y_d \equiv y'_d + \sum_{i=1}^{d-1} y'_i - \sum_{i=1}^{d-1} y_i  \right) \right] \notag \\
    &= (1-Q_{d-1}) \cdot \frac{1}{p-1}.
\end{align}

Hence for $d \ge 2$ we have that 

\begin{align}
    C_d &= \frac{1}{p (p-1)} (1 - p \cdot C_{d-1}) \notag \\
    &= \frac{1}{p(p-1)} - \frac{C_{d-1}}{p-1} \notag \\
    &= \frac{1}{p^2} \left( 1 - \frac{1}{(1-p)^{d-1}} \right).
\end{align}
This completes the proof.
\end{proof}

\begin{lemma} \label{lem:equiv_group_interaction_equality} For any $\cX_d$ where $d \in [m]$ and $v \in \R^{p^m}$ such that $\norm{v}_2 = 1$ and $v \perp V_s$ (defined in \Cref{eq:defi_s}) it holds that 
\[
\sum_{(x, x') \in \cX_d} v(x) v(x') =  (-1)^{d} {m \choose d}.
\]
\end{lemma}
\begin{proof}[Proof of \Cref{lem:equiv_group_interaction_equality}]
Let us define $G(d)$ as
\begin{align}
G(d) &\triangleq \frac{1}{{m \choose d}}\sum_{(x,x') \in \cX_d} v(x) v(x') \notag \\
&= \E_{\sigma \in \sS_m} \sum_{y \in [p]^{m-d}} \sum_{(t,t') \in \cX_{d,d}} v \big(\sigma(y \Vert t)\big) v\big(\sigma(y \Vert t')\big).
\end{align}
Since $v \perp V_s$ we have that
\begin{align}
&\E_{\sigma \in \sS_m} \left[ \sum_{y \in [p]^d} \sum_{(t,t') \in \cX_{d-1, d-1}} \left( \sum_{s \in [p]} v\big( \sigma(y \Vert s \Vert t) \big) \right) \left( \sum_{s \in [p]} v\big( \sigma(y \Vert s \Vert t') \big) \right) \right] \notag \\
&\qquad= \E_{\sigma \in \sS_m} \left[\sum_{(y \Vert s) \in [p]^{m-d+1}} \sum_{(t,t') \in \cX_{d-1, d-1}} v\big( \sigma(y \Vert s \Vert t) \big) v\big( \sigma(y \Vert s \Vert t') \big)\right] \notag \\
&\qquad\quad+  \E_{\sigma \in \sS_m} \left[ \sum_{y \in [p]^{m-d}} \sum_{(s\Vert t, s' \Vert t') \in \cX_{d,d}} v\big( \sigma(y \Vert s \Vert t') \big) v\big( \sigma(y \Vert s' \Vert t') \big) \right] \notag \\
&\qquad= G(d-1) + G(d) = 0.
\end{align}
Note that $G(0) = \sum_{x} v(x)^2 = \norm{v}_2^2 = 1$. Hence, $G(d) = (-1)^d$. This completes the proof.
\end{proof}

\section{Generalization Upper Bound for Regression} \label{app:sec:generalization_bound_3d}

The original model is $f\left(\begin{pmatrix} e_i, e_j \end{pmatrix}^\top \right) = V (W \begin{pmatrix} e_i , e_j \end{pmatrix}^\top)^{\odot 2}$. We consider the model $g \left(\begin{pmatrix} e_i, e_j, e_k \end{pmatrix}^\top \right) = \left \langle e_k, f\left(\begin{pmatrix} e_i, e_j \end{pmatrix}^\top \right) \right \rangle$ and the function class $\cH$ is defined over $g$ with different weights $\theta =(W, V)$ where $W \in \R^{h \times 2p}$ and $V \in \R^{p \times h}$. For an input $x \in \R^{3p}$, we define two slices $x' \triangleq x\operatorname{[:2p]}$ and $x'' \triangleq x\operatorname{[2p:]}$. We also define $\cW_{h,r} \triangleq \{ W \in \R^{h \times 2p}: \norm{W}_\infty \le r\}$ and $\cV_{h,r} \triangleq \{ V \in \R^{p \times h}: \norm{V}_\infty \le r\}$ which we will use later to denote parameters of our function. We further define $\dset_n = \{x_1, x_2, \cdots, x_n\}$ such that for all $a \in [n]$ we have $x_a = \begin{pmatrix}
    e_i, e_j, e_k
\end{pmatrix}^\top$ for some $i,j,k \in [p]$. For this section, we fix the set $\{x_1, x_2, \cdots, x_n\} \sim \Unif(\cX)$ and denote by $\cR_n(\cH)$ the empirical Rademacher complexity of the function class $\cH$ defined as
\begin{equation} \label{eq:emp_rademacher_complexity}
    \cR_n(\cH) \triangleq \E_{\sigma \sim \Unif(\{\pm 1\}^n)} \left[ \sup_{\Psi \in \cH} \frac{1}{n} \sum_{i=1}^n \Psi(x_i) \sigma_i \right]
\end{equation}
where $\cH$ maps $\cX \to \R$ and $n \in \sN$.

\begin{lemma} \label{lemma:transform_vwx_to_ux}
Consider the function classes 
\begin{equation} \label{sup:eq:single_neuron_net_class}
\cH^w_{r, r'} \triangleq \left\{ \Psi: \R^{3p} \to \R \, \mid \, \exists \left [ W^\top \in \cW_{w,r} \wedge V \in \cV_{w,r'} \right] \text{ s.t. } \Psi(x) = \inner{x''}{V\inner{W}{x'}^2} \right\}
\end{equation}
 and 
\begin{equation} \label{sup:eq:cubic_net_class}
\cG_r \triangleq \left\{ g: \R^{3p} \to \R \, \mid \, \exists \left [ U \in \R^{4 \times 3p} \wedge \norm{U}_\infty \le r \right]  \text{ s.t. } g(x) = \sum_{i=1}^4 \langle U_i^\top, x \rangle^3 \right\}
\end{equation}
where $U_i$ denotes the $i$'th row of $U$.
The function class $\cH^1_{r,r'}$ is contained in $\cG_{\max(r, r')}$ (and hence $\cR_n(\cH_{r,r'}) \le \cR_n(\cG_{\max(r, r')})$). 
\end{lemma}

\begin{proof}[Proof of \Cref{lemma:transform_vwx_to_ux}]
We prove this lemma by showing that for each pair of matrices $W, V$ of $\Psi \in \cH^1_{r,r'}$, we can construct a matrix $U$ of $g \in \cG_{\max(r, r')}$ such that for all $x \in \R^{3p}$, $h(x) = g(x)$. Consider an arbitrary parameterization $W, V$ of $h$ such that $\Psi(x) = \inner{x''}{V \inner{W}{x'}^2}$. We can construct $U =  \sqrt[3]{\frac{2}{9}}\begin{pmatrix}Q_1, Q_2, Q_3, Q_4 \end{pmatrix}^\top$ where
\begin{align}
\begin{split}
    Q_1 &= \begin{pmatrix} W \\ V \end{pmatrix}, \qquad Q_2 = \begin{pmatrix} -W \\ V \end{pmatrix}, \qquad Q_3 = \begin{pmatrix} -W/2 \\ V \end{pmatrix}, \qquad Q_4 = \begin{pmatrix} W/2 \\ -V \end{pmatrix}.
\end{split}
\end{align}
Observe that 
\begin{align}
&g(x) = \sum_{i=1}^4 \inner{U_i^\top}{x}^3 \notag \\
&= \frac{2}{9}  \left[ \inner{U_1}{x}^3 + \inner{U_2}{x}^3 + \inner{U_3}{x}^3 + \inner{U_4}{x}^3 \right] \notag \\
&= \frac{2}{9} \left[ \left(\inner{W}{x'} + \inner{V}{x''} \right)^3 - \left(\inner{W}{x'} - \inner{V}{x''} \right)^3 - \left(\frac{\inner{W}{x'}}{2} + \inner{V}{x''} \right)^3 + \left(\frac{\inner{W}{x'}}{2} - \inner{V}{x''} \right)^3 \right] \notag \\
&= \inner{x''}{V\inner{W}{x'}^2} \notag \\
&= \Psi(x).
\end{align}
It is straightforward to see that $\norm{U}_\infty = \max\left(\norm{W}_\infty, \norm{V}_\infty \right)$, which completes the proof.
\end{proof}

\begin{lemma} \label{lem:rademacher_cubic_net}
Consider the function class $\cG_r$ defined in \cref{sup:eq:cubic_net_class}. It holds that 
\[\cR_n(\cG_r) \le \frac{324r^3 \sqrt{p}}{\sqrt{n}}.\]
\end{lemma}

\begin{proof}[Proof of \Cref{lem:rademacher_cubic_net}]
We can derive the Rademacher complexity of $\cG_r$ as
\begin{align}
\cR_{n}(\cG_r) &= \E_{\sigma \sim \operatorname{Unif}\left( \{\pm 1\}^n\right)} \left[ \sup_{U \in \R^{4 \times 3p}, \norm{U}_\infty \le r} \frac{1}{n} \sum_{a=1}^n \sigma_a \sum_{c=1}^4 \langle U_c^\top , x_a \rangle^3 \right] \notag \\
&\le \E_{\sigma \sim \operatorname{Unif}\left( \{\pm 1\}^n\right)} \left[ \sup_{U \in \R^{3p}, \norm{U}_\infty \le r} \frac{4}{n} \sum_{a=1}^n \sigma_a \langle U, x_a \rangle^3 \right] \notag \\
&\le \E_{\sigma \sim \operatorname{Unif}\left( \{\pm 1\}^n\right)} \left[ \sup_{U \in \R^{3p}, \norm{U}_\infty \le r} \frac{108 r^2}{n} \sum_{a=1}^n \sigma_a \langle U, x_a \rangle \right]
\intertext{where we used Talagrand’s contraction lemma (\Cref{lem:talgarand_contraction_lemma}) using the fact that $f(x)=x^3$ is $27r^2$-lipschitzness in $[-3r,3r]$ and that $\abs{\inner{U}{x_a}}\le 3r$ for all $\norm{U}_\infty\le r$, and we can continue:}
&\le \E_{\sigma \sim \operatorname{Unif}\left( \{\pm 1\}^n\right)} \left[ \sup_{U \in \R^{3p}, \norm{U}_\infty \le r} \frac{108 r^2}{n} \norm{U}_2 \left \lVert \sum_{a=1}^n \sigma_a x_a \right \rVert_2 \right] \notag \\
&\le \frac{108r^3 \sqrt{3p}}{n} \E_{\sigma \sim \operatorname{Unif}\left( \{\pm 1\}^n\right)} \left[ \left \lVert \sum_{a=1}^n \sigma_a x_a \right \rVert_2 \right] \notag \\
&\le \frac{324 r^3\sqrt{p}}{\sqrt{n}}.
\end{align}
where the last step follows from Rademacher complexity of linear model and the fact that $\norm{x}_2 = \sqrt{3}$ for all $x \in \dset_n$ and $\E_{\sigma \sim \Unif(\{\pm 1\}^n)} \left[ \norm{\sum_{a=1}^N \sigma_a}_2 \right] \le \sqrt{n}$.
\end{proof}

\begin{lemma}[Talagrand’s Contraction Lemma] \label{lem:talgarand_contraction_lemma}
Let $\cH$ be an arbitrary function class and $g$ be an $L$-Lipschitz function. It holds that 
\[
\cR_n\left( g \circ \cH \right) \le L \cdot \cR_n \left( \cH \right).
\]
\end{lemma}
A proof of this standard result can be found, for instance, as Lemma 5.7 of \citet{mrt}.

\begin{remark} \label{rem:rademacher_single_layer_net}
\cref{lem:rademacher_cubic_net} directly implies the following Rademacher complexity bound for the function class $\cH^1_{r,r'}$ defined in \cref{sup:eq:single_neuron_net_class}: \[\cR_n(\cH^1_{r,r'}) \le \frac{324 \sqrt{p}}{\sqrt{n}} \max(r,r')^3.\]
\end{remark}

\begin{lemma} \label{lem:rademacher_one_hidden_layer_net}
\[\cR_n(\cH^h_{r,r'}) \le \frac{324 h\sqrt{p}}{\sqrt{n}} \max(r, r')^3.\]
\end{lemma}

\begin{proof}[Proof of \Cref{lem:rademacher_one_hidden_layer_net}]

\begin{align}
\cR_{n}(\cH^h_{r,r'}) &= \E_{\sigma \sim \operatorname{Unif}\left( \{\pm 1\}^n\right)} \left[ \sup_{\Psi \in \cH^h_{r,r'}} \frac{1}{n} \sum_{i=1}^n \sigma_i \Psi(x_i) \right] \notag \\
&= \E_{\sigma \sim \operatorname{Unif}\left( \{\pm 1\}^n\right)} \left[ \sup_{W \in \cW_{h,r}, V \in \cV_{h,r'}} \frac{1}{n} \sum_{a=1}^n \sigma_a \left \langle e_{k_a}, V \left( W  \begin{pmatrix} e_{i_a} \\ e_{j_a} \end{pmatrix} \right)^{\odot 2} \right \rangle \right] \notag \\
&= \E_{\sigma \sim \operatorname{Unif}\left( \{\pm 1\}^n\right)} \left[ \sup_{W \in \cW_{h,r}, V \in \cV_{h,r'}} \frac{1}{n} \sum_{a=1}^n \sigma_a \left \langle e_{k_a}, \sum_{b=1}^h V_{:,b} \left( W_b  \begin{pmatrix} e_{i_a} \\ e_{j_a} \end{pmatrix} \right)^{\odot 2} \right \rangle \right] \notag \\
&\le \E_{\sigma \sim \operatorname{Unif}\left( \{\pm 1\}^n\right)} \left[ \sup_{W \in \cW_{1,r}, V \in \cV_{1,r'}} \frac{h}{n} \sum_{a=1}^n \sigma_a \left \langle e_{k_a}, V \left \langle W,  \begin{pmatrix} e_{i_a} \\ e_{j_a} \end{pmatrix} \right \rangle^{2} \right \rangle \right] \notag \\
&= h \cR_{n} (\cH^1_{r,r'})
\end{align}
where in the second to last line we used the fact that the Rademacher complexity of a two-layer NN with $h$ hidden neurons is bounded by $h$ times that of a single-hidden-neuron counterpart. Thus, applying \cref{rem:rademacher_single_layer_net} we can conclude that
\[
\cR_n(\cH^h_{r,r'}) \le \frac{324 h\sqrt{p}}{\sqrt{n}} \max(r, r')^3.
\]
\end{proof}

We are now ready to present the proof of the sample complexity upper bound for one-hidden-layer networks in the regression task. We first present the following Theorem from \citet{srebro2010smoothness} on bounding the excess risk of $H$-smooth loss functions.

\begin{theorem}[Theorem 1 from \citet{srebro2010smoothness}] \label{prop:smoothness_risk_ub} For an $H$-smooth non-negative loss $\ell$ such that for all $x, y, \Psi, \abs{\ell(\Psi(x), y)} \le b$ , for any $\delta > 0$ we have with probability at least $1-\delta$ over a random sample size of $n$, for any $\Psi \in \cH$,
\[
L(\Psi) \le \hat{L}(\Psi) + K \left(\sqrt{\hat{L}(\Psi)} \left( \sqrt{H} \log^{1.5} n \cR_n(\cH) + \sqrt{\frac{b \log (1/\delta)}{n}}\right) + H \log^3 n \cR^2_n(\cH) + \frac{b \log (1/\delta)}{n}  \right)
\]
where $K$ is a positive constant, $L(\Psi)$ denotes the population loss of $\Psi$ according to $\ell$ and $\hat{L}(\Psi)$ denotes the loss of $\Psi$ on the mentioned sample of size $n$ according to $\ell$. 
\end{theorem}

We now present the proof of \cref{th:generalization_bound_3d}.
\generalizationbound*

\begin{proof}
Consider the function class $\cH_{R, R}^h$ for whom we have already proved a Rademacher complexity upper bound. As $\norm{\theta}_\infty \le R$, and all the inputs are one-hot, for all $x = (e_i, e_j, e_k)$ and $g \in \cH^h_{R,R}$ it holds that $g(x) \le 4hR^3$. This boundedness accordingly implies smoothness of $\ell_2$ loss on this function class with $H=1$. Hence, \cref{prop:smoothness_risk_ub} directly applies to our function class, yielding:
\[\cL_{l_2}(g(\theta, \cdot)) \le \frac{CR^6h^2}{n} \bigg(p \log^3 n + \log (1 / \delta)\bigg)\] 
for some positive constant $C$ independet of other values.
\end{proof}

Finally, we remark that if the $\ell_2$ loss is small enough, then the misclassification error is guaranteed to be zero.
\begin{proposition} \label{lem:bounded_l2_loss_bounded_error}
Consider a predictor $g: \cX \to \R$. The population misclassification error is upper-bounded by $2\cL_{\ell_2}(g) / p$.
\end{proposition}
\begin{proof}
Note that each $(x, y) \in \cX \times \cY$ that is misclassified induces an $\ell_2$ loss of at least $p^2/2$. To see that why, for each pair $(e_a, e_b)$ to be misclassified while attaining minimum possible $\ell_2$ loss we need
\begin{align*}
    g((e_a, e_b, e_c)) &< \frac{p}{2} \\
    g((e_a, e_b, e_d)) &> \frac{p}{2} \\
    g((e_a, e_b, e_k)) &= 0 \quad \text{for all } k \not\in \{c,d\} 
\end{align*}
where $c = a + b \pmod p$, $d \in [p] \neq c$ and $k \in [p]$. Hence, each of $(e_a, e_b, e_c)$ and $(e_a, e_b, e_d)$ introduce an $\ell_2$ loss of at least $p^2/4$ in the regerssion task. \zhiyuan{needs to improve technical writing}
\end{proof}

\section{Construction of Interpolating Solution with Small \texorpdfstring{$\ell_\infty$}{} Norm} \label{app:manual_construction}

In this section, we prove \Cref{thm:small_inf_norm_exist}. We present a construction of weights that interpolates the dataset for $h=8p$. Then we generalize this result to any $h\ge 8p$ by duplicating the first $8p$ neurons $\left\lfloor \frac{h}{8p}\right\rfloor$ times, where each copy is $\left\lfloor \frac{h}{8p}\right\rfloor^{-\frac{1}{3}}$ times smaller in magnitude.

\smallinfnormexist*
\begin{proof}[Proof of \Cref{thm:small_inf_norm_exist}]
We begin by constructing 8 matrices of size $p \times 2p$ denoted by $W^{(i)}$. For every $n,m\in [p]$, we have that 

\noindent\begin{minipage}{.6\linewidth}
\begin{align} \label{eq:construction_1st_layer}
\begin{split}
W^{(1)}_{k,n} &= 
    \cos \left(\frac{2\pi k}{p} n \right) \qquad W^{(1)}_{k,m+p} = \textcolor{white}{+}\cos \left(\frac{2\pi k}{p} m \right) \\
W^{(2)}_{k,n} &=
    \cos \left(\frac{2\pi k}{p} n \right) \qquad W^{(2)}_{k,m+p =} -\cos \left(\frac{2\pi k}{p} m \right) \\
W^{(3)}_{k,n} &=
    \sin \left(\frac{2\pi k}{p} n \right) \qquad W^{(3)}_{k,m+p} = \textcolor{white}{+}\sin \left(\frac{2\pi k}{p} m \right) \\
W^{(4)}_{k,n} &=
    \sin \left(\frac{2\pi k}{p} n \right) \qquad W^{(4)}_{k,m+p} = -\sin \left(\frac{2\pi k}{p} m \right)\\
W^{(5)}_{k,n} &=
    \sin \left(\frac{2\pi k}{p} n \right) \qquad W^{(5)}_{k,m+p} = \textcolor{white}{+} \cos \left(\frac{2\pi k}{p} m \right) \\
W^{(6)}_{k,n} &=
    \sin \left(\frac{2\pi k}{p} n \right) \qquad W^{(6)}_{k,m+p} = -\cos \left(\frac{2\pi k}{p} m \right) \\
W^{(7)}_{k,n} &= 
    \cos \left(\frac{2\pi k}{p} n \right) \qquad W^{(7)}_{k,m+p} = \textcolor{white}{+} \sin \left(\frac{2\pi k}{p} m \right)  \\
W^{(8)}_{k,n} &=
    -\cos \left(\frac{2\pi k}{p} n \right) \qquad W^{(8)}_{k,m+p} = \sin \left(\frac{2\pi k}{p} m \right) \\
\end{split}
\end{align}
\end{minipage}
\begin{minipage}{.35\linewidth}
\begin{align} \label{eq:construction_2nd_layer}
\begin{split}
  V_{q, \, 8k:8(k+1)} = \begin{pmatrix}
      \textcolor{white}{+}\cos\left( \frac{2\pi k}{p} q \right) \\
      -\cos\left( \frac{2\pi k}{p} q \right) \\
      -\cos\left( \frac{2\pi k}{p} q \right) \\
      \textcolor{white}{+}\cos\left( \frac{2\pi k}{p} q \right) \\
      \textcolor{white}{+}\sin\left( \frac{2\pi k}{p} q \right) \\
      -\sin\left( \frac{2\pi k}{p} q \right) \\
      \textcolor{white}{+}\sin\left( \frac{2\pi k}{p} q \right) \\
      -\sin\left( \frac{2\pi k}{p} q \right) \\
  \end{pmatrix}^\top
\end{split}
\end{align}
\end{minipage}

Each $W^{(i)}$ for $1\le i\le 8$ is a $p \times 2p$ matrix, whose elements are given by the equations presented above. Hence, in each equation $k, n, m \in \{0, 1, \cdots, p \}$.
The construction of the first layer is based on stacking $W^{(i)}$ for $1\le i\le 8$ to construct $W \in \R^{8p \times 2p}$. The weights of the second layer are given in \cref{eq:construction_2nd_layer}, where $V_{q,\, 8k:8(k+1)}$ presents a slice of the second layer and $q, k \in \{0, 1, \cdots, p-1\}$.

To show that this construction solves the modular addition problem analytically, we will analytically perform the inference step for two arbitrary inputs $n, m$ where $x=(e_n, e_m)$. We denote $h= (Wx)^{\odot 2} \in \R^{8p}$ as the post-activations of the first layer, which is given by

\begin{equation}
    h_{8k:8(k+1)} = \begin{pmatrix}
        \cos \left(\frac{2\pi k}{p} n \right) + \cos \left(\frac{2\pi k}{p} m \right) \\
        \cos \left(\frac{2\pi k}{p} n \right) - \cos \left(\frac{2\pi k}{p} m \right) \\
        \sin \left(\frac{2\pi k}{p} n \right) + \sin \left(\frac{2\pi k}{p} m \right) \\
        \sin \left(\frac{2\pi k}{p} n \right) - \sin \left(\frac{2\pi k}{p} m \right) \\
        \sin \left(\frac{2\pi k}{p} n \right) + \cos \left(\frac{2\pi k}{p} m \right) \\
        \sin \left(\frac{2\pi k}{p} n \right) - \cos \left(\frac{2\pi k}{p} m \right) \\
        \cos \left(\frac{2\pi k}{p} n \right) + \sin \left(\frac{2\pi k}{p} m \right) \\
        \cos \left(\frac{2\pi k}{p} n \right) - \sin \left(\frac{2\pi k}{p} m \right) \\
    \end{pmatrix}^{2}.
\end{equation}

Note that for each $k$, we have that (after dropping $(e_n, e_m)$ for simplicity)
\begin{equation}
h_{8k} - h_{8k+1} =  2 \cos \left( \frac{2\pi k}{p} (n+m) \right) + 2 \cos \left( \frac{2\pi k}{p} (n-m) \right)
\end{equation}
and
\begin{equation}
h_{8k+2} - h_{8k+3} =  2 \cos \left( \frac{2\pi k}{p} (n-m) \right) - 2 \cos \left( \frac{2\pi k}{p} (n+m) \right)
\end{equation}
and 
\begin{equation}
h_{8k+4} - h_{8k+5} =  2 \sin \left( \frac{2\pi k}{p} (n+m) \right) + 2 \sin \left( \frac{2\pi k}{p} (n-m) \right)
\end{equation}
and 
\begin{equation}
h_{8k+6} - h_{8k+7} =  2 \sin \left( \frac{2\pi k}{p} (n+m) \right) - 2 \sin \left( \frac{2\pi k}{p} (n-m) \right).
\end{equation}

Hence, 
\begin{equation}
h_{8k} - h_{8k+1} - h_{8k+2} + h_{8k+3} =  4 \cos \left( \frac{2\pi k}{p} (n+m) \right)
\end{equation}
and
\begin{equation}
h_{8k+4} - h_{8k+5} + h_{8k+6} - h_{8k+7} =  4 \sin \left( \frac{2\pi k}{p} (n+m) \right).
\end{equation}
Using the fact that $\cos(a-b) = \cos(a)\cos(b) - \sin(a)\sin(b)$, we can see that

\begin{align}
\begin{split}
[f(e_n,e_m)]_q = \langle V_{q,:} , h \rangle &= 4 \sum_{k=0}^{p-1} \cos \left( \frac{2\pi k}{p} q \right) \cos \left( \frac{2\pi k}{p} (n+m) \right) - \sin \left( \frac{2\pi k}{p} q \right) \sin \left( \frac{2\pi k}{p} (n+m) \right) \\
&= 4 \sum_{k=0}^{p-1} \cos \left( \frac{2\pi k}{p} (m+n-q) \right) \\
&= 4p  \, \mathbf{1} \big( (m+n-q) \bmod p = 0 \big)
\end{split}
\end{align}
where the last equality follows from Euler's identity and needs $p$ to be odd. 
\end{proof}
\begin{remark} \label{remark:4p_neuron_construction}
Assuming $p$ is odd, we need at most $4p$ hidden neurons to interpolate the modular addition task. 
\end{remark}

Observing the fact that $\cos(2\pi - a) = \cos(a)$, we can see that
\begin{equation} \label{eq:root_of_identity_half}
\sum_{k=0}^{p-1} \cos \left( \frac{2\pi k}{p} (m+n-q) \right) = 1 + 2 \sum_{k=1}^{\frac{p-1}{2}} \cos \left( \frac{2\pi k}{p} (m+n-q) \right)
\end{equation}
where we replaced $\cos \left( \frac{2 \pi 0}{p} (m+n-q) \right)$ with 1. Based on \cref{eq:root_of_identity_half}, we can cut out half of the weights of the first and second layer, and only construct the frequencies up to $\frac{p-1}{2}$, which results in only needing $4p$ hidden neurons to construct the interpolating solution.

\section{Margin-Based Generalization Bound for Classification} \label{app:gen_bound}

We begin by providing some background and notation on sub-exponential random variables, which will be later used in the proof of our margin-based generalization bound.

\subsection{Background on sub-exponential variables}
The following proofs rely heavily on concentration inequalities for sub-exponential random variables; we will first review some background on these quantities.

A real-valued random variable $X$ with mean $\mu$ is called \textit{sub-exponential} \citep[see e.g.][]{wainwright_2019} if there are non-negative parameters $(\nu, \alpha)$ such that
\begin{equation}
\E[e^{\lambda (X - \mu)}] \le e^{\frac{\nu^2 \lambda^2}{2}} \quad \text{ for all } \abs{\lambda} < \frac{1}{\alpha}.
\label{eq:se-def}
\end{equation}
We use $X \sim SE(\nu, \alpha)$ to denote that $X$ is a sub-exponential random variable with parameters $(\nu, \alpha)$,
but note that this is not a particular distribution.

One famous sub-exponential random variable is the product of the absolute value of two standard normal distributions, $z_i \sim \N(0, 1)$, such that the two factors are either independent ($X_1 = \abs{z_1} \abs{z_2} \sim SE(\nu_p, \alpha_p)$ with mean $2/\pi$) or the same ($X_2 = z^2 \sim SE(2, 4)$ with mean 1). We now present a few lemmas regarding sub-exponential random variables that will come in handy in the later subsections of the appendix.

\begin{lemma} \label{Lemma_scaled_sub_exponential}
Assume $X$ is  sub-exponential with parameters $(\nu, \alpha)$. It holds that the random variable $sX$ where $s \in \R^+$ is also sub-exponential, but with parameters $(s \nu, s\alpha)$.
\end{lemma}

\begin{proof}
Consider $X \sim SE(\nu, \alpha)$ and $X' = sX$ with $\E[X'] = s \E[X]$. Based on the definition of sub-exponential random variables

\begin{align}
\begin{split}
    &\E \left[ \exp \left(\lambda (X - \mu) \right) \right]  \le \exp(\frac{\nu^2 \lambda^2}{2}) \quad \text{ for all } \abs{\lambda} < \frac{1}{\alpha} \\
    &\Longrightarrow \E \left[ \exp \left(\frac{\lambda}{s} (sX - s\mu) \right) \right] \le \exp(\frac{\nu^2 s^2 \frac{\lambda^2}{s^2}}{2}) \quad \text{ for all } \abs{\frac{\lambda}{s}} < \frac{1}{s\alpha} \\
    & \xRightarrow{\lambda' = \frac{\lambda}{s}} \E \left[ \exp \left(\lambda' (X' - \mu') \right) \right] \le \exp(\frac{{\nu^2 s}^2 {\lambda'}^2}{2}) \quad \text{ for all } \abs{\lambda'} < \frac{1}{s\alpha}
\end{split}
\end{align}
Defining $\alpha' = s \alpha$ and $\nu' = s \nu$ we see that $X' \sim SE(s\nu, s \alpha)$.
\end{proof}

\begin{proposition} \label{prop:supp:prob_sum_subexp}
If all of the random variables $X_i$ for $i \in [N]$ for $N \in \mathbb{N}^+$ are sub-exponential with parameters $(\nu_i, \alpha_i)$, and all of them are independent, then $\sum_{i=1}^N X_i \in SE(\sqrt{\sum_{i=1}^N \nu_i^2}, \max_i \alpha_i)$,
and 
$\frac{1}{N} \sum_{i=1}^N X_i \sim SE\left( \frac{1}{\sqrt N} \sqrt{\frac1N \sum_{i=1}^N \nu_i^2}, \frac1N \max_i \alpha_i\right)$.
\end{proposition}
\begin{proof}
This is a simplification of the discussion prior to equation 2.18 of \citet{wainwright_2019}.
\end{proof}

\begin{proposition} \label{bernsetin_ineq}
For a random variable $X \sim SE(\nu, \alpha)$, the following concentration inequality holds:
\[
\Pr \left( \abs{X - \mu} \ge t \right) \le 2 \exp \left( - \min \left( \frac{t^2}{2\nu^2}, \frac{t}{2\alpha} \right) \right)
.\]
\end{proposition}

\begin{proof}
The proof is straightforward from multiplying the result derived in Equation 2.18 of \citet{wainwright_2019} by a scalar.
\end{proof}

\begin{corollary} \label{bernstein_in_delta}
Consider $X \sim SE(\nu, \alpha)$, the following bound holds with probability at least $1-\delta$:
\[
\abs{X - \mu} < \max \left(\nu \sqrt{2 \log \frac{2}{\delta}}, 2\alpha \log \frac{2}{\delta} \right)
.\]
\end{corollary}

A sub-Gaussian random variable, $SG(\nu)$, is one which satisfies \eqref{eq:se-def} for all $\lambda$, i.e.\ it is the limit of $SE(\nu, \alpha)$ as $\alpha \to 0$.
\begin{proposition}[Chernoff bound] \label{chernoff_in_delta}
If $X$ is $SG(\nu)$,
then with probability at least $1-\delta$,
$\lvert X - \mu \rvert \le \nu \sqrt{2 \log \frac2\delta}$.
\end{proposition}
\begin{proposition}[Hoeffding's inequality] \label{hoeffding_in_delta}
If $X_1, \dots, X_n$ are independent variables with means $\mu_i$ and each $SG(\nu_i)$, then $\left\lvert \sum_{i=1}^n X_i - \sum_{i=1}^n \mu_i \right\rvert \le \sqrt{2 \left( \sum_{i=1}^n \nu_i^2 \right) \log\frac2\delta}$ with probability at least $1 - \delta$.
\end{proposition}

\subsection{Generalization Bound}
We are now ready to state the main theorem for proving an upper bound on the number of training points needed to generalize. We begin by defining some notations and operators that will be useful in the main proof. First, we define $\mathbf{0}_{a \times b}$ to be the zero matrix (or vector in case its one-dimensional) of shape $a \times b$. 
\begin{definition} \label{def:classification_weights_2layer_net}
Assume $p \ge 2$ is an integer. We define $\Theta_r^{h,b}$ as the set of possible parameters of one-hidden layer quadratic networks of width $h$ whose $\ell_\infty$ norm is bounded by $r$ and have $b$ output logits. Formally,
\[\Theta_r^{h, b} \triangleq \left\{ (W, V) \; \bigg\vert \; W \in \R^{h \times 2p}, V \in \R^{b \times h}, \norm{V}_\infty \le r, \norm{W}_\infty \le r \right\}.\]
We also define $\cM_r^h \triangleq \left\{ (W, V) \in \Theta_r^h \; \vert \; \forall i \in [p]; \, \sum_{j=1}^h V_{ij} = 0 \right\}$ as the parameters whose second layer weights' rows have a zero sum.
\end{definition} 
\begin{definition} \label{def:noise_addition_operator}
We next define an operator for adding gaussian noise to input matrices or vectors:
\[
\Lambda_\sigma(A) = A + \tilde{A} 
\]
where $A$ is any matrix or vector and $\tilde{A}$ has the same shape as $A$ except its entries are i.i.d sampled from the Gaussian distribution $\N(0, \sigma^2)$.
\end{definition}
Our proof relies on Lemma 1 from \citet{neyshabur2018a}. For convenience, we re-iterate this lemma here.
\begin{lemma}[Lemma 1 from \citep{neyshabur2018a}] \label{neyshabur_margin_lemma}
Let $f(\theta; \cdot)$ be a predictor $\cX \to \R^{K}$ for some integer $K > 0$ with parameters $\theta$ and $P$ be any distribution on parameters that is independent of the training data. For any $\gamma, \delta > 0$, it holds with probability at least $1-\delta$ over randomness of training for any $\theta$ and any distribution on parameters $\cP_{\theta}$ such that $\Pr_{u \sim \cP_{\theta}} \left[\max_{x \in \cX, k \in [K]} \big\lvert f_k(\theta + u; x) - f_k(\theta; x) \big\rvert < \gamma /4 \right] > 1/2$ 
\[
L_0(f, \theta, \dset) \le L_{\gamma}(f, \theta, \dset_{\train}) + 4\sqrt{\frac{\operatorname{KL}(\theta + u \Vert P) + \log \frac{6m}{\delta}}{m-1}}
\]
where $\dset$ denotes the population.
\end{lemma}
The following is main Theorem for this section.

\classgeneralizationbound*

\begin{proof}
Let us first construct $\theta = (W, V) \in \cM_r^{h,p}$ where $h = 2h'$ from $\theta' = (W', V')$ such that $W = \begin{bmatrix} W' \\ W' \end{bmatrix}$ and $V = \begin{bmatrix} V' & -V' \end{bmatrix}$. This network has the same outputs as the original one with parameters $\theta'$, while each row in $V$ has a zero sum (and hence $\theta \in \cM_r^{h,p}$). Since the outputs of the network with parameters $\theta$ are the same as those of $\theta'$, any generalization bound applying to parameters $\theta$ also applies to the parameters $\theta'$. Note that 
\begin{align}
    f_c \Big(\big(\Lambda_{\sigma}(V), \Lambda_{\sigma}(W) \big), (e_a, e_b) \Big) &= (V_c + \tilde{V}_c)^\top \Big( \big(W + \tilde{W} \big) (e_a, e_b) \Big)^{\odot 2} \notag \\
    &= \underbrace{ f_c(\theta, (e_a, e_b)) + V_c \Big( \tilde{Q}^{\odot 2} + 2 Q \odot \tilde{Q} \Big)}_{\text{\Cref{lem:concentration_noisy_first_layer}}} + \underbrace{\tilde{V}_c (Q + \tilde{Q})^{\odot 2}}_{\text{\Cref{lem:concentration_of_noise_with_noisy_first_layer}}}.
\end{align}
where we denoted $\tilde{V} = \Lambda_{\sigma}(V) - V, \tilde{W} = \Lambda_{\sigma}(W) - W$, $Q = (W_a + W_b)$ and $\tilde{Q} = (\tilde{W}_a + \tilde{W}_b)$. As noted in the inequality, we can apply \Cref{lem:concentration_noisy_first_layer,lem:concentration_of_noise_with_noisy_first_layer} to show that for any $\delta_1 \in (0,1)$ with probability at least $1 - \delta_1$ over the randomness of perturbation it holds that
\begin{align}
    & \left \lvert f_c \Big(\big(\Lambda_{\sigma}(V), \Lambda_{\sigma}(W) \big), (e_a, e_b) \Big) - f_c \big(\theta, (e_a, e_b) \big) \right \rvert \notag \\
    &\qquad\le  16 \sqrt{2h \log \frac{2}{\delta_1}} \max \left(r^2 \sigma, r \sigma^2 \right) + 32 \sqrt{2h} \left(\log \frac{2(h+1)}{\delta_1} \right)^{3/2} \max(\sigma^3, r\sigma^2, r^2 \sigma) \notag \\
    &\qquad\le 64 \sqrt{2h} \left(\log \frac{2(h+1)}{\delta_1} \right)^{3/2} \max(\sigma^3, r\sigma^2, r^2 \sigma).
\end{align}
Apply a union bound on all different $c \in [p]$ to see that for any $\delta_2 \in (0,1)$ with probability at least $1 - \delta_2$ over randomness of perturbation
\begin{align} 
    &\max_{c \in [p]} \Bigg\lvert f_c \Big(\big(\Lambda_{\sigma}(V), \Lambda_{\sigma}(W) \big), (e_a, e_b) \Big) - f_c \big(\theta, (e_a, e_b) \big) \Bigg \rvert  \\
    &\qquad\qquad \le 64 \sqrt{2h} \left(\log \frac{2p(h+1)}{\delta_2} \right)^{3/2} \max(\sigma^3, r\sigma^2, r^2 \sigma).
\end{align}
Hence, we'd want 
\begin{align}
    &64 \sqrt{2h} \left(\log \frac{2p(h+1)}{\delta_2} \right)^{3/2} \max\big((\sigma/r)^3, (\sigma/r)^2, \sigma/r \big) \le \frac{\gamma}{4r^3} \notag \\
    &\Longrightarrow \max\big((\sigma/r)^3, (\sigma/r)^2, \sigma/r \big) \le \frac{\gamma/r^3}{256\sqrt{2h} \left(\log \frac{2(h+1)}{\delta_1} \right)^{3/2}} \\
    &\Longrightarrow \sigma = \begin{cases}
        \left(\frac{\gamma/r^3}{256\sqrt{2h} \left(\log \frac{2(h+1)}{\delta_1} \right)^{3/2}} \right)^{1/3} r & \textbf{(*)} \vspace*{3mm}\\
        \frac{\gamma/r^3}{256\sqrt{2h} \left(\log \frac{2(h+1)}{\delta_1} \right)^{3/2}} r & \textbf{(**)}
    \end{cases}
\end{align}
where $\frac{\gamma/r^3}{256\sqrt{2h} \left(\log \frac{2(h+1)}{\delta_1} \right)^{3/2}} > 1$ decides the event $\textbf{(*)}$ and $\frac{\gamma/r^3}{256\sqrt{2h} \left(\log \frac{2(h+1)}{\delta_1} \right)^{3/2}} \le 1$ decides the event $\textbf{(**)}$. 
Assuming we're in the regime where $\sigma > r$, we can choose $\delta_2 < 1/2$ to see that with probability at least $1/2$ 
\begin{align} 
    &\max_{c \in [p]} \Bigg\lvert f_c \Big(\big(\Lambda_{\sigma}(V), \Lambda_{\sigma}(W) \big), (e_a, e_b) \Big) - f_c \big(\theta, (e_a, e_b) \big) \Bigg \rvert \le \gamma / 4.
\end{align}
Note that $\operatorname{KL}(\Lambda_{\sigma}(A) \Vert \N(0, \sigma^2)) \le \frac{\norm{A}^2_F}{2\sigma^2}$ for any matrix $A$.  Apply \Cref{neyshabur_margin_lemma} to see that with probability at least $1 - \delta$ over randomness of $\dset_{\train}$ of size $n$ we have that
\begin{align}
    L_0 (f, \theta, \dset) &\le L_{\gamma} (f, \theta, \dset_{\train}) + 4\sqrt{\frac{\frac{3hp \log \frac{2(h+1)}{\delta_1} }{\left(\frac{\gamma/r^3}{256\sqrt{2h}} \right)^{2/3}} + \log \frac{6n}{\delta}}{n-1}} \notag \\
    &\le L_{\gamma} (f, \theta, \dset_{\train}) + \tilde\bigO\left(\sqrt{\frac{p}{n}} \cdot \sqrt[3]{\frac{h^2}{\gamma / r^3}}\right)
\end{align}
for any $\theta \in \cM_r^{h,p}$. 
\end{proof}

\begin{lemma} \label{lem:concentration_noisy_second_layer}
Choose $\sigma, \delta > 0$ and integer $p \ge 2$. For any $r > 0$ and integers $h \ge 1, a, b \in [p]$ it holds with probability at least $1-\delta$ that
\[
\left \lvert {\Lambda_\sigma(V)}^\top \big( W (e_a, e_b)
  \big)^{\odot 2} - V^\top \big( W (e_a, e_b)
  \big)^{\odot 2} \right \rvert \le 4r^2\sigma \sqrt{h\log \frac{2}{\delta}}
\]
where $(W, V) \in \Theta_r^{h,1}$.
\end{lemma}

\begin{proof}
Denote by $\tilde{V} = \Lambda_\sigma(V) - V$ and $Q = \big( W(e_a, e_b)\big)^{\odot 2}$. We can expand the target as $V^\top Q + \tilde{V}^\top Q$ where the first summand is constant and the second summand is distributed according to $\N \left(0, \sigma^2 \norm{Q}_2^2 \right)$. Note that since $\norm{W}_{\infty} \le r$, we have that $\norm{Q}_2^2 \le 16hr^4$. Applying \Cref{chernoff_in_delta} on this Gaussian random variable, one can see that with probability at least $1 - \delta$ over randomness of perturbation
\begin{align}
\left \lvert {\Lambda_\sigma(V)}^\top \big( W (e_a, e_b)
  \big)^{\odot 2} - V^\top Q \right \rvert &\le 4r^2\sigma \sqrt{h\log \frac{2}{\delta}}.
\end{align}
\end{proof}

\begin{lemma} \label{lem:concentration_noisy_first_layer}
Choose $\sigma, \delta > 0$ and integer $p \ge 2$. For any $r > 0$ and integers $h \ge 8 \log \frac{2}{\delta}, a, b \in [p]$ it holds with probability at least $1-\delta$ that
\[
\left \lvert V^\top \big( \Lambda_\sigma(W) (e_a, e_b)\big)^{\odot 2} - V^\top \Big( W(e_a, e_b) \Big)^{\odot 2} \right \rvert \le 16 \sqrt{h \log \frac{2}{\delta}} \max \left(r^2 \sigma, r \sigma^2 \right)
\]
where $(W, V) \in \cM_r^{h,1}$.
\end{lemma}

\begin{proof}
Denote by $\tilde{W} = \Lambda_\sigma(W) - W$, $Q = \big(W(e_a,e_b)\big)^{\odot 2}$ and $\tilde{Q} = \big(\tilde{W}(e_a,e_b)\big)^{\odot 2}$. Note that each coordinate of $\tilde{Q}$ is sub-exponential with parameters $\SE(4\sigma^2,8\sigma^2)$ and mean 2. We can expand 
\[
V^\top \big( \Lambda_\sigma(W) (e_a, e_b)\big)^{\odot 2} = V^\top Q + V^\top \tilde{Q} + 2 V^\top \Big((W_a + W_b) \odot (\tilde{W}_a + \tilde{W}_b) \Big).
\]
Note that $V^\top \tilde{Q}$ is sub-exponential with parameters $\SE(4r\sigma^2\sqrt{h}, 8r\sigma^2)$ \footnote{Note that these parameters are not tight, but this doesn't affect the correctness of this argument. For example, a random variable that is $\SE(a, b)$ is also $\SE(2a, 2b)$, or if it's $\SG(a)$, then it is also $\SG(2a)$.} and mean 0 (due to $\sum_{i=1}^h V_i = 0$ and linearity of expectation). We can apply \Cref{bernstein_in_delta} to see that with probability at least $1 - \delta/2$ 
\[
\left \lvert V^\top \tilde{Q} \right \rvert \le 8r\sigma^2 \max \left(\sqrt{2h \log \frac{2}{\delta}}, 4\log \frac{2}{\delta} \right).
\]
Moreover, Since $ 2 V^\top \Big((W_a + W_b) \odot (\tilde{W}_a + \tilde{W}_b) \Big)$ is distributed according to $\cN(0, 8\sigma^2 \norm{V \odot (W_a + W_b)}_2^2)$ and $\norm{V \odot (W_a + W_b)}_2^2 \le 4hr^4$, applying \cref{chernoff_in_delta} reveals that with probability at least $1 - \delta/2$ over randomness of perturbation 
\[
\left \lvert 2 V^\top \Big((W_a + W_b) \odot (\tilde{W}_a + \tilde{W}_b) \Big) \right \rvert \le 8 r^2 \sigma \sqrt{h \log \frac{2}{\delta}}.
\]
Combining the two equations above shows that with probability at least $1 - \delta$ over randomness of perturbation it holds that
\[
\left \lvert V^\top \big( \Lambda_\sigma(W) (e_a, e_b)\big)^{\odot 2} -  V^\top Q \right \rvert \le 16 \sqrt{h \log \frac{2}{\delta}} \max \left(r^2 \sigma, r \sigma^2 \right).
\] 
\end{proof}

\begin{lemma} \label{lem:concentration_of_noise_products}
Choose $\sigma, \delta > 0$ and integer $p \ge 2$. For any $r > 0$ and integers $h \ge e \delta, a, b \in [p]$ it holds with probability at least $1 - \delta$ over randomness of perturbation that
\[
\left \lvert \Lambda_{\sigma}(\mathbf{0}_h)^\top \Big( \Lambda_{\sigma}(\mathbf{0}_{h \times 2p}) (e_a, e_b) \Big)^{\odot 2} \right \rvert \le \left(2\sigma^2 + 8 \sqrt{2} \sigma^2 \log \frac{2(h+1)}{\delta} \right) \sigma \sqrt{h \log \frac{2(h+1)}{\delta}}.
\]
\end{lemma}
\begin{proof}
Denote by $\tilde{V} = \Lambda_{\sigma}(\mathbf{0}_h)$ and $\tilde{Q} = \Big( \Lambda_{\sigma}(\mathbf{0}_{h \times 2p}) (e_a, e_b) \Big)^{\odot 2}$. It's easy to see that each coordinate of $\tilde{Q}$ is sub-exponential with parameters $\SE(2\sigma^2\sqrt{2}, 4\sigma^2\sqrt{2})$ and mean $2\sigma^2   $. To bound $\tilde{V}^\top \tilde{Q}$, we employ the following strategy: since each coordinate of $\tilde{Q}$ is a sub-exponential random variable, we can use a union bound in combination with \Cref{bernstein_in_delta} to derive a bound on the maximum value of them. Next, we pull this maximum value out of the sum, and apply \Cref{hoeffding_in_delta} to bound the sum of remaining independent Gaussians. Combining these two high probability events, we present a high probability bound on $\tilde{V}^\top \tilde{Q}$ being bounded. Formally, for arbitrary $\delta_1, \delta_2 > 0$: 
\begin{align}
    &\Pr \left[ \tilde{Q}_i \le 2\sigma^2 + \sigma^2 \max\left( 4\sqrt{\log \frac{2h}{\delta_1}}, 8 \sqrt{2} \log \frac{2h}{\delta_1} \right) \text{ for all } i \in [h]  \right] \ge 1 - \delta_1 \notag \\
    &\qquad \Longrightarrow \Pr \left[ \left \lvert \tilde{V}^\top \tilde{Q} \right \rvert \le \left(2\sigma^2 + 8 \sqrt{2} \sigma^2 \log \frac{2h}{\delta_1} \right) \left\lvert \sum_{i=1}^h \tilde{V_i} \right \rvert \right] \ge 1 - \delta_1 \notag \\
    &\qquad \Longrightarrow \Pr \left[ \left \lvert \tilde{V}^\top \tilde{Q} \right \rvert \le \left(2\sigma^2 + 8 \sqrt{2} \sigma^2 \log \frac{2h}{\delta_1} \right) \sigma \sqrt{h \log \frac{2}{\delta_2}} \right] \ge 1 - \delta_1 - \delta_2 \notag \\
    &\qquad \Longrightarrow \Pr \left[ \left \lvert \tilde{V}^\top \tilde{Q} \right \rvert \le \left(2\sigma^2 + 8 \sqrt{2} \sigma^2 \log \frac{2(h+1)}{\delta} \right) \sigma \sqrt{h \log \frac{2(h+1)}{\delta}} \right] \ge 1 - \delta
\end{align}
where for the last step to be correct we chose $\delta_1 = \frac{h}{h+1} \delta$ and $\delta_2 = \frac{1}{h+1} \delta$.
\end{proof}

\begin{lemma} \label{lem:concentration_of_noise_with_noisy_first_layer}
Choose $\sigma, \delta > 0$ and integer $p \ge 2$. For any $r > 0$ and integers $h \ge 8 \log \frac{2}{\delta}, a, b \in [p]$ it holds that 
\[
\left \lvert {\Lambda_\sigma(\mathbf{0}_{h})}^\top \Big( \Lambda_\sigma(W) (e_a, e_b)
  \Big)^{\odot 2} \right \rvert \le 32 \sqrt{h} \left(\log \frac{2(h+1)}{\delta_4} \right)^{3/2} \max(\sigma^3, r\sigma^2, r^2 \sigma, \sigma).
\]
where $W \in \R^{h \times p}$ such that $\norm{W}_\infty \le r$. 
\end{lemma}

\begin{proof}

Denote by $\tilde{V} = \Lambda_\sigma(\mathbf{0}_h)$, $\tilde{W} = \Lambda_\sigma(W) - W$, $Q = W (e_a, e_b)$ and $\tilde{Q} = \tilde{W} (e_a, e_b)$. We have that
\begin{align}
{\Lambda_\sigma(\mathbf{0}_{h})}^\top \Big( \Lambda_\sigma(W) (e_a, e_b)
\Big)^{\odot 2} &= \tilde{V}^\top \Big( W_a + W_b + \tilde{W}_a + \tilde{W}_b \Big)^{\odot 2} \notag \\
&= \tilde{V}^\top \Big( Q^{\odot 2} + \tilde{Q}^{\odot 2} + 2 Q \odot \tilde{Q} \Big)
\end{align}

In \Cref{lem:concentration_noisy_second_layer} we have already shown that for any $\delta_1 > 0$ with probability at least $1 - \delta_1$ over randomness of perturbation it holds that
\begin{equation} \label{eq:contraction_ub1}
    \left \lvert \tilde{V}^\top Q^{\odot 2} \right \rvert \le 4r^2\sigma \sqrt{h \log \frac{2}{\delta_1}}.
\end{equation}

Denote $\xi = \tilde{V} \odot \tilde{Q}$. $\xi$, the vector of product of two independent Gaussians, is sub-exponential with parameters $\SE(2\sigma^2\sqrt{2}, 4\sigma^2\sqrt{2})$ and mean 0 (and hence sum of its coordinates is $\SE(2\sigma^2\sqrt{2h}, 4\sigma^2\sqrt{2})$)\amin{are parameters correct?}. Since $\norm{W}_{\infty} \le r$, applying \Cref{bernstein_in_delta} yields that for any $\delta_2 > 0$, it holds with probability at least $1 - \delta_2$ that
\begin{align} \label{eq:contraction_ub2}
\left \lvert 2 \tilde{V}^\top \big(Q \odot \tilde{Q}\big) \right \rvert &\le 4\sqrt{2} r \sigma^2 \max \left( \sqrt{2h \log \frac{2}{\delta_2}},  4 \log \frac{2}{\delta_2} \right) \notag \\
&\le 16 r \sigma^2 \sqrt{h \log \frac{2}{\delta_2}}.
\end{align}

Finally, we employ \Cref{lem:concentration_of_noise_products} to show that for any $\delta_3 > 0$ it holds with probability at least $1 - \delta_3$ that
\begin{equation} \label{eq:contraction_ub3}
\left\lvert \tilde{V}^\top \tilde{Q}^{\odot 2} \right \rvert \le \left(2\sigma^2 + 8 \sqrt{2} \sigma^2 \log \frac{2(h+1)}{\delta_3} \right) \sigma \sqrt{h \log \frac{2(h+1)}{\delta_3}}.
\end{equation}

Applying a union bound on \Cref{eq:contraction_ub1,eq:contraction_ub2,eq:contraction_ub3} and choosign $\delta_4 = \delta_1 / 3 = \delta_2 / 3 = \delta_3 / (3h+3)$ reveals that with probability at least $1 - \delta_4$
\begin{align}
\left \lvert {\Lambda_\sigma(\mathbf{0}_{h})}^\top \Big( \Lambda_\sigma(W) (e_a, e_b)
\Big)^{\odot 2} \right \rvert &\le  16\sqrt{h \log \frac{h}{\delta_4}} \max \left( r \sigma^2, r^2 \sigma \right) +  16 \sqrt{2h} \sigma^3 \left(\log \frac{2(h+1)}{\delta_4} \right)^{3/2} \notag \\
&\le 32 \sqrt{h} \left(\log \frac{2(h+1)}{\delta_4} \right)^{3/2} \max(\sigma^3, r\sigma^2, r^2 \sigma).
\end{align}

\end{proof}

\end{document}